\documentclass[10pt,journal,one column, compsoc]{IEEEtran}
\usepackage{multirow}
\usepackage{amssymb}
\usepackage{amsthm}
\usepackage{amsbsy, mathrsfs, mathdots}

\usepackage{algorithm}
\usepackage{algpseudocode}
\usepackage[hidelinks]{hyperref}

\newtheorem{theorem}{Theorem}
\newtheorem{lemma}[theorem]{Lemma}

\newtheorem{remark}{Remark}
\newtheorem{example}{Example}
\newtheorem{assumption}{Assumption}

\usepackage[pdftex]{graphicx}
\usepackage{graphbox}
\usepackage{mathtools, nccmath}
\usepackage{bbm}
\usepackage[utf8]{inputenc}
\usepackage[T1]{fontenc}    
\usepackage{lmodern}

\usepackage{booktabs}       
\usepackage{amsfonts}       
\usepackage{nicefrac}      
\usepackage{microtype}      
\usepackage{yhmath}        
\usepackage{titlesec}
\usepackage{mwe}
\usepackage{multirow}
\usepackage{hhline}
\usepackage{color}
\usepackage{epsfig}
\usepackage{subfig}
\usepackage{graphicx}  
\usepackage{bm}
\usepackage{comment}
\usepackage[numbers,square,comma, sort&compress]{natbib}
\usepackage{epstopdf}
\usepackage{physics}
\usepackage{setspace}
 \usepackage[dvipsnames]{xcolor}

\def\R{\mathbb{R}}

\def\BC{{\bm{C}}}
\def\BU{{\bm{U}}}
\def\BR{{\bm{R}}}
\def\BV{\bm{V}}
\def\BW{\bm{W}}
\def\BSigma{\bm{\Sigma}}
\def\BX{\bm{X}}
\def\BY{\bm{Y}}
\def\tX{\widetilde{\bm{X}}}

\def\BW{\bm{W}}
\def\cI{{\mathcal{I}}}
\def\cJ{{\mathcal{J}}}

\def\cH{\mathcal{H}}
\def\cO{\mathcal{O}}
\def\cP{\mathcal{P}}
\def\be{\bm{e}}

\def\fro{{\mathrm{F}}}

\def\rank{\mathrm{rank}}

\DeclareMathOperator*{\minimize}{\mathrm{minimize}}
\DeclareMathOperator*{\subject}{\mathrm{subject~to}}

\graphicspath{{fig/}}

\definecolor{officegreen}{rgb}{0.0, 0.5, 0.0}
\begin{document}

\title{Matrix Completion with Cross-Concentrated Sampling: Bridging Uniform Sampling and CUR Sampling}

\author{HanQin~Cai, 
        Longxiu~Huang, 
        Pengyu~Li,
        and~Deanna~Needell
\IEEEcompsocitemizethanks{\IEEEcompsocthanksitem 
H.Q.~Cai is with Department of Statistics and Data Science and Department of Computer Science, University of Central Florida, Orlando, FL 32816, USA. (email: \href{mailto:hqcai@ucf.edu}{hqcai@ucf.edu}) \protect\\
 L.~Huang is with   Department of Computational Mathematics, Science and Engineering and Department of Mathematics, Michigan State University, East Lansing, MI 48823, USA. (email: \href{mailto:huangl3@msu.edu}{huangl3@msu.edu}) \protect\\
P.~Li is with Department of Electrical Engineering and Computer Science, University of Michigan, Ann Arbor, MI 48109, USA. (email: \href{mailto:lipengyu@umich.edu}{lipengyu@umich.edu}) \protect\\
D.~Needell is with Department of  Mathematics, University of California, Los Angeles, Los Angeles, CA 90095, USA. (email: \href{mailto:deanna@math.ucla.edu}{deanna@math.ucla.edu})  \protect\\
 }
}

\IEEEtitleabstractindextext{
\begin{abstract}
While uniform sampling has been widely studied in the matrix completion literature, CUR sampling approximates a low-rank matrix via row and column samples. Unfortunately, both sampling models lack flexibility for various circumstances in real-world applications. In this work, we propose a novel and easy-to-implement sampling strategy, coined Cross-Concentrated Sampling (CCS). By bridging uniform sampling and CUR sampling,  CCS provides extra flexibility that can potentially save sampling costs in applications. In addition, we also provide a sufficient condition for CCS-based matrix completion. %has been established \dn{This sentence reads like someone else established this condition}.
Moreover, we propose a highly efficient non-convex algorithm, termed Iterative CUR Completion (ICURC), for the proposed CCS model. Numerical experiments verify the empirical advantages of CCS and ICURC against uniform sampling and its baseline algorithms, on both synthetic and real-world datasets.  ~\\
\end{abstract}

}

\maketitle

\section{Introduction}

\IEEEPARstart{T}{he} problem of \textit{matrix completion} (MC) \cite{candes2009exact} has received much attention since the last decade. It has arisen in a wide range of applications, e.g., collaborative filtering \cite{Netflix,goldberg1992using}, image processing \cite{chen2004recovering,hu2012fast}, signal processing \cite{cai2019fast,cai2022structured}, genomics \cite{chi2013genotype,cai2016structured}, multi-task learning \cite{argyriou2008convex}, system identification \cite{liu2010interior}, and sensor localization \cite{singer2010uniqueness,tasissa2018exact}.

In this paper, we study MC under the setting of fixed rank. Consider a rank-$r$ matrix $\BX$ with observed entries indexed by a set $\Omega$.  MC aims to recover the original matrix $\BX$ from its partial observations. Naturally, we can model this problem as a minimization problem: 
\begin{equation}
    \begin{split}
     \minimize_{\tX}&\quad\frac{1}{2}\langle\cP_{\Omega}(\BX-\tX),\BX-\tX\rangle\\ 
     \subject&\quad \rank(\tX)=r,
    \end{split}
\end{equation} 
where $\langle~,\,\rangle$ denotes the Frobenius inner product and $\cP$ is the sampling operator defined as
\begin{equation} \label{eq:sample operator}
\cP_{\Omega}(\BX)=\sum_{(i,j)\in\Omega}[\BX]_{i,j}\be_i\be_j^\top.
\end{equation}
For the success of recovery, the general setting of MC requires the observation set $\Omega$ to be sampled via a certain unbiased stochastic process, e.g., 
uniform sampling with/without replacement \cite{recht2011simpler,klopp2014noisy} or Bernoulli sampling  \cite{candes2009exact}
 over all matrix entries. While such sample setting has been well studied in both theoretical and empirical aspects \cite{candes2009exact,cai2010singular,candes2010power,keshavan2010matrix,recht2011simpler,jain2013low,sun2016guaranteed}, it is too restricted in some applications due to hardware, financial, or environmental limitations. For instance, in the application of collaborative filtering, the rows and columns of the data matrix represent the users and rated objects (e.g., movies and merchandise) respectively. The unbiased sample models implicitly assume that all users have the interest to rate all objects with same probability---something that is quite unlikely in practice.

Here we consider \textit{CUR decomposition} as a potential alternative for efficient MC. 
CUR decomposition is also known as skeleton decomposition \cite{zamarashkin1997pseudo,chiu2013sublinear}, which attempts to utilize the self-expressiveness of the data in its low-rank matrix decomposition. 
There are several different, yet equivalent, formats for CUR decomposition \cite{HammHuang,hamm2023generalized}. In particular, we focus on the following CUR format:
\begin{equation} \label{eq:CUR}
    \BX=\BC \BU^\dagger \BR,
\end{equation}
where $\BX$ is the low-rank data matrix, $\BR$ and $\BC$ are selected row and column submatrices of $\BX$, and $\BU$ is the intersection submatrix of $\BR$ and $\BC$. Obviously, \eqref{eq:CUR} may not hold for arbitrary column and row submatrices. In fact, the following theorem gives a necessary and sufficient condition for CUR decomposition.
\begin{theorem}[\cite{HammHuang}] \label{thm:cur_iff_condition}
For given row and column submatrices $\BR$ and $\BC$, the CUR decomposition \eqref{eq:CUR} holds if and only if $\rank(\BU)=\rank(\BX)=r$.
\end{theorem}
Moreover, with the commonly assumed $\mu$-incoherence (see Assumption~\ref{asm:inco} in Section~\ref{sec:theoretical_results} later), the following theorem suggests that uniformly selected column and row submatrices are sufficiently good for CUR decomposition.
\begin{theorem}[{\cite{chiu2013sublinear,cai2020rapid,cai2021robust}}] \label{thm:uniform_sample_is_good}
    Let $\BX\in\R^{n\times n}$ be a rank-$r$ matrix with $\mu$-incoherence\footnote{To simplify our expression, we assume the matrices are square throughout the paper but all the results can be generalized to rectangular matrices.}. Suppose we sample $|\cI|\geq 10\mu_1 r \log(n)$ rows and $|\cJ|\geq 10\mu_2 r \log(n)$ columns uniformly with replacement.  Then 
    $\BU = [\BX]_{\cI,\cJ}$  satisfies $\rank(\BU)=\rank(\BX)$ 
    with probability at least $1- \frac{4r}{n^2}$. 
\end{theorem}

Combining Theorems~\ref{thm:cur_iff_condition} and \ref{thm:uniform_sample_is_good}, one can see that an incoherent low-rank matrix can be recovered from its uniformly sampled rows and columns via CUR decomposition. In this sense, CUR decomposition can be viewed as an MC solver  \cite{xu2015cur,cai2016structured}. We call the corresponding sampling model \textit{CUR sampling}, i.e., full observation on the sampled rows and columns. 
Nevertheless, the model of CUR sampling is also too restricted in many applications, especially with larger-scale problems.
{For instance, in the application of larege-scale collaborative filtering, CUR sampling implicitly assumes that some users rate \textit{all} objects and some objects are rated by \textit{all}  users, which is clearly impractical.}

In the era of big data, it is urgent to explore some efficient sampling models that suit various real-world circumstances. While both the uniform sampling and CUR sampling have limits in applications, the blank space between them leads to a more flexible and attainable sampling strategy. 
In this work, we propose a novel sampling model, coined \textit{Cross-Concentrated Sampling} (CCS), to bridge the aforementioned uniform sampling and CUR sampling. {Our approach allows partial observations on selected row and column submatrices, making it much practical in many applications.}

\subsection{Related Work and Contributions}
\subsubsection{Matrix Completion}
The pioneering work \cite{candes2009exact} studies the matrix completion problem with the Bernoulli sampling model. By relaxing the non-convex problem to a convex nuclear norm minimization, it shows that sampling $\cO(rn^{1.2}\log(n))$ entries is sufficient for exact recovery with high probability, which is subsequently improved to $\cO(rn\log^2(n))$ in \cite{chen2015incoherence}. 
Another study \cite{recht2011simpler} focuses on the uniform sampling model, and achieves the same improved sample complexity $\cO(rn\log^2(n))$. 
The standard algorithms for solving the nuclear norm minimization are semidefinite programming \cite{vandenberghe1996semidefinite} and singular value thresholding \cite{cai2010singular}. 
More recently, many non-convex algorithms that aim at the original non-convex problem have also been studied: \cite{jain2010guaranteed,jain2015fast} use the technique of singular value projection (SVP) and provide strong empirical performance; however, the theoretical sample complexity is high as $\cO(r^5n\log^3(n))$.  The works \cite{jain2013low,keshavan2012efficient} are based on alternating minimization and have the sample complexities $\cO(r^{2.5}n\log(n)\log(\frac{1}{\varepsilon}))$ and $\cO(r n\log(n)\log(\frac{1}{\varepsilon}))$ respectively, where $\varepsilon$ is the desired accuracy. Note that the term $\log(\frac{1}{\varepsilon})$ is introduced since \cite{jain2013low,keshavan2012efficient} require iterative resampling. 
A series of work \cite{keshavan2010matrix,sun2016guaranteed,zheng2016convergence,tong2021accelerating} studies the (modified) gradient descent methods on Grassmannian manifold where the sharpest sample complexity is $\cO(r^2n\log(n))$.
\cite{vandereycken2013low,wei2020guarantees} focus on fast Riemannian optimization approaches and achieve the sample complexity $\cO(r^2n\log^2(n))$.
Nevertheless, all these algorithms are designed for Bernoulli or uniform sampling models.

Note that \cite{candes2010power} shows the equivalence between the Bernoulli and uniform sampling models in matrix completion. 
Thus, for the ease of presentation, we only discuss the uniform sampling model in this paper; however, we emphasize that our approach can be easily extended to bridge between Bernoulli sampling and CUR sampling for a similar result.

\subsubsection{CUR Decomposition}
For a given rank-$r$ matrix $\BX\in\mathbb{R}^{n\times n}$, CUR decomposition represents $\BX$ by its submatrices. There are two different versions of CUR decomposition. Set $\BC=[\BX]_{:,\cJ}$ and $\BR=[\BX]_{\cI,:}$. One type of CUR decomposition is of the form \eqref{eq:CUR}. Another version expresses $\BX$ as $\BC\BC^\dagger \BX \BR^\dagger \BR$. The equivalence of these two distinct   CUR decompositions is proved in \cite{HammHuang}. Ensurance of the exact CUR decomposition is equivalent to the condition that $\rank(\BU)=\rank(\BX)$ with $\BU=[\BX]_{\cI,\cJ}$. There are deterministic \cite{AB2013,AltschulerGreedyCSSP,LiDeterministicCSSP} and random \cite{DemanetWu,DKMIII,DMM08,hamm2020stability,DMPNAS,WZ_2013,tropp2009column} methods to select the row and column subsets to form CUR decomposition.   Deterministic sampling needs to sample fewer rows and columns, e.g., $\cO(r)$ to guarantee CUR decomposition but it needs to access the full data and is more computationally costly.  Random sampling is usually computationally cheaper, but it requires more rows and columns.   In the literature, there are three popularly used random sampling distributions: uniform \cite{chiu2013sublinear}, column/row length \cite{DKMIII}, and leverage scores \cite{DMM08}. Compared with the other two, uniform sampling is the easiest and cheapest to implement and does not need to access the full data, but it may fail to provide good results for a generic matrix \cite{HH_PCD2019}. However, as discussed, if the given matrix is incoherent, then uniform sampling can guarantee good performance \cite{chiu2013sublinear,cai2020rapid,cai2021robust,hamm2022riemannian}. 

Although the application of CUR decomposition on MC has already been discussed in \cite{xu2015cur,cai2016structured}, both papers require full observation on the selected row and column submatrices, which, as discussed, is too restricted in some applications.

\subsubsection{Contributions}
This paper bridges the uniform sampling and CUR sampling for matrix completion (MC) problems. Under the commonly used incoherence assumption, we propose a novel sampling model, coined Cross-Concentrated Sampling (CCS). To summarize, our main contributions are as follows:
\begin{enumerate}
    \item We propose a flexible and attainable sampling model, coined CCS, for MC that bridges uniform sampling and CUR sampling (see Procedure~\ref{alg:CCS}). 
    \item We establish a sufficient condition for exact data recovery from the proposed CCS model, specifically $\cO(r^2n\log^2(n))$ samples are sufficient to exactly reconstruct the missing data with a high probability (see Theorem~\ref{thm:sufficient_condition}). 
    \item We design a highly efficient non-convex algorithm, 
    dubbed Iterative CUR Completion (ICURC), for solving CCS-based MC problem (see Algorithm~\ref{Algo:ICURC}). In particular, ICURC costs merely $\cO(nr(|\cI|+|\cJ|))$ flops provided $|\cI|,|\cJ|\ll n$.
    \item We demonstrate the effectiveness and efficiency of the CCS model and the corresponding algorithm on both synthetic and real-world datasets (see Section~\ref{sec:numerical}).
\end{enumerate}

\subsection{Notation}
Given matrices $\BX\in\mathbb{R}^{n\times n}$, $[\BX]_{i,j}$, $[\BX]_{\cI,:}$, $[\BX]_{:,\cJ}$, and   $[\BX]_{\cI,\cJ}$ denote the $(i,j)$-th entry of $\BX$,  the row submatrix with row indices $\cI$,   the column submatrix with column indices $\cJ$, and the submatrix of $\BX$ with row indices $\cI$ and column indices $\cJ$, respectively.  $\|\BX\|_{\fro}:=(\sum_{i,j}[\BX]_{i,j}^2)^{1/2}$ denotes the Frobenius norm of $\BX$, $\|\BX\|_{2,\infty}:=\max_{i}(\sum_{j}[\BX]_{i,j}^2)^{1/2}$ denotes the largest row-wise $\ell_2$-norm, $\|\BX\|_{\infty}=\max_{i,j}|[\BX]_{i,j}|$,  $\BX^\dagger$ represents the Moore–Penrose inverse of $\BX$, and $\BX^\top$ is the transpose of $\BX$. For $\BX,\BY\in\mathbb{R}^{n\times n}$,  {$\langle\BX,\BY\rangle=\sum_{i,j}[\BX]_{i,j}[\BY]_{i,j}$} denotes the Frobenius inner product of $\BX$ and $\BY$.
The symbol $[n]$ denotes the set $\{1,\cdots,n\}$ for all $n\in\mathbb{Z}^+$. $\cI\times [n]$ denotes the set $\{(i,j):i\in\cI,j\in[n]\}$ and  $[n]\times \cJ$ denotes the set $\{(i,j):i\in[n],j\in\cJ\}$. Unless otherwise specified, the term \textit{uniform sampling} refers to uniform sampling with replacement throughout the paper. Additionally, some important symbols are summarized in Table~\ref{tab:imp_notation}.

\begin{table}[h]
    \centering
    \caption{Table of notation.}

    \begin{tabular}{c|c}
    \toprule
  \textsc{Not.}  & \textsc{Description}  \\
         \midrule
         $\cP_{\Omega}$&sampling operator on the set $\Omega$ (see \eqref{eq:sample operator})\\
         %\hline
         $\cI$& row indices for the row submatrix $\BR$\\
         %\hline
          $\cJ$& column indices for the column submatrix $\BC$\\
          %\hline
       $\Omega_{\BR}$  & indices set of the samples on the submatrix $\BR$  \\
         %\hline
       $\Omega_{\BC}$  & indices set of the samples on the submatrix $\BC$ \\
        
         %\hline 
         $\delta$ & percentage of sampled columns or rows
         \\
         %\hline
       $p$  & uniform observation rate on the submatrices \\ 
       %\hline
        $s$  &   overall observation size on the full matrix \\
         $\alpha$  &   overall observation rate on the full matrix \\
         
         \bottomrule

    \end{tabular}
    \label{tab:imp_notation}
\end{table}

\section{Proposed Model} \label{sec:proposed_model}

\begin{figure*}[th]
    \centering
     \subfloat[Uniform Sampling]{\includegraphics[width=0.24\textwidth]{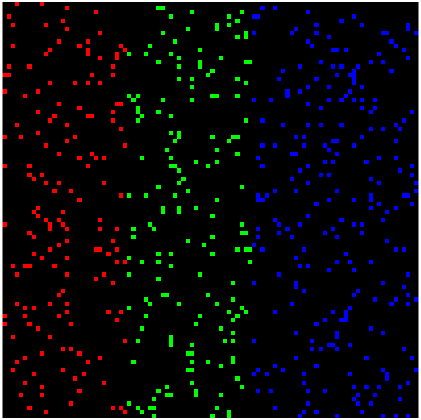}
    }
    \hfill
    \subfloat[CCS--Less Concentrated]{\includegraphics[width=0.24\textwidth]{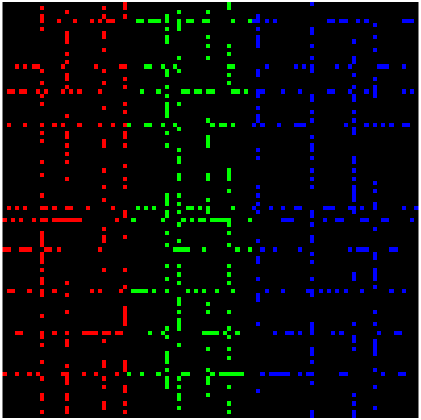}
    }
    \hfill
    \subfloat[CCS--More Concentrated]{\includegraphics[width=0.24\textwidth]{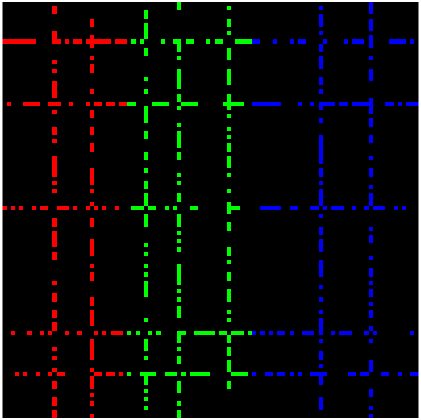}
    }\hfill
      \subfloat[CUR Sampling]{\includegraphics[width=0.24\textwidth]{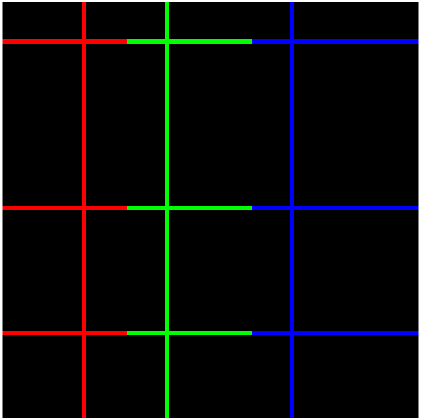}
    }\\
    \caption{Visual illustrations of different sampling schemes. 
    From left to right, sampling methods change from the uniform sampling style to the CUR sampling style with the same total observations rate. Colored pixels indicate  observed entries, while  black pixels mean  missing entries.}
    \label{fig:cur_visual}
\end{figure*}

We aim to design an efficient and effective sampling strategy for varying circumstances in real-world applications. Motivated by the example of collaborative filtering and CUR decomposition, we propose a novel sampling model that samples the entries concentrated on a subset of rows and columns of the original data matrix. Formally, let $\BR=[\BX]_{\cI,:}$ and $\BC=[\BX]_{:,\cJ}$ be selected (by indices sets $\cI$ and $\cJ$) row and column submatrices of the data matrix $\BX$, respectively. Next, we uniformly (with replacement) sample entries on $\BR$ and $\BC$, i.e., the samples are concentrated on the selected row and column submatrices. Since the visualization (Figure~\ref{fig:cur_visual}) of the submatrices $\BR$ and $\BC$ together looks like crosses, we name this sampling model \textit{Cross-Concentrated Sampling} (CCS). 
Moreover, we illustrate CCS against uniform sampling and CUR sampling in Figure~\ref{fig:cur_visual}. One can see that CCS becomes CUR sampling if samples are dense enough to fully observe the submatrices, and CCS becomes uniform sampling if all rows and columns are selected into the submatrices.

We denote the indices sets of the cross-concentrated samples by $\Omega_\BR$ and $\Omega_\BC$ with respect to the notation of the submatrices. The samples in the intersection submatrix $\BU$ are directly inherited from $\Omega_\BR \cup \Omega_\BC$\footnote{We are abusing set notation here. Since we use  ``sampling with replacement'' and thus allow repeated samples, $\Omega_\BR$ and $\Omega_\BC$ are not precisely sets, and $\Omega_\BR \cup \Omega_\BC$ is actually $\Omega_\BR + \Omega_\BC$.}:
\begin{equation}
    \Omega_\BU = \left\{(i,j)\in \Omega_\BR \cup \Omega_\BC~|~ i\in \cI \textnormal{ and } j\in\cJ \right\}.
\end{equation}
Thus, the expected observation rate of $\Omega_\BU$ is the sum of $\Omega_\BR$'s and $\Omega_\BC$'s observation rates. Our task is recovering the underlying rank-$r$ $\BX$ from the observations on $\Omega_\BR \cup \Omega_\BC$: 
\begin{equation} \label{eq:CCS-based MC}
\begin{split}
    \minimize_{\tX}&\quad\frac{1}{2}\langle\cP_{\Omega_\BR \cup \Omega_\BC}(\BX-\tX),\BX-\tX\rangle\\ % + \|\cP_{\Omega_\BC}(\BX-\tX)\|_\fro^2\right)\\
    \subject&\quad \rank(\tX)=r,
\end{split}
\end{equation}
where $\cP$ is the sampling operator defined in \eqref{eq:sample operator}. Note that $\cP$ is not a projection operator if there are repeated samples in the indices set. Hence, $\langle\cP_\Omega(\BX),\BX\rangle$ may not equal to $\|\cP_\Omega(\BX)\|_\fro^2$ in our formula.

\begin{remark}
For successful recovery, the rows and columns we sample from must span the full row and column space of $\BX$. In other words, we require $\rank(\BU)=\rank(\BX)$ where $\BU=[\BX]_{\cI,\cJ}$ is the intersection submatrix of $\BR$ and $\BC$. By Theorem~\ref{thm:uniform_sample_is_good}, this condition can be achieved by uniformly sampling $\cO(r\log(n))$ row and column indices if $\BX$ is incoherent. One may select as low as $\cO(r)$ rows and columns based on prior knowledge in some applications. We consider this as a necessary condition for CCS-based matrix completion.
\end{remark}

\begin{example} \label{exp:Netflix}
Consider the famous Netflix problem where each row of the data matrix represents a user and each column represents a movie \cite{Netflix}. The CCS model randomly selects some users and has them randomly rate a sufficient amount of
movies (perhaps through monetary incentive), then randomly selects some movies and has adequate number of users rate them (perhaps through website promotion).
If CCS has access to some background information of the users, then we can select fewer but representative users from various backgrounds for the survey. Similarly, fewer but representative movies of each category will be promoted for enough ratings. Compare to uniform sampling, CCS is able to collect desired data points much more efficiently. Compared to CUR sampling, CCS is more realistic, as CUR sampling asks all the selected users to rate all existing movies, and all existing users rate all the selected movies.
\end{example}

In summary, we present the CCS model as Procedure~\ref{alg:CCS} for 
low-rank matrices with incoherence.

\floatname{algorithm}{Procedure}
\begin{algorithm}[h]
\caption{Cross-Concentrated Sampling (CCS)}\label{alg:CCS}
\begin{algorithmic}[1]
\State \textbf{Input:}  $\BX$: access to underlying low-rank matrix. 
\State  Uniformly choose row and column indices $\cI,\cJ$.
\State Set $\BR=[\BX]_{\cI,:}$ and $\BC=[\BX]_{:,\cJ}$.
\State Uniformly sample entries in $\BR$ and $\BC$, then record the sampled locations as $\Omega_{\BR}$ and $\Omega_{\BC}$, respectively.
\State\textbf{Output:}  $[\BX]_{\Omega_{\BR}\cup\Omega_{\BC}}$, $\Omega_{\BR}$, $\Omega_{\BC}$, $\cI$, $\cJ$.
\end{algorithmic}
\end{algorithm}

\subsection{Theoretical Results}  \label{sec:theoretical_results}
In this section, we study the %theoretical aspect of 
CCS model from a theoretical perspective. In particular, we provide a sufficient condition to guarantee the uniqueness of solutions with the samples generated by Procedure~\ref{alg:CCS}. 
The proofs are deferred to Section~\ref{sec:proofs}. 

We start with the formal expression of the widely used incoherence assumption. 
\begin{assumption}[$\{\mu_1,\mu_2\}$-incoherence] \label{asm:inco}
Let $\BX\in\R^{n\times n}$  be a rank-$r$ matrix. $\BX$ is $\{\mu_{1},\mu_{2}\}$-incoherent if 
\begin{equation*}
    \left\| {\BW}\right\|_{2,\infty}\leq \sqrt{\frac{\mu_{1} r}{n}} \quad \textnormal{ and } \quad
    \left\|{\BV}\right\|_{2,\infty}\leq \sqrt{\frac{\mu_{2} r}{n}} 
\end{equation*}
for some constants $\mu_{1}$ and $\mu_{2}$, where ${\BW}\bm{\BSigma}{\BV}^\top$ is the compact singular value decomposition (SVD) of $\BX$. In some context, we use the term $\mu$-incoherence where $\mu:=\max\{\mu_1,\mu_2\}$ for simplicity. 
\end{assumption}

Next, we present a variant of \cite[Theorem 3.5]{cai2021robust} that shows how the matrix properties are transformed to its uniformly sampled row submatrix, which is a keystone to our main theorem. Similar results hold for a uniformly sampled column submatrix as well. 

\begin{lemma}\label{COR:UniformIncoherence}
Suppose that $\BX\in\R^{n\times n}$ is a rank-$r$ matrix that satisfies Assumption~\ref{asm:inco}. Let $\kappa$ denote the condition number of $\BX$. Suppose that the indices set $\cI\subseteq[n]$ is chosen by sampling uniformly without replacement {to yield $\BR=[\BX]_{\cI,:}$}. If $|\cI|\geq \mu_1 r^2\log^2(n)$,  
then the following conditions hold with probability at least $1-\frac{r}{n^{0.4r\log(n)}}$: 
\begin{equation*} 
    \begin{split}
        \mu_{1\BR}\leq 4\kappa^2\mu_1, \quad
        \mu_{2\BR}\leq\mu_2,\quad
          \kappa_\BR\leq 2\sqrt{\mu_1 r}\kappa,
    \end{split}
\end{equation*}
where $\{\mu_{1\BR},\mu_{2\BR}\}$ and $\kappa_\BR$ are the incoherence parameters and the condition number of $\BR$ respectively.
\end{lemma}

Now, we are ready to present our main theorem and its  proof  is deferred to Section~\ref{subsec:pf_main_thm}. 

\begin{theorem} \label{thm:sufficient_condition}
Suppose $\BX\in\mathbb{R}^{n\times n}$ is rank-$r$ matrix that satisfies Assumption~\ref{asm:inco}. Let $\kappa$ denote the condition number of $\BX$. Suppose that $\cI,\cJ\subseteq[n]$ are chosen uniformly with replacement to yield $\BR=[\BX]_{\cI,:}$ and $\BC=[\BX]_{:,\cJ}$. Suppose $\Omega_{\BR}$ and $\Omega_{\BC}$ are sampled uniformly with replacement. If 
\begin{equation*}
    \begin{aligned}
        |\cI|&\geq 512\beta\kappa^2r^2\mu_1\mu_2  \log^2(n), \\
        |\cJ|&\geq 512\beta\kappa^2r^2\mu_1\mu_2 \log^2(n), \\
        |\Omega_{\BR}|&\geq 128\beta\kappa^2r^2\mu_1\mu_2(n+|\cI|)\log^2(2n),\\
         |\Omega_{\BC}|&\geq 128\beta\kappa^2r^2\mu_1\mu_2(n+|\cJ|)\log^2(2n), 
    \end{aligned}
\end{equation*}
for some absolute constant $\beta>1$, then $\BX$ can be uniquely determined 
from  $\Omega_{\BR}\cup \Omega_{\BC}$ with probability at least
\[1-\frac{2r}{n^{0.4r\log(n)}}-\frac{2}{n^{2\beta^{0.5}-2}}-\sum_{i=1}^{2}\frac{6\log(n)}{(n+\mu_i r^2\log^2(n))^{2\beta-2}}.
\]
\end{theorem}

 Theorem~\ref{thm:sufficient_condition} shows a sufficient sample complexity for CCS-based MC is $\cO(r^2n\log^2(n))$. 
Compared to the state-of-the-art uniform-sampling-based MC approach \cite{tong2021accelerating} that requires $\cO(r^2n\log(n))$ samples, our result is merely a factor of $\log(n)$ worse under the same incoherence assumption\footnote{Some works achieve better sample complexities by utilizing additional assumption(s) that we do not use. For example, \cite{recht2011simpler,chen2015incoherence} obtain a sample complexity $\cO(rn\log^2(n))$, but also use the additional assumption: $\|\BW\BV^\top\|_\infty\leq \mu_3 \sqrt{r}/n$ for some constant $\mu_3$.}. Moreover, the same condition also guarantees the uniqueness of the solution. 
Essentially, Theorem~\ref{thm:sufficient_condition} states that concentrating the samples into some properly chosen rows and columns will not blow up the required sampling complexity. Hence, based on the circumstance, one can choose how concentrated the samples are, which can potentially simplify the sampling process in some applications, e.g. as we have discussed in Example~\ref{exp:Netflix}.

\section{A Non-Convex Solver} \label{sec:proposed algo}
\begin{figure}[h!]
\centering
    \subfloat[Groundtruth]{\includegraphics[width=0.33\linewidth]{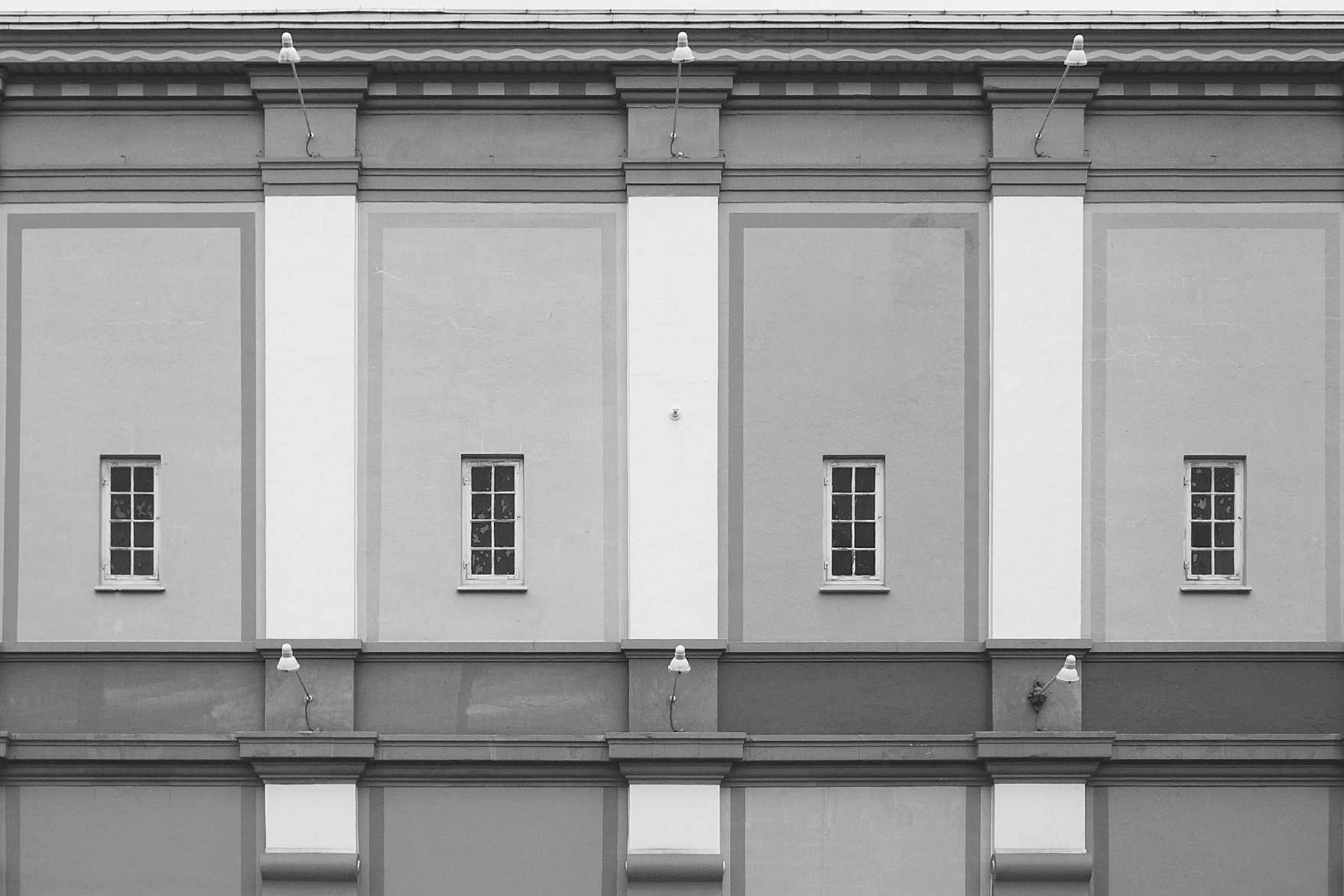}} \hfill
  \subfloat[Observed]{\includegraphics[width=0.33\linewidth]{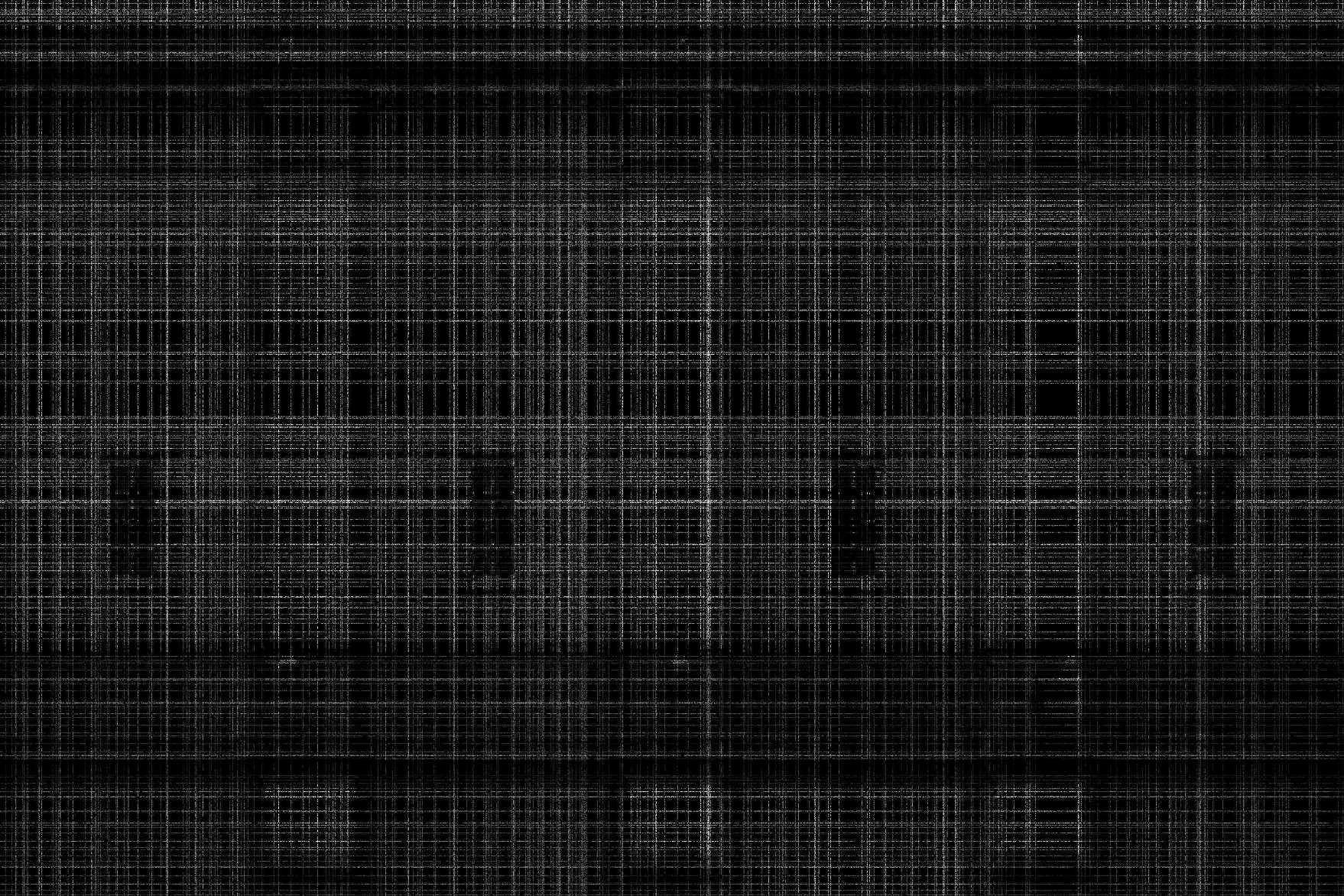}}\hfill
   \subfloat[Reconstruction]{\includegraphics[width=0.33\linewidth]{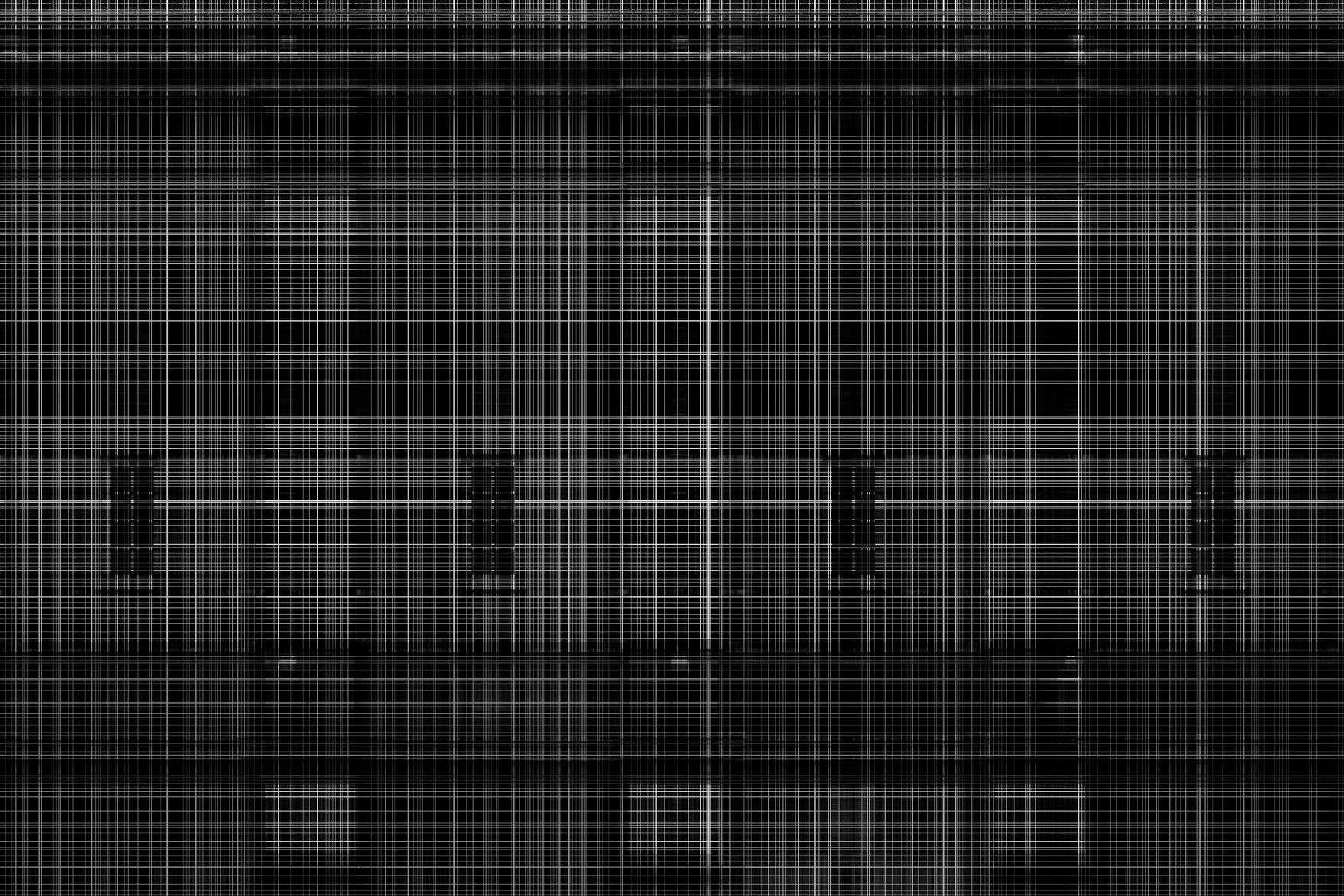}}
    \caption{Visual results for image inpainting  from    the CCS-based  samples   via ScalePGD algorithm \cite{tong2021accelerating}. See a more detailed setting in Section~\ref{subsec:three_image}.
    }
    \label{fig:GMC_on_CCS}
\end{figure}

In this section, we discuss how to effectively and efficiently solve the CCS-based MC problem. First, we consider directly applying the existing uniform-sampling-based MC algorithms. Unfortunately, it turns out that the uniform-sampling-based algorithms are not suitable for the CCS model. For example, as shown in Figure~\ref{fig:GMC_on_CCS}, we apply the state-of-the-art ScaledPGD \cite{tong2021accelerating} on a CCS-based image recovery problem and the result is not visually desirable. 
Therefore, we must develop new algorithm(s) for the proposed CCS model. 

Recall that CCS samples uniformly on the selected row and column submatrices. Thus, a natural and simple approach is solving the MC problem in two steps: 
\begin{enumerate}
    \item Applying certain off-the-shelf MC algorithms to recover the submatrices $\BR$ and $\BC$ separately. Note that $\Omega_\BR$ and $\Omega_\BC$ are uniformly sampled in $\BR$ and $\BC$. Thus, any uniform-sampling-based MC solver will work, provided enough samples.
    \item Applying the standard CUR decomposition, i.e., \eqref{eq:CUR}, to recover $\BX$ from the fully reconstructed $\BR$ and $\BC$.
\end{enumerate}
This approach is named Two-Step Completion (TSC) and is later summarized as Algorithm~\ref{ALG:TS_MC} in Section~\ref{sec:proofs}. However, this algorithm is not utilizing the samples in $\BC$ when recovering $\BR$ and vice versa. Thus, one does not expect TSC to be a highly effective solver for \eqref{eq:CCS-based MC}.

To take full advantage of the CCS structure, we propose a novel non-convex algorithm, coined Iterative CUR Completion (ICURC), for the CCS-based MC problem. ICURC is built upon the framework of projected gradient descent \cite{jain2010guaranteed}. 
Notice that 
\begin{equation}
    \begin{aligned}
        &~\mathbb{E}(\cP_{\Omega_\BR\cup\Omega_{\BC}}(\BX))\\
       = &~\mathbb{E}(\cP_{\Omega_\BR}(\BX))+\mathbb{E}(\cP_{\Omega_{\BC}}(\BX))\\
        =&~p_1\cP_{\cI\times [n] }(\BX)+ p_2\cP_{[n]\times \cJ}(\BX)\\
        =&~p_1\cP_{\cI\times \cJ^c }(\BX)+p_2\cP_{\cI^c\times \cJ}(\BX)+(p_1+p_2)\cP_{\cI\times\cJ}(\BX),
    \end{aligned}
\end{equation}
where $p_1=\frac{|\Omega_{\BR}|}{n|\cI|}$, $p_2=\frac{|\Omega_{\BC}|}{n|\cJ|}$, $\cJ^c=[n]\setminus\cJ:=\{j\in[n]:j\not\in \cJ\}$, and $\cI^c=[n]\setminus\cI:=\{i\in[n]:i\not\in \cI\}$. Therefore, in each iteration, we  divide the updating of  $\BR$ and  $\BC$ into three part: $[\BR]_{:,\cJ^c}$, $[\BC]_{\cI^c,:}$, and $\BU$. Additionally, to enforce the rank constrain on $\BX$, we project the updated $\BU$ to its best rank-$r$ approximation.  Specifically, we perform a step of gradient descent on $[\BR]_{:,\cJ^c}$, $[\BC]_{\cI^c,:}$, and $\BU$    via the following formula:
\begin{equation}\label{eq:update RC}
\begin{aligned}
     [\BR_{k+1}]_{:,\cJ^{c}}&:=[\BX_k]_{\cI,\cJ^c}+\eta_R[\cP_{\Omega_\BR}(\BX-\BX_k)]_{\cI,\cJ^c}, \\
     [\BC_{k+1}]_{\cI^{c},:}&:=[\BX_k]_{\cI^c,\cJ}+\eta_C[\cP_{\Omega_\BC}(\BX-\BX_k)]_{\cI^c,\cJ},
     \end{aligned}
\end{equation}
and
\begin{equation}\label{eq:update U}
\begin{aligned}
     \BU_{k+1}&:=\cH_r\left([\BX_{k}]_{\cI,\cJ}+\eta_U[\cP_{\Omega_\BR\cup\Omega_{\BC}}(\BX-\BX_k)]_{\cI,\cJ}  \right),
\end{aligned}
\end{equation}
where $\eta_R, \eta_C, \eta_U>0$ are the step sizes, and $\cH_r$ is the rank-$r$ truncated SVD operator. We then set  $[\BR_{k+1}]_{:,\cJ}=[\BC_{k+1}]_{\cI,:}=\BU_{k+1}$.  Hence, updated via CUR decomposition, the approximated data matrix 
\begin{equation} \label{eq:update X}
    \BX_{k+1}=\BC_{k+1}\BU_{k+1}^\dagger\BR_{k+1}
\end{equation}
is also rank-$r$. Although CUR decomposition is not the most accurate low-rank approximation, it is close enough for our iterative algorithm.

With the initial guess $\BX_0=\bm{0}$, ICURC runs the above steps iteratively until convergence. In particular, we set the stopping criterion to be $e_k\leq \varepsilon$ where 
\begin{equation} \label{eq:e_k}
    e_k=\frac{\langle\cP_{\Omega_\BR\cup\Omega_\BC}(\BX-\BX_k),\BX-\BX_k\rangle}{\langle\cP_{\Omega_\BR\cup\Omega_\BC}(\BX),\BX\rangle}
\end{equation}
and $\varepsilon$ is the targeted accuracy. 
We summarize the proposed non-convex algorithm as Algorithm~\ref{Algo:ICURC}.

\floatname{algorithm}{Algorithm}
\begin{algorithm}[h]
\caption{Iterative CUR Completion (ICURC) for CCS} \label{Algo:ICURC}
\begin{algorithmic}[1]
\State \textbf{Input:} 
$[\BX]_{\Omega_{\BR}\cup\Omega_{\BC}}$: observed data; 
$\Omega_\BR, \Omega_\BC$: observation locations; 
$\cI,\cJ$: row and column indices that define $\BR$ and $\BC$ respectively; $\eta_R, \eta_C, \eta_U$: step sizes; 
$r$: target rank;
$\varepsilon$: target precision level. 

\State $\BX_0=\bm{0}$
\State $\cI^{c}=[n]\setminus\cI$, $\cJ^{c}=[n]\setminus\cJ$, $k=0$

\While {$e_k > \varepsilon$} \textcolor{officegreen}{\algorithmiccomment{$e_k$ is defined in \eqref{eq:e_k}}}
    \State  $[\BR_{k+1}]_{:,\cJ^{c}}=[\BX_k]_{\cI,\cJ^c}+\eta_R[\cP_{\Omega_\BR}(\BX-\BX_k)]_{\cI,\cJ^c}$
    \State $[\BC_{k+1}]_{\cI^{c},:}=[\BX_k]_{\cI^c,\cJ}+\eta_C[\cP_{\Omega_\BC}(\BX-\BX_k)]_{\cI^c,\cJ}$
    \State $\BU_{k+1}=\cH_r\left([\BX_{k}]_{\cI,\cJ}+\eta_U[\cP_{\Omega_\BR\cup\Omega_{\BC}}(\BX-\BX_k)]_{\cI,\cJ}  \right)$
       \State $[\BR_{k+1}]_{:,\cJ}=\BU_{k+1}$
       \State $[\BC_{k+1}]_{\cI,:}=\BU_{k+1}$
    \State $\BX_{k+1}$ = $\BC_{k+1}\BU_{k+1}^{\dagger}\BR_{k+1}$
    \textcolor{officegreen}{\algorithmiccomment{Do not compute, see \eqref{eq:fast compute}}}
\State $k = k + 1$
\EndWhile
\State \textbf{Output:} 
$\BC_k,\BU_{k},\BR_k$: CUR components of $\BX$. 
\end{algorithmic}
\end{algorithm}

\begin{remark} \label{remark:step size}
Inspired by \cite[Theorem~7]{recht2011simpler}, we recommend the step size $\eta_R=\frac{1}{p_1}$, $\eta_C = \frac{1}{p_2}$, and $\eta_U=\frac{1}{p_1+p_2}$
for Algorithm~\ref{Algo:ICURC}, where $p_1$ and $p_2$ are the observation rates of $\Omega_\BR$ and $\Omega_\BC$, respectively. 
\end{remark}

\subsection{Efficient Implementation}
We provide the implementation details and the breakdown of the computational costs for Algorithm~\ref{Algo:ICURC}. 
The steps of updating $[\BR_{k+1}]_{\cI,\cJ^c}$ and $[\BC_{k+1}]_{\cI^c,\cJ}$, i.e., \eqref{eq:update RC}, cost $\cO(|\Omega_\BR|+|\Omega_\BC|-|\Omega_{\BU}|)$ flops as we only update the observed locations. Note that the intersection matrix $\BU$ is ${|\cI|\times|\cJ|}$. Thus, in \eqref{eq:update U} and \eqref{eq:update X}, computing the truncated SVD and pseudo-inverse of $\BU_{k+1}$ costs $\cO(r|\cI||\cJ|)$ flops. Calculating $\BX_{k+1}$ in \eqref{eq:update X} seems expensive at first sight; fortunately, we do not actually have to form the whole $n\times n$ matrix. Looking back at \eqref{eq:update RC} and \eqref{eq:update U}, we only use the selected rows and columns from $\BX_k$ and they can be efficiently obtained via
\begin{equation} \label{eq:fast compute}
    \begin{split}
        [\BX_k]_{\cI,\cJ^c}&= [\BC_k]_{\cI,:}\BU_k^\dagger[\BR_k]_{:,\cJ^c}=\BU_k\BU_k^\dagger[\BR_k]_{:,\cJ^c},\\
        [\BX_k]_{\cI^c,\cJ}&= [\BC_k]_{\cI^c,:}\BU_k^\dagger[\BR_k]_{:,\cJ}=[\BC_k]_{\cI^c,:}\BU_k^\dagger\BU_k,\\
        [\BX_k]_{\cI,\cJ}&=[\BC_k]_{\cI,:}\BU_k^\dagger[\BR_k]_{:,\cJ}=\BU_k.
    \end{split}
\end{equation}
The computational complexity of \eqref{eq:fast compute} is $\cO(nr(|\cI|+|\cJ|))$. Finally, computing the stopping criterion  $e_k$, i.e., \eqref{eq:e_k}, costs $\cO(|\Omega_\BR|+|\Omega_\BC|)$ flops. 

Overall, ICURC costs total $\cO(nr(|\cI|+|\cJ|))$  flops per iteration. 
Moreover, \eqref{eq:fast compute} also suggests that we only need to pass the updated CUR components (i.e., $[\BR_k]_{:,\cJ^c}$, $[\BC_k]_{\cI^c,:}$, and $\BU_k$) to the next iteration instead of the much larger $\BX_{k}$. Hence, we conclude that ICURC is computationally and memory efficient provided $|\cI|, |\cJ|\ll n$. 

\begin{remark}\label{rmk:linear cong}Empirically, we observe that ICURC achieves a linear convergence rate when the samples are concentrated on a smaller collection of rows and columns. For the interested reader, we include several numerical results for ICURC's convergence behavior in the appendix.
\end{remark}

\section{Numerical Experiments} \label{sec:numerical}

 In this section, we compare the empirical performance of our CCS-model-based ICURC, i.e., Algorithm~\ref{Algo:ICURC}, against the state-of-the-art algorithms (ScaledPGD \cite{tong2021accelerating} and SVP \cite{jain2010guaranteed}) based on uniform sampling method. The related codes are provided in \url{https://github.com/huangl3/CCS-ICURC}.  
All experiments are implemented on Matlab R2020a and executed on a Linux workstation equipped with Intel i9-9940X CPU (3.3GHz @ 14 cores) and 128GB DDR4 RAM. We emphasize again that our approach is computationally and memory efficient. All our experiments can be reproduced with a basic laptop.

\subsection{Synthetic Examples} \label{subsec:synthetic}

In matrix completion, an important question is how many measurements are needed for an algorithm to ensure a reliable reconstruction of a low-rank matrix. Thus, we investigate the recoverability of the ICURC algorithm based on the CCS model in the framework of phase transitions: 
  \begin{enumerate}
      \item   Phase transition that studies the  required  sampling rate over the whole matrix  based on  different sizes of the selected row and column submatrices and different uniform sampling rates on the selected submatrices.
      \item   Phase transition that explores  the required measurements over different sizes of the original problem.
  \end{enumerate}
On the synthetic simulation, we only consider square problems and consider a low-rank matrix $\BX\in\mathbb{R}^{n\times n}$ of rank $r$. Thus, we sample the same number of rows and columns in the following simulations, i.e., $|\cI|=|\cJ|$.  
{In comparison, we have also compared our results with that of SVP and ScaledPGD based on a uniform sampling model.}  

\subsubsection{Empirical Phase Transition}
\begin{figure*}[th]
    \centering

    \subfloat{\includegraphics[width=0.333\textwidth]{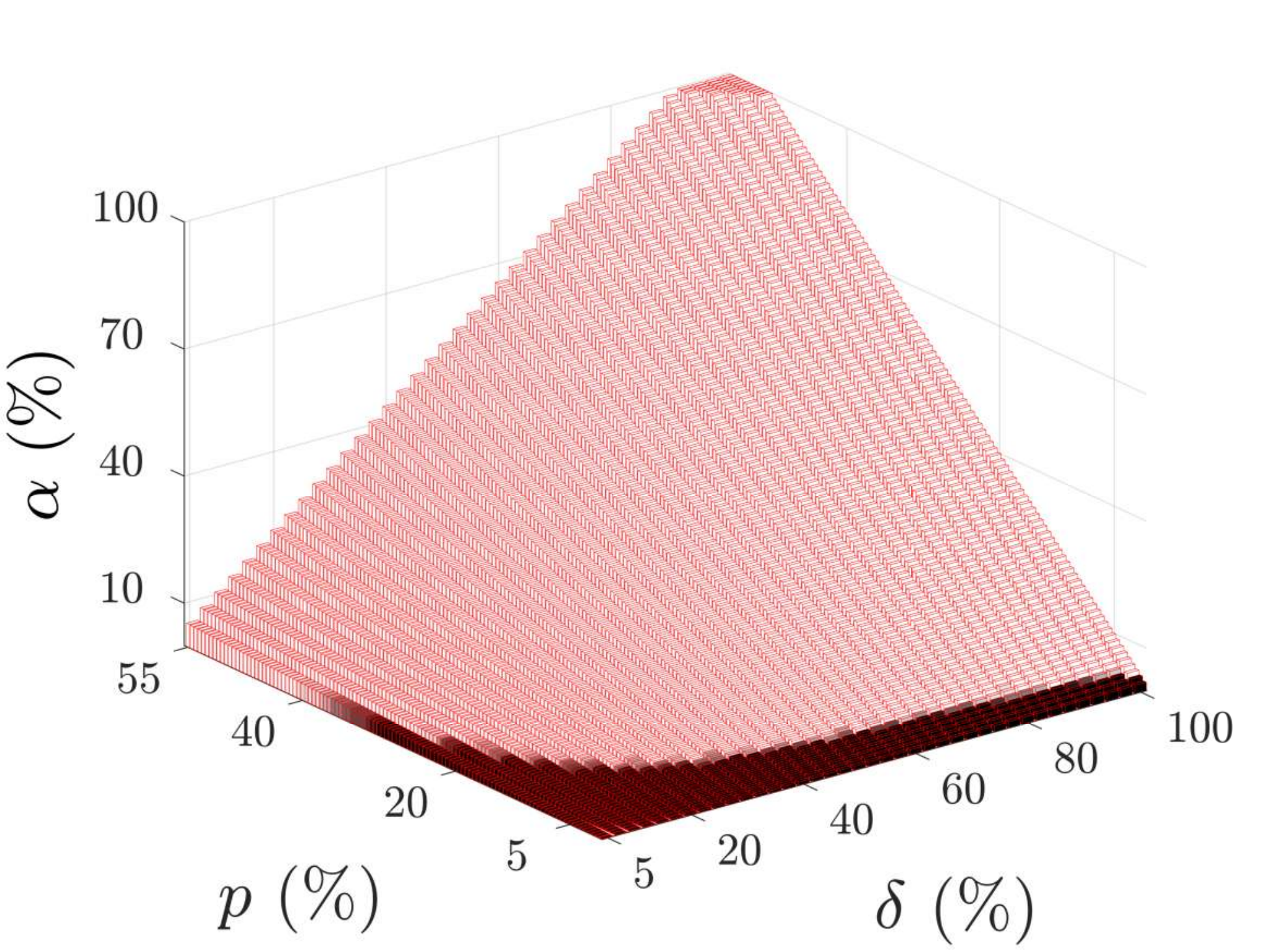}}
    \subfloat{\includegraphics[width=0.333\textwidth]{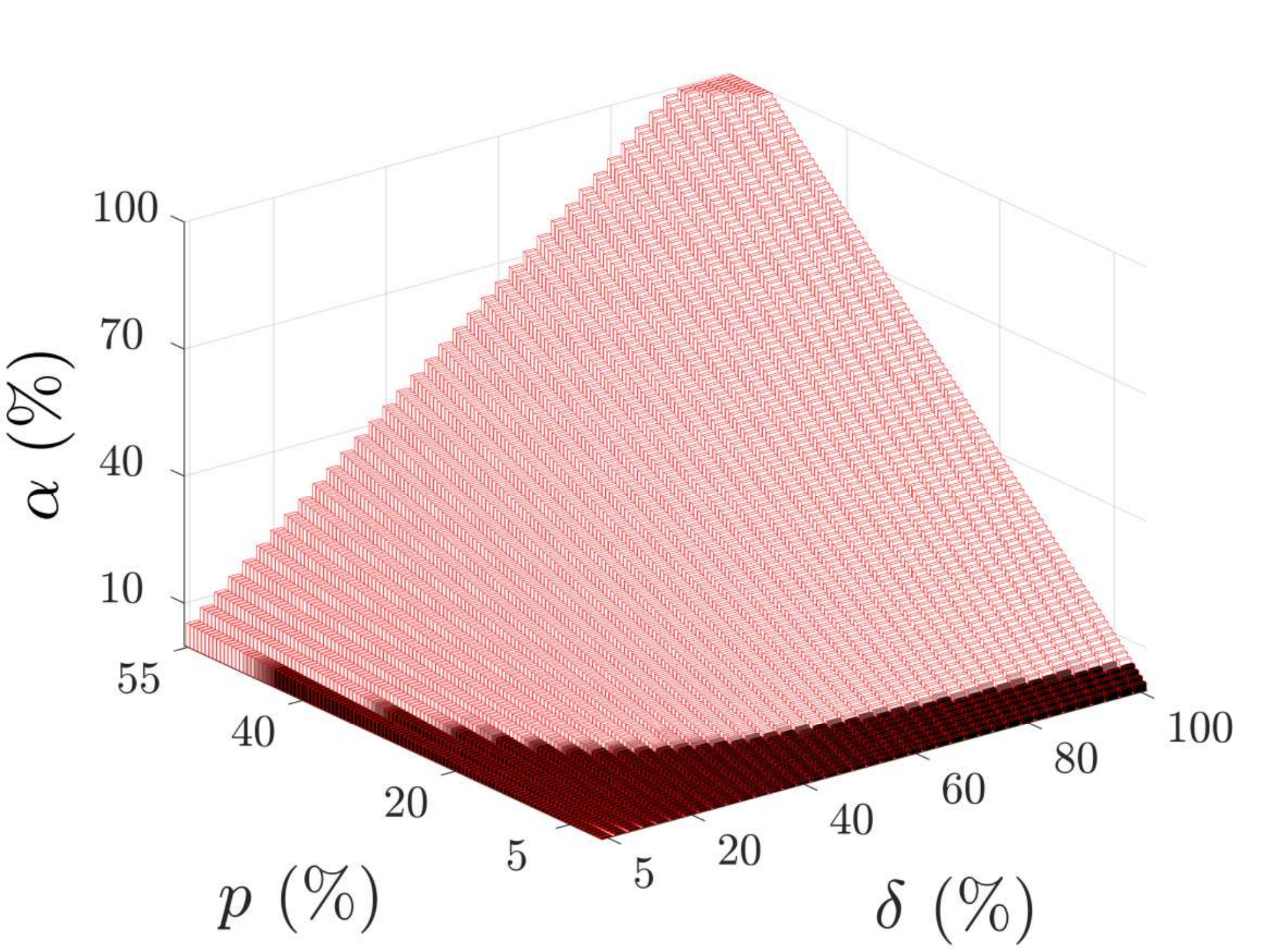}}
    \subfloat{\includegraphics[width=0.333\textwidth]{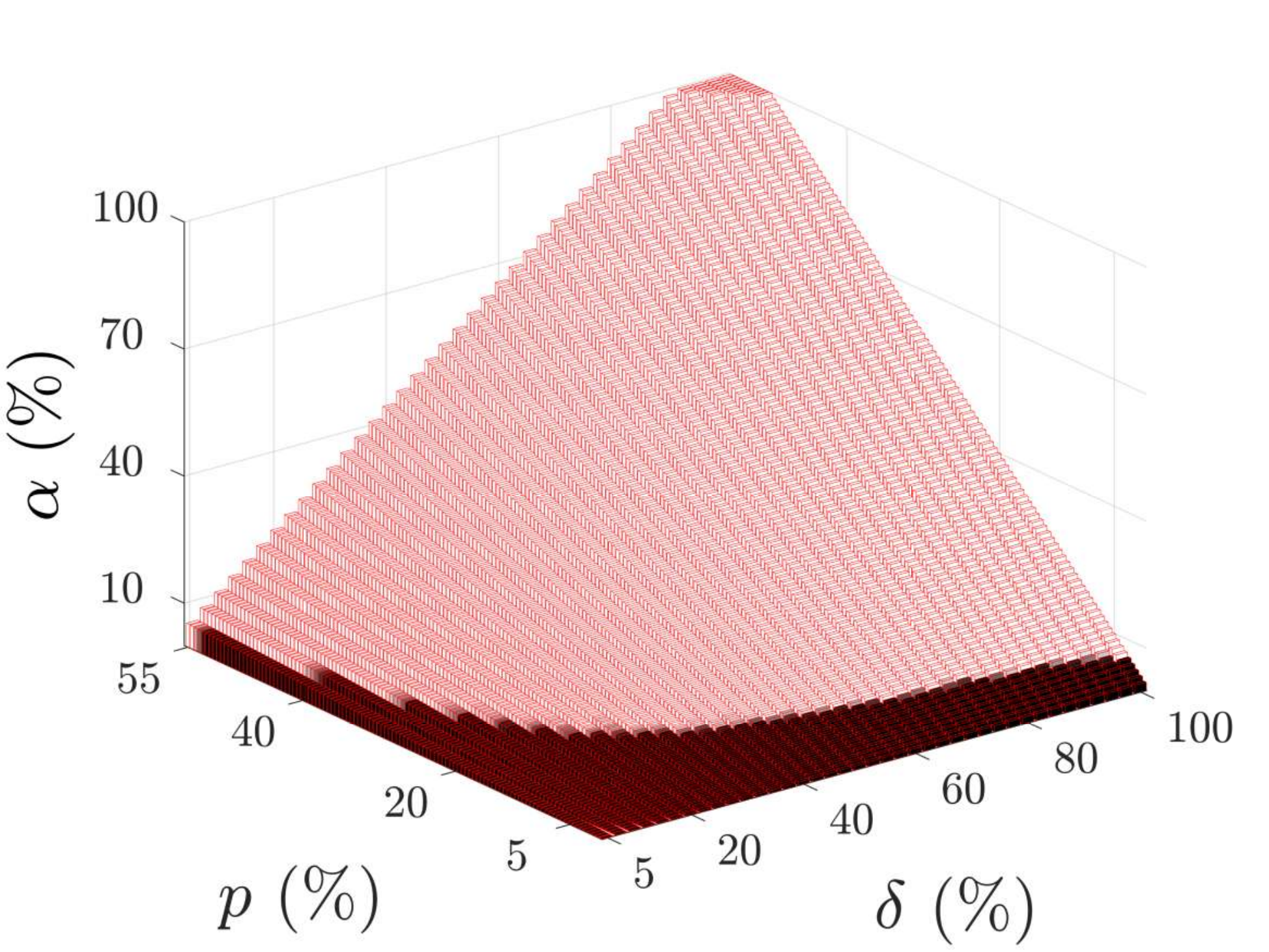}} \\
    \subfloat{\includegraphics[width=0.333\textwidth]{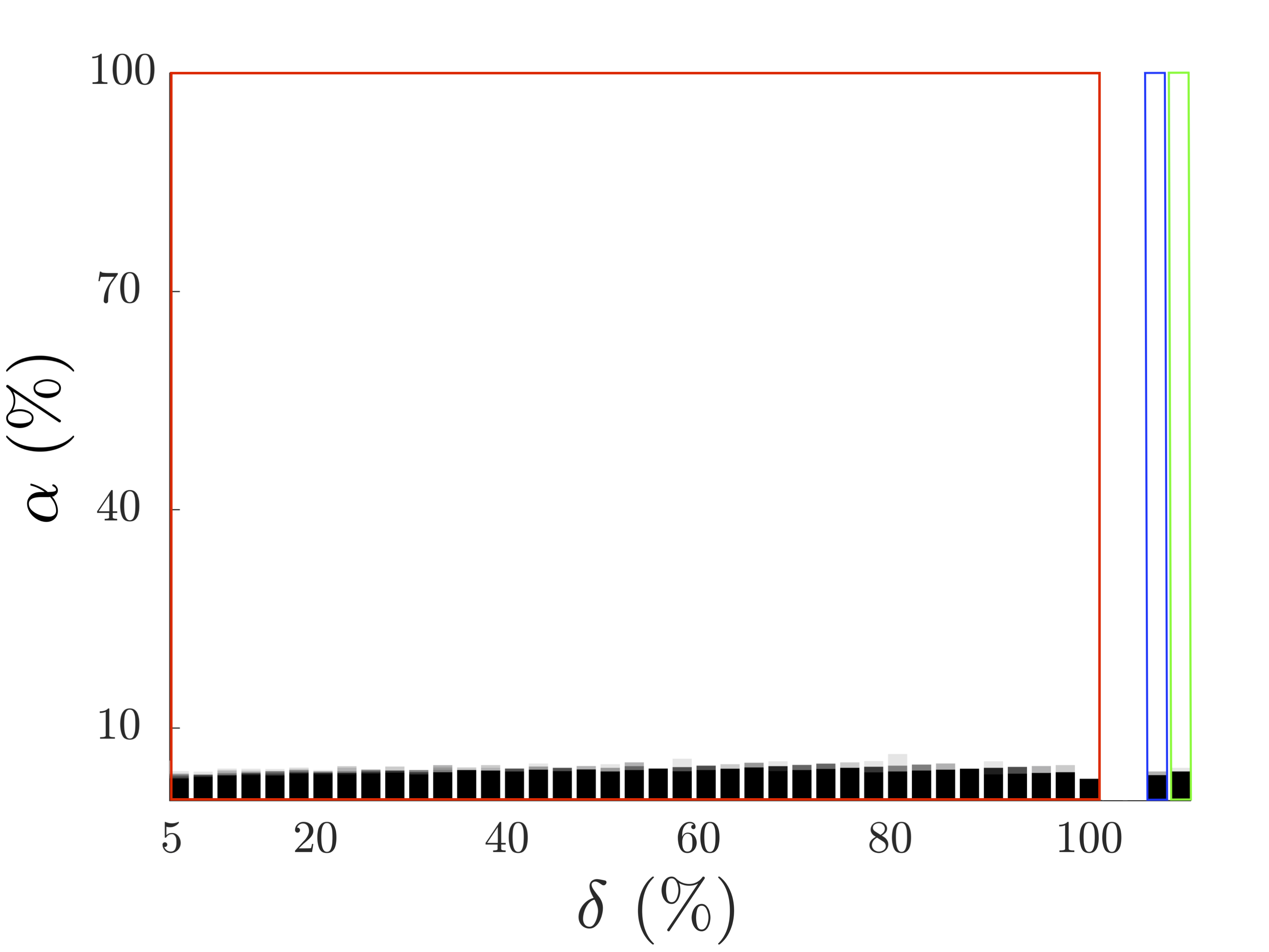}}
    \subfloat{\includegraphics[width=0.333\textwidth]{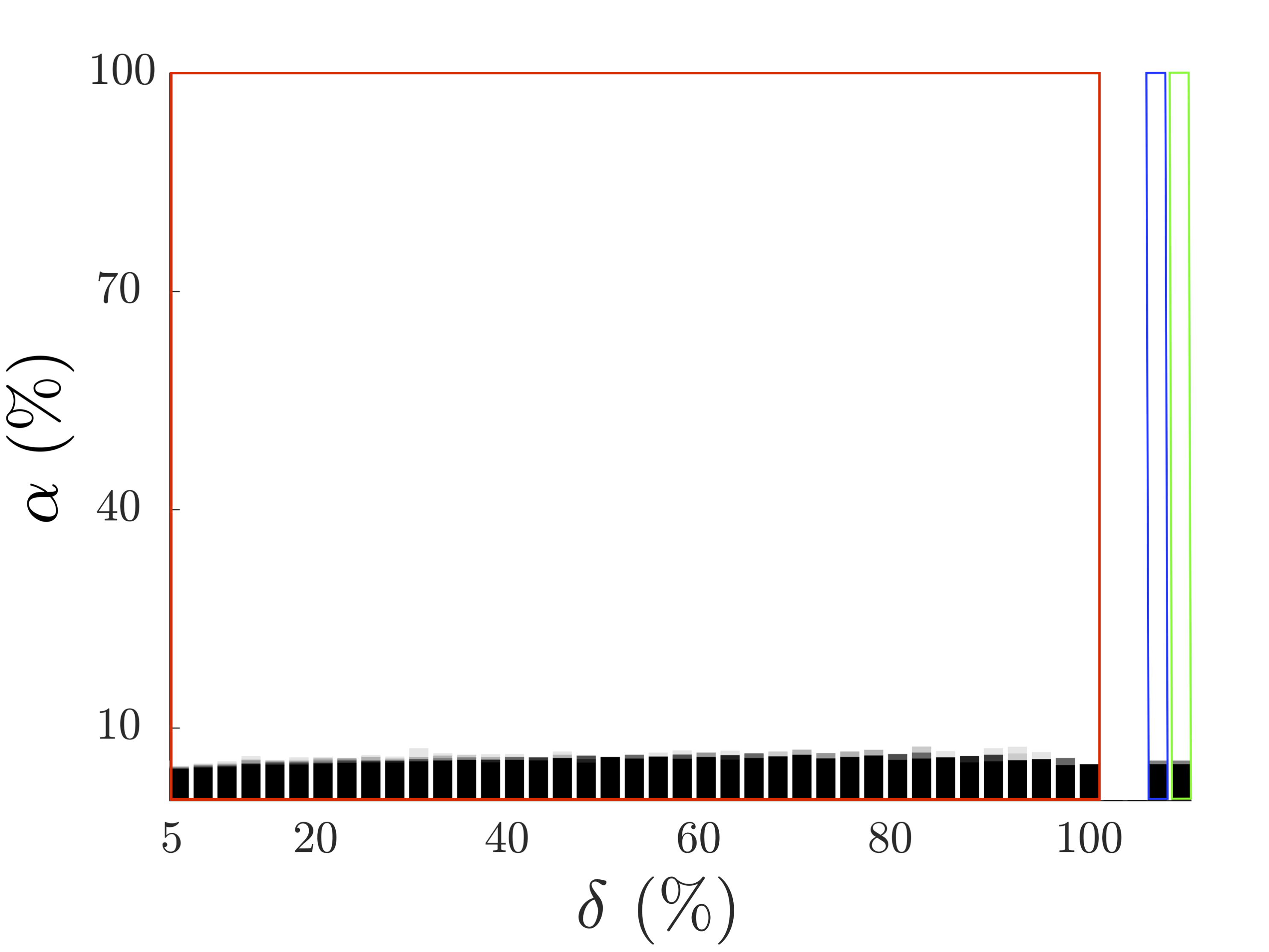}}
    \subfloat{\includegraphics[width=0.333\textwidth]{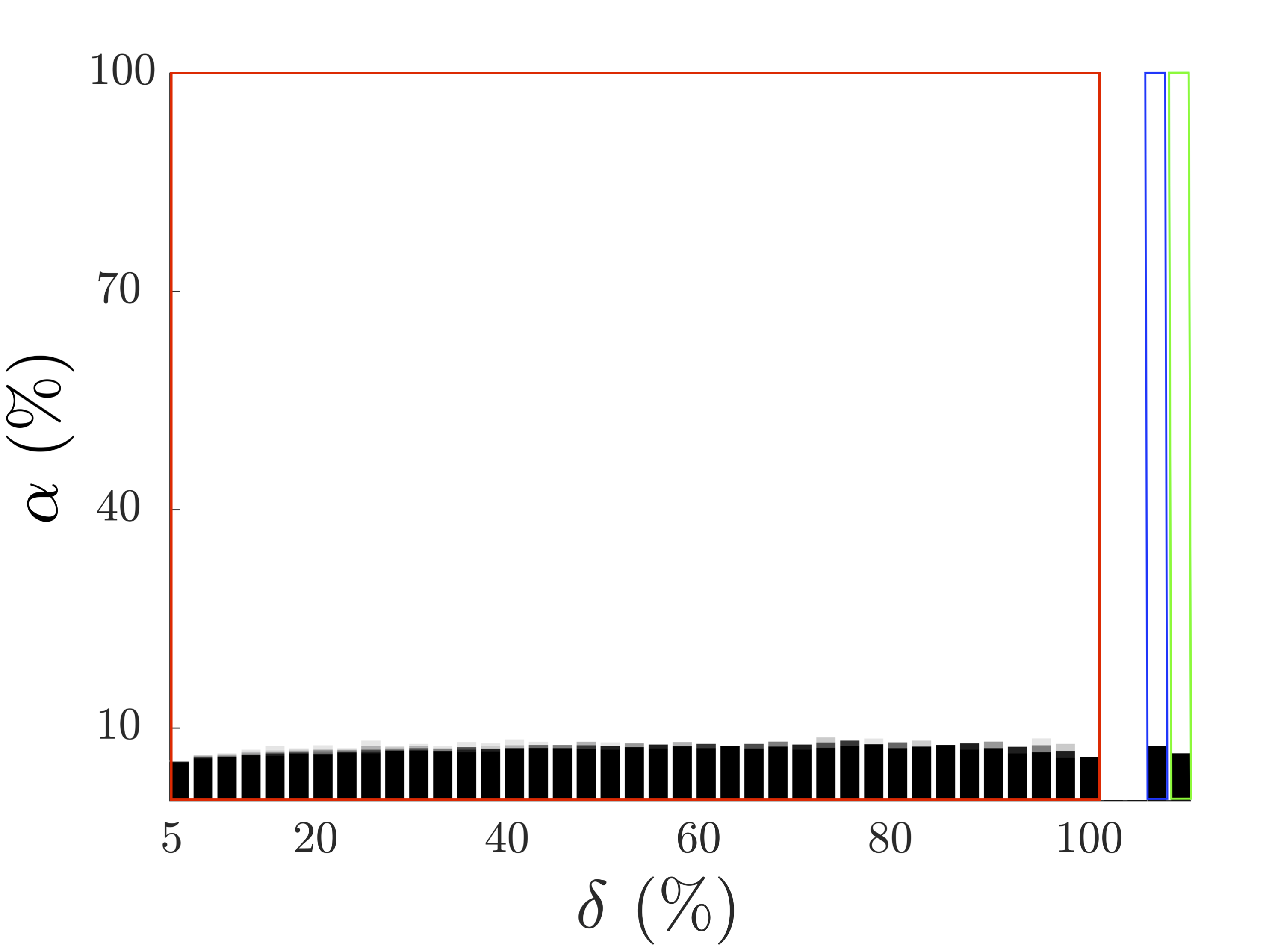}} 
 \\
     \subfloat{\includegraphics[width=0.36\textwidth]{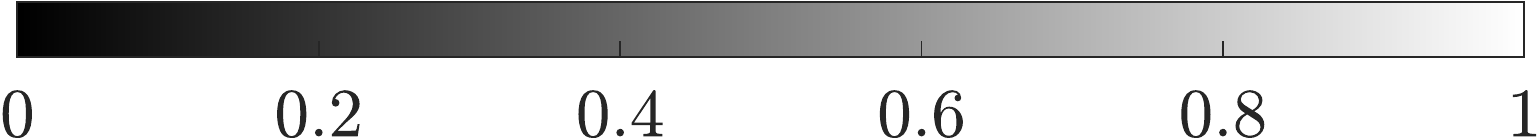}} 
      \caption{Empirical phase transition in the overall sampling rate $\alpha$, the percentage of selected rows and columns $\delta$ , and uniform sampling rates on the selected submatrices $p$. 
     {\textbf{Row 1}: 3D-view of the empirical phase transition of ICURC. 
     \textbf{Row 2}: 2D view of empirical phase transition of ICURC (in the red box), ScaledPGD (in the blue box), and SVP (in the green box).  \textbf{Left}: $r=5$. \textbf{Middle}: $r=10$. \textbf{Right}: $r=15$. One can see that as rank increases, the required overall sampling rate increases correspondingly. Additionally, the CCS model provides flexibility in obtaining a sufficient amount of data to ensure completing the missing data  successfully  and the  performance of the ICURC algorithm from the CCS-based samples is comparable to that of the state-of-the-art algorithms (SVP and ScaledPGD) from the   uniform-sampling-based samples. 
      }
      }
    \label{fig:both_phtr}
\end{figure*}

\begin{figure*}[th]
    \centering
    \includegraphics[width=0.28\textwidth]{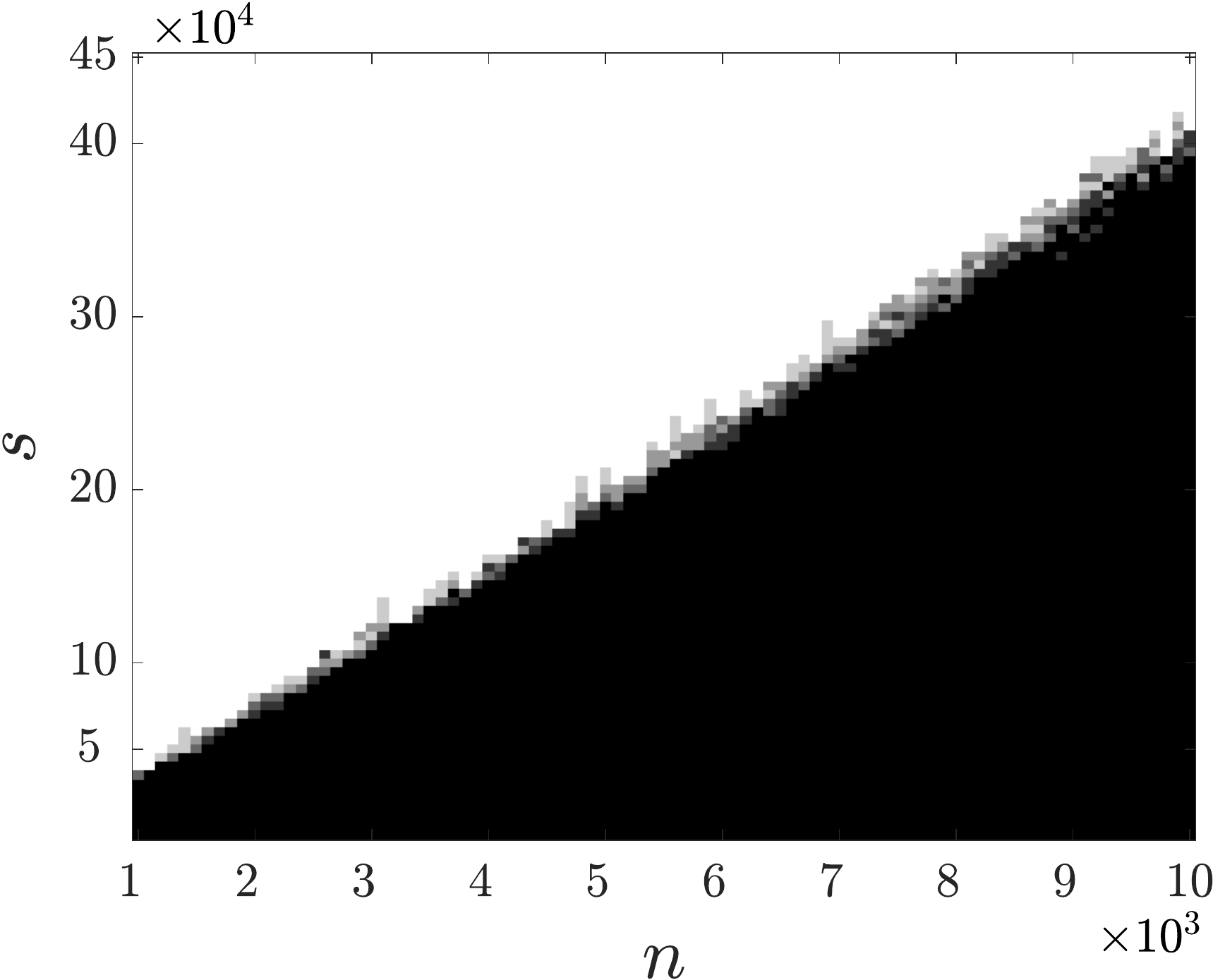}\qquad\quad
    \includegraphics[width=0.28\textwidth]{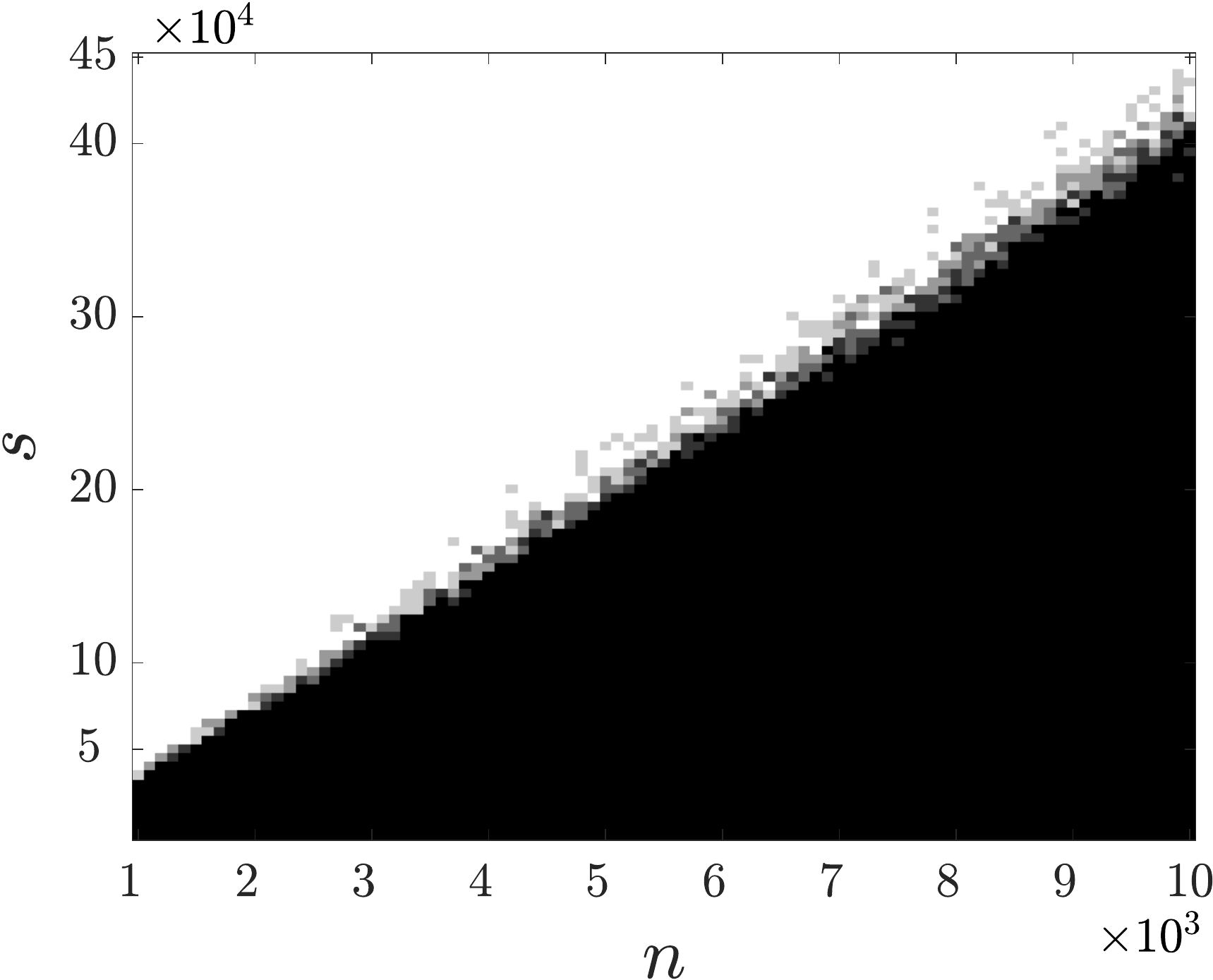} \qquad\quad
    \includegraphics[width=0.28\textwidth]{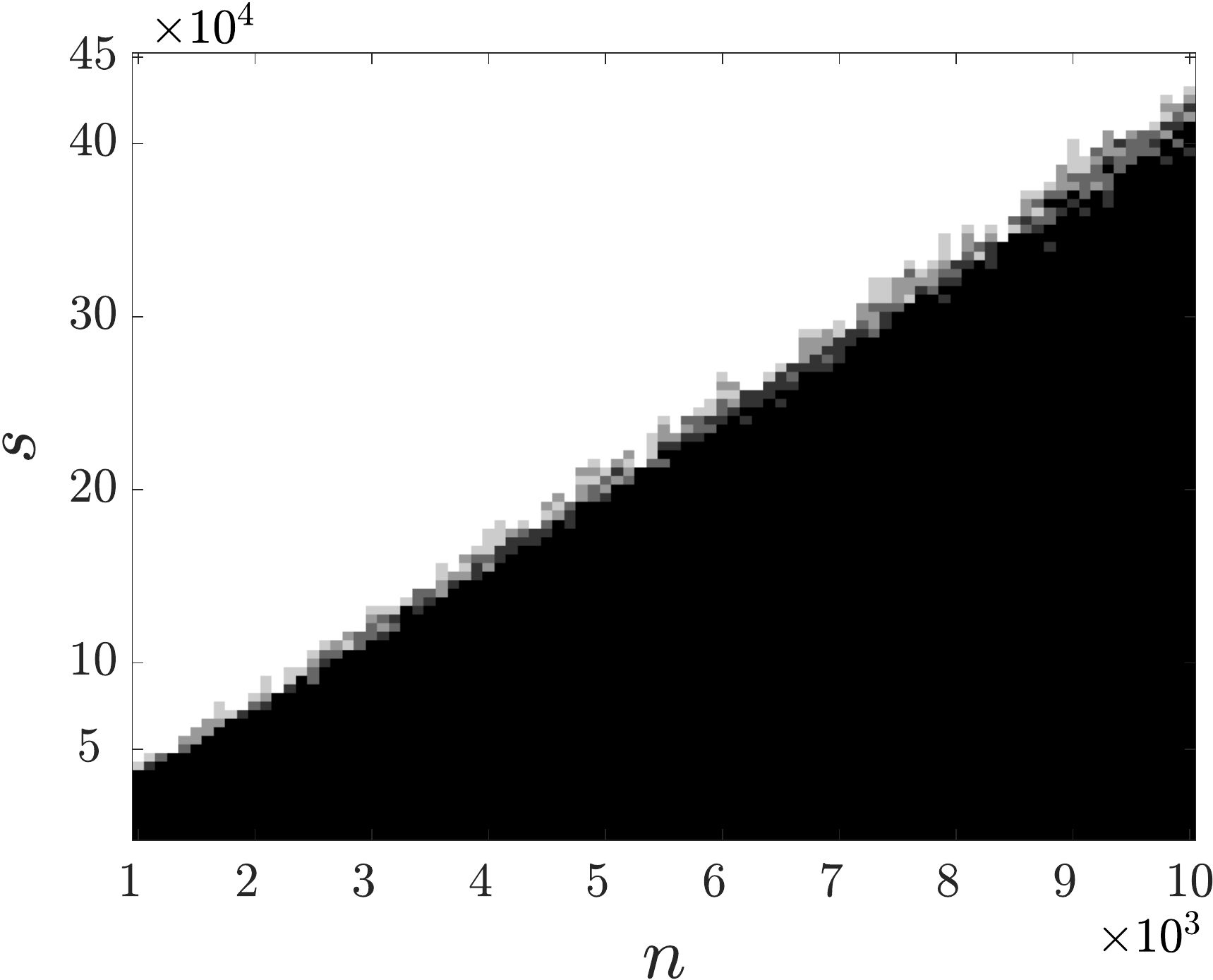}     
    \\
    
    \vspace{3mm}
    
    \includegraphics[width=0.28\textwidth]{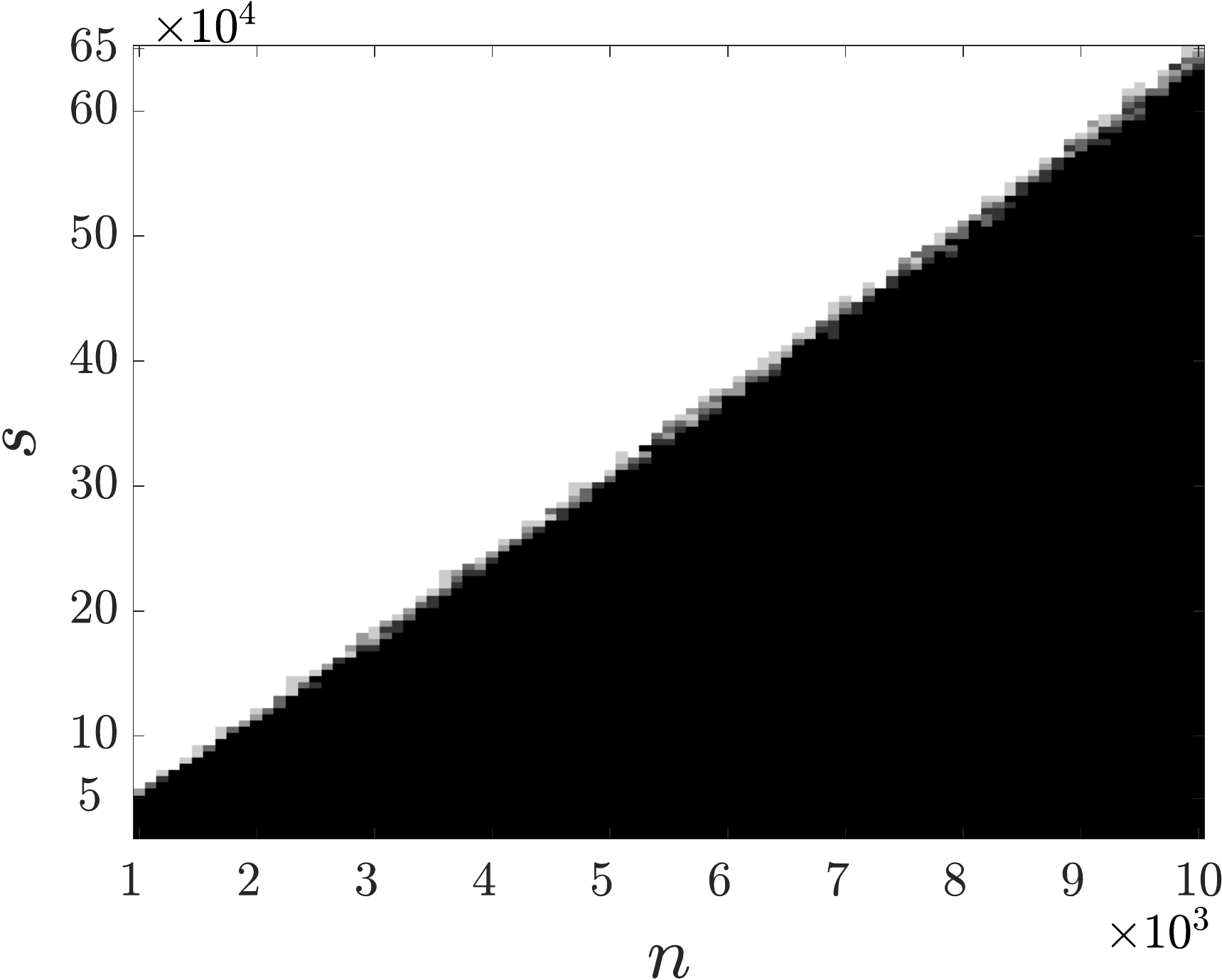}\qquad\quad
    \includegraphics[width=0.28\textwidth]{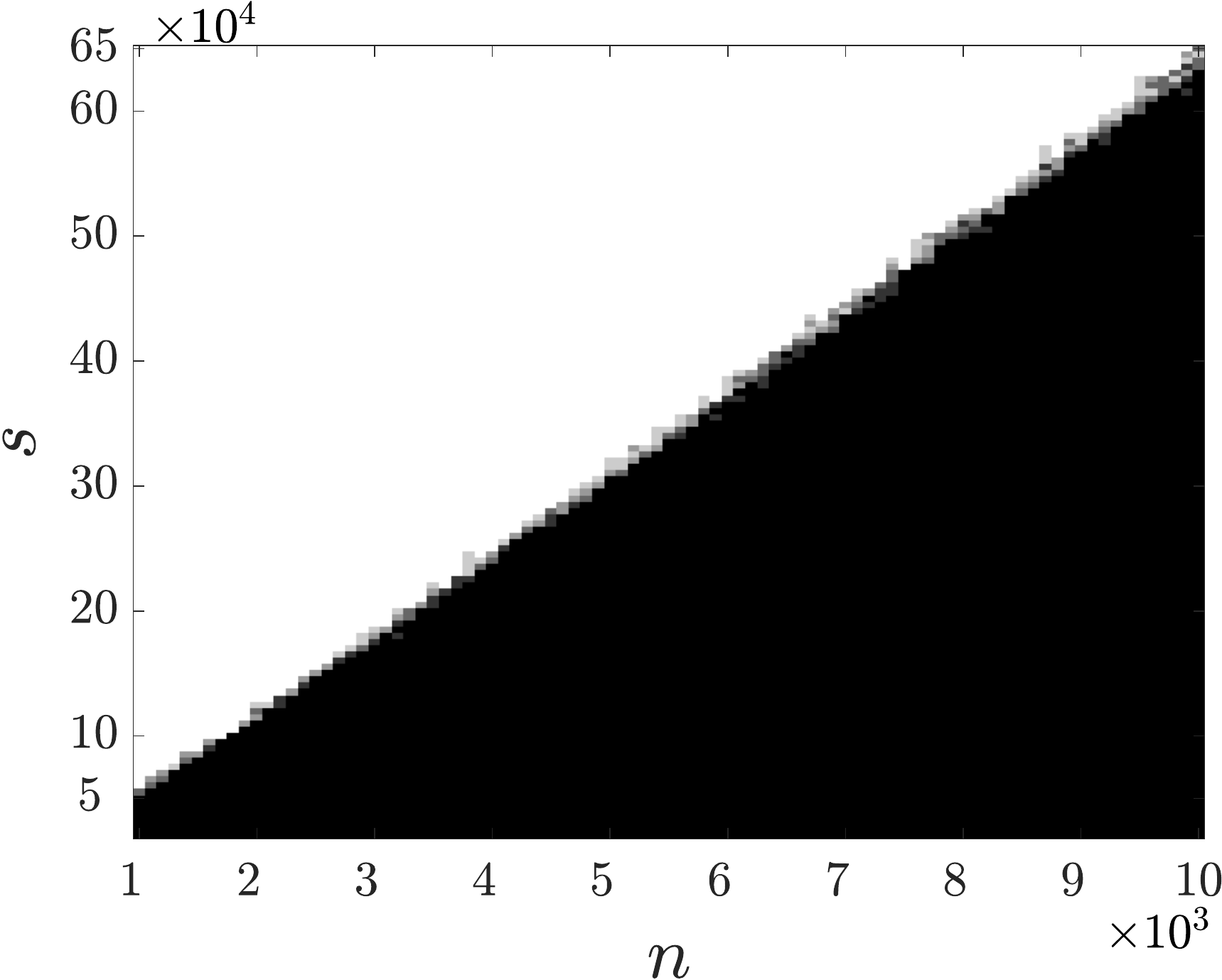}\qquad\quad
    \includegraphics[width=0.28\textwidth]{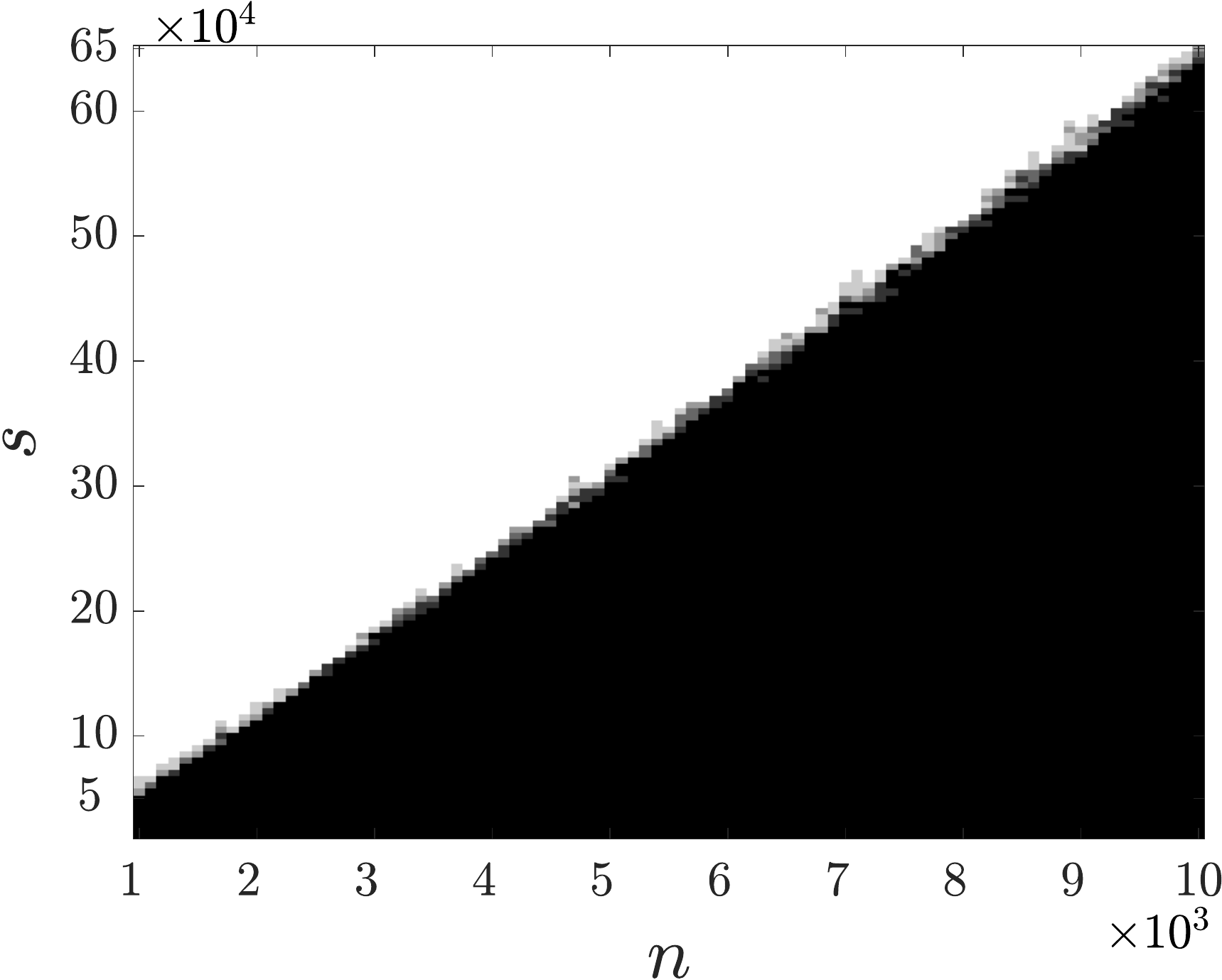}
    
    \vspace{3mm}
    
    \includegraphics[width=0.28\textwidth]{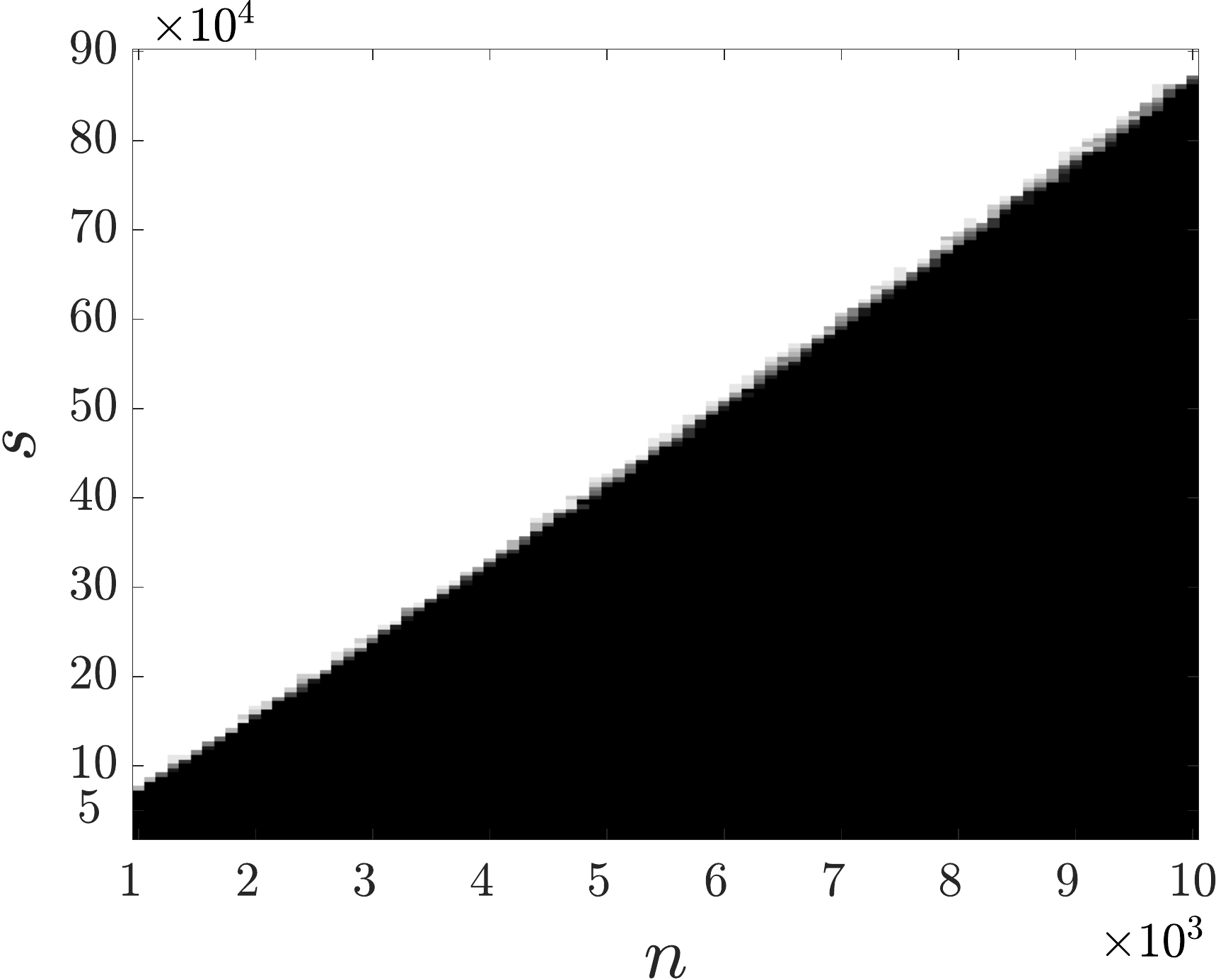}\qquad\quad
    \includegraphics[width=0.28\textwidth]{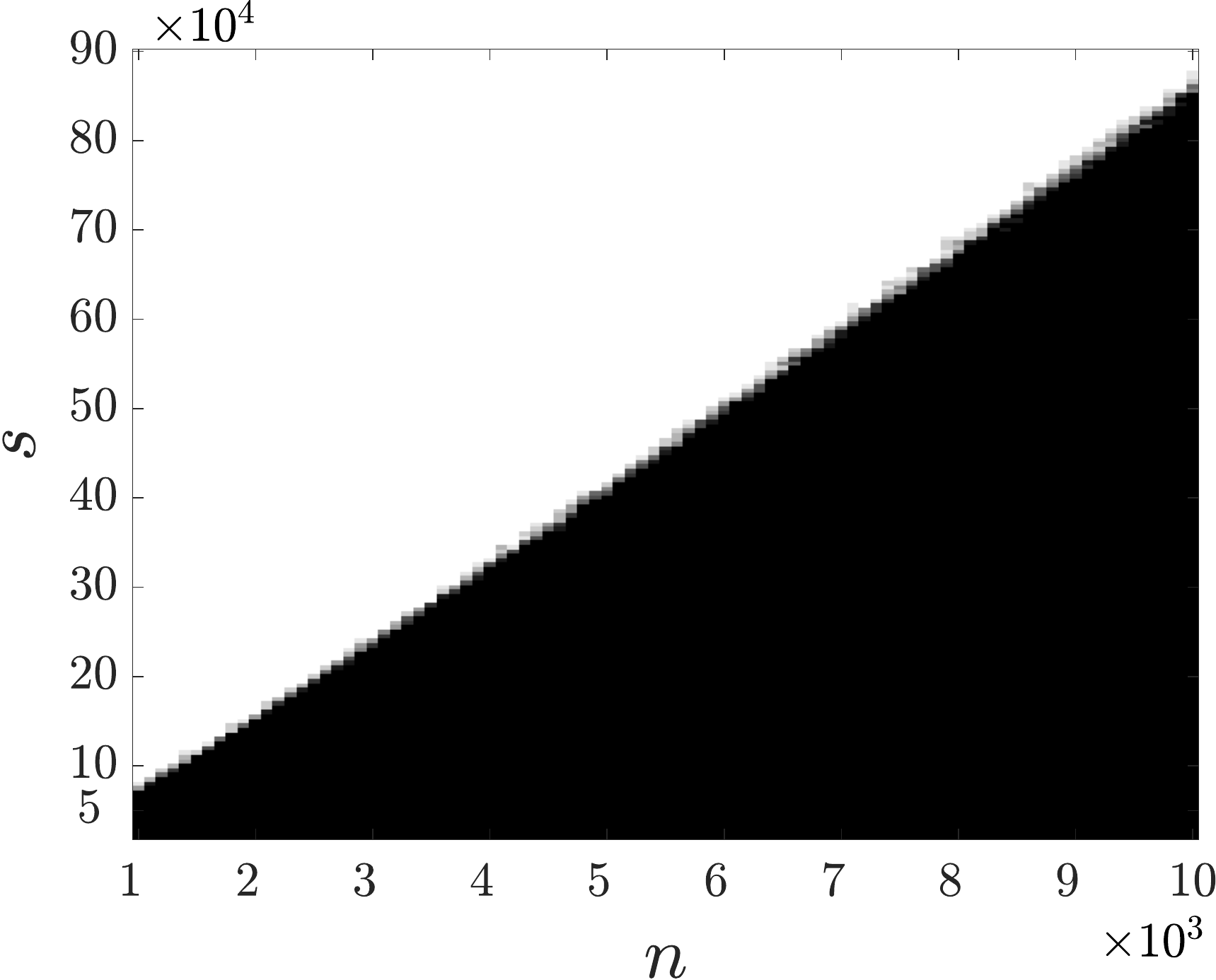}\qquad\quad
    \includegraphics[width=0.28\textwidth]{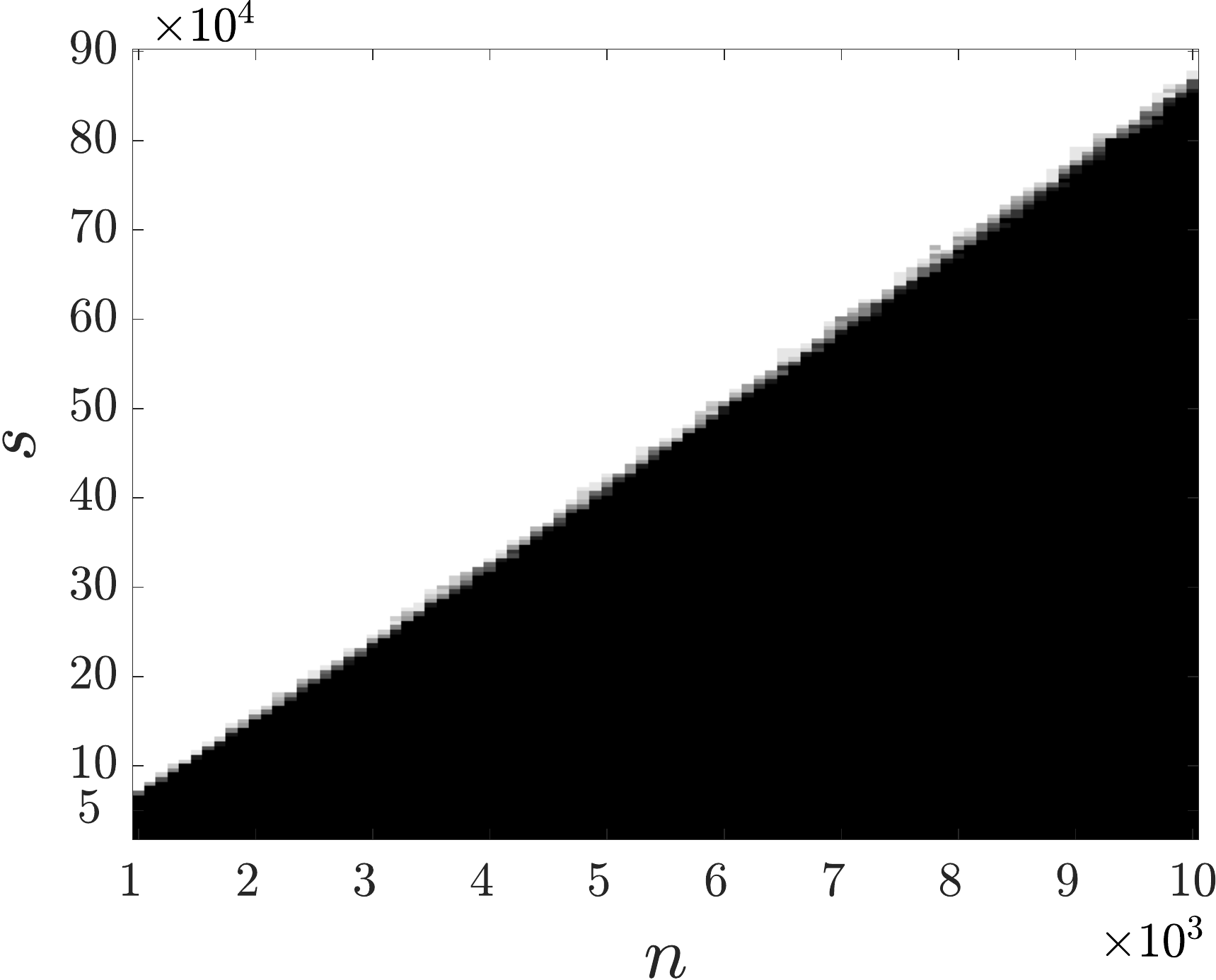}

      \caption{
      Empirical phase transitions of ICURC in overall observation size $s$ and problem size $n$. The column (resp. row) number of the concentrated column (resp. row) submatrix equals to $cr\log^2(n)$. \textbf{Row 1}: $r = 5$. \textbf{Row 2}: $r = 10$. \textbf{Row 3}: $r = 15$. \textbf{Left}: $c = 0.25$. \textbf{Middle}: $c = 0.5$. \textbf{Right}: $c = 1$. The required samples for guaranteed matrix completion are independent of the size of the concentrated submatrices.} 
    \label{fig:both_phtr-2}
\end{figure*}

First, we explore the recovery ability of ICURC for CCS with different combinations of the sampling number of rows and columns with $|\cI|=|\cJ|=\delta n$ and the uniform sampling rate $p$ on the selected rows and columns. This experiment runs on the matrix of  size $10^{3}\times 10^{3}$ and under different rank settings with rank $r \in \{5,10,15\}$. For each  rank,  we generate $20$ test examples for each given  pair of $(\delta,p)$ and an example is considered to be successfully  solved if the relative error
\begin{equation}\label{eqn:varepsilonk}
\varepsilon_k:= \frac{\|\BX-\BC_k\BU_k^\dagger \BR_k\|_\fro}{\|\BX\|_\fro}\leq 10^{-2}.  
\end{equation}

These simulation results are summarized in Figure~\ref{fig:both_phtr}, where the first row presents the 3D view of the phase transition results and the second row shows the corresponding 2D view results by adding the uniform sampling results in the last two columns.  In the 3D result, the $z$-axis stands for the overall sampling rate over the whole matrix. In Figure~\ref{fig:both_phtr}, a white pixel represents the successful completion of these $20$ tests and a black pixel means all the $20$ tests are failed. From these figures, one can see that as the rank $r$ increases, the required overall sampling rate becomes larger to guarantee successful completion since a larger rank $r$ leads to a more difficult problem. Moreover,  regardless of the combinations of the sizes of the concentrated row and column submatrices and the sampling rates on the selected submatrices, we guarantee matrix completion  
as long as the combinations result in a sufficiently large total sampling rate (see the second row of Figure~\ref{fig:both_phtr}). This observation shows that CCS provides the flexibility to sample low-rank matrix and still ensures the successful completion of the underlying low-rank matrix from its samples. From the second row of Figure~\ref{fig:both_phtr}, one can see that the required sampling rates for ICURC on the CCS model are comparable with that for the ScaledPGD \cite{tong2021accelerating} and the SVP \cite{jain2010guaranteed} algorithms on the uniform sampling.

 We also investigate the recovery ability of our ICURC for CCS in the framework of phase transition by studying the relation between the required measurements of the underlying low-rank matrices and their size $n$. The results are reported in Figure~\ref{fig:both_phtr-2}.
 Therein, we first uniformly sample a row submatrix $\BR\in\mathbb{R}^{cr\log^2(n)\times n}$ and a column submatrix $\BC\in\mathbb{R}^{n\times cr\log^2(n)}$, where $r \in \{5,10, 15\}$ and $c \in \{0.25,0.5,1\}$. Then, we uniformly sample $s/2$ entries on each submatrix, i.e., $s$ samples in total and the intersection submatrix $\BU$ has a denser sampling rate than $\BR$ and $\BC$. Similar to Figure~\ref{fig:both_phtr}, we generate $20$ problems for each given pair of $(n,s)$ under each setting of $(r,c)$,  and a problem is considered to be successfully solved if $\varepsilon_k\leq 10^{-2}$ (recall $\varepsilon_k$ is defined in \eqref{eqn:varepsilonk}). From Figure~\ref{fig:both_phtr-2}, one can see that the overall required samples to guarantee the recovery of the missing data is independent of the size of the concentrated submatrices. These observations further illustrate the flexibility of our CCS model.

\subsection{Image Recovery} \label{subsec:three_image}
\begin{table*}[h]
\centering
\caption{Image inpainting results on Building and Window datasets.  Under various setups of CCS, ICURC can achieve higher SNR with shorter runtime compared with other methods. 
}
\begin{tabular}{cc|ccc|ccc}
\toprule
\multicolumn{2}{c|}{\textsc{Dataset}}                                     & \multicolumn{3}{c|}{{Building}}                                                                                                                                 & \multicolumn{3}{c}{Window}                                                                      \\ \midrule
\multicolumn{2}{c|}{\textsc{Overall Observation Rate} ($\alpha$)}                             & \multicolumn{1}{c}{10 \%}            & \multicolumn{1}{c}{12 \%}            & 14 \%                     & \multicolumn{1}{c}{10 \%}            & \multicolumn{1}{c}{12 \%}            & 14 \%            \\ \midrule
\multicolumn{1}{c|}{\multirow{5}{*}{$\mathrm{SNR}$}}      & ICURC-8  & \multicolumn{1}{c}{\textbf{23.762}} & \multicolumn{1}{c}{\textbf{24.750}} & \textbf{25.195} &  \multicolumn{1}{c}{31.792}          & \multicolumn{1}{c}{32.343}          & 32.5533          \\ %\cline{2-8} 
\multicolumn{1}{c|}{}                          & ICURC-9  & \multicolumn{1}{c}{23.688}          & \multicolumn{1}{c}{24.557}          & 25.108          &  \multicolumn{1}{c}{\textbf{31.831}} & \multicolumn{1}{c}{\textbf{32.513}} & \textbf{32.984} \\ %\cline{2-8} 
\multicolumn{1}{c|}{}                          & ICURC-10 & \multicolumn{1}{c}{23.629}          & \multicolumn{1}{c}{24.229}          & 24.823          & \multicolumn{1}{c}{31.546}          & \multicolumn{1}{c}{32.364}          & 32.911          \\ %\cline{2-8} 
\multicolumn{1}{c|}{}                          & ScaledPGD & \multicolumn{1}{c}{20.593}          & \multicolumn{1}{c}{21.722}          & 21.734                 & \multicolumn{1}{c}{31.338}          & \multicolumn{1}{c}{31.918}          & 32.693          \\ %\cline{2-8} 
\multicolumn{1}{c|}{}                          & SVP      & \multicolumn{1}{c}{18.065}          & \multicolumn{1}{c}{18.940}          & 19.607                 & \multicolumn{1}{c}{27.451}          & \multicolumn{1}{c}{29.4900}          & 30.8541          \\ \midrule
\multicolumn{1}{c|}{\multirow{5}{*}{\textsc{Runtime} (sec)}} & ICURC-8  & \multicolumn{1}{c}{\textbf{0.400}}  & \multicolumn{1}{c}{\textbf{0.366}}  & \textbf{0.306}   & \multicolumn{1}{c}{\textbf{0.339}}  & \multicolumn{1}{c}{\textbf{0.244}}  & \textbf{0.313}  \\ %\cline{2-8} 
\multicolumn{1}{c|}{}                          & ICURC-9  & \multicolumn{1}{c}{0.482}           & \multicolumn{1}{c}{0.416}           & 0.395                     & \multicolumn{1}{c}{0.343}           & \multicolumn{1}{c}{0.313}           & 0.364           \\ %\cline{2-8} 
\multicolumn{1}{c|}{}                          & ICURC-10 & \multicolumn{1}{c}{0.616}           & \multicolumn{1}{c}{0.518}           & 0.425                    & \multicolumn{1}{c}{0.627}           & \multicolumn{1}{c}{0.462}           & 0.396           \\ %\cline{2-8} 
\multicolumn{1}{c|}{}                          & ScaledPGD & \multicolumn{1}{c}{3.518}           & \multicolumn{1}{c}{3.405}           & 3.230                  & \multicolumn{1}{c}{2.722}           & \multicolumn{1}{c}{2.314}           & 2.141           \\ %\cline{2-8} 
\multicolumn{1}{c|}{}                          & SVP      & \multicolumn{1}{c}{9.925}           & \multicolumn{1}{c}{14.675}          & 14.661                     & \multicolumn{1}{c}{10.626}          & \multicolumn{1}{c}{10.102}          & 9.873           \\ 
\bottomrule
\end{tabular}
\label{tab:three_image_results}
 \end{table*}

 \begin{figure*}[th!]
 \centering
 \begin{minipage}{.24\linewidth} \centering \small  Original \end{minipage}
 \begin{minipage}{.24\linewidth} \centering \small  ScaledPGD \end{minipage}
\begin{minipage}{.24\linewidth} \centering \small  SVP \end{minipage}
\begin{minipage}{.24\linewidth} \centering \small  ICURC-10 \end{minipage}\\
\vspace{0.7mm}
\includegraphics[width=0.24\textwidth]{main_image/b3_original.jpg}
\includegraphics[width=0.24\textwidth]{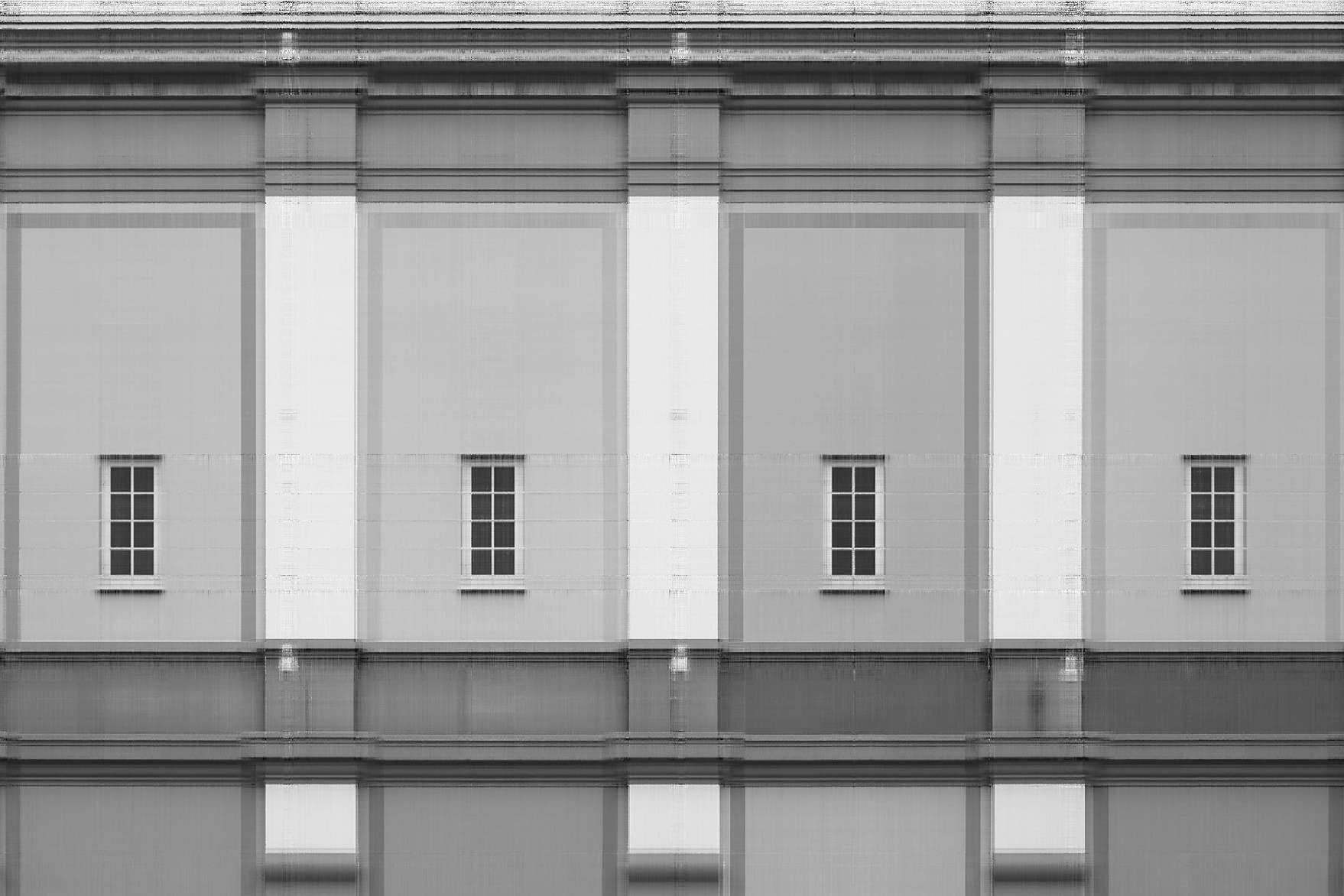}
\includegraphics[width=0.24\textwidth]{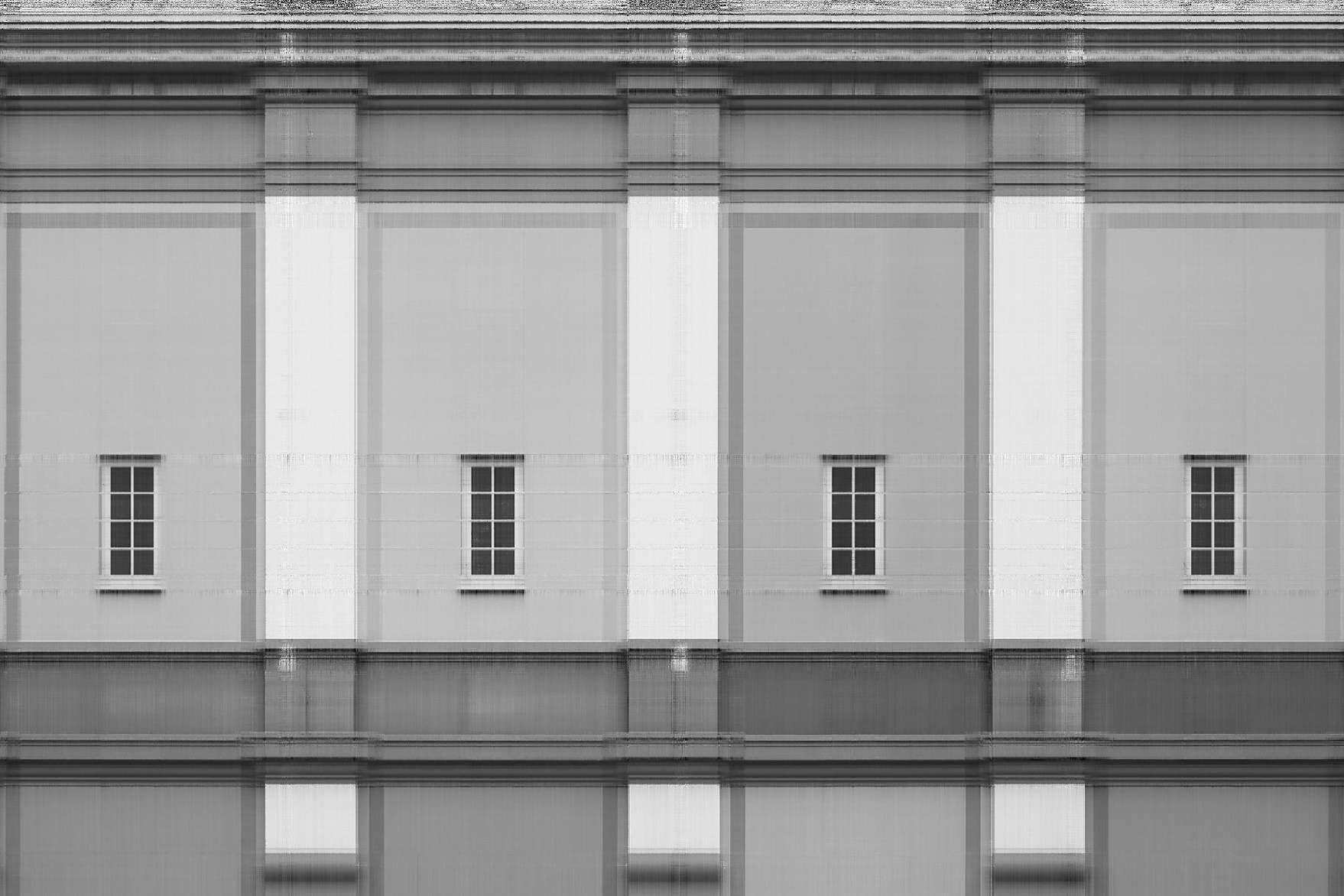}
\includegraphics[width=0.24\textwidth]{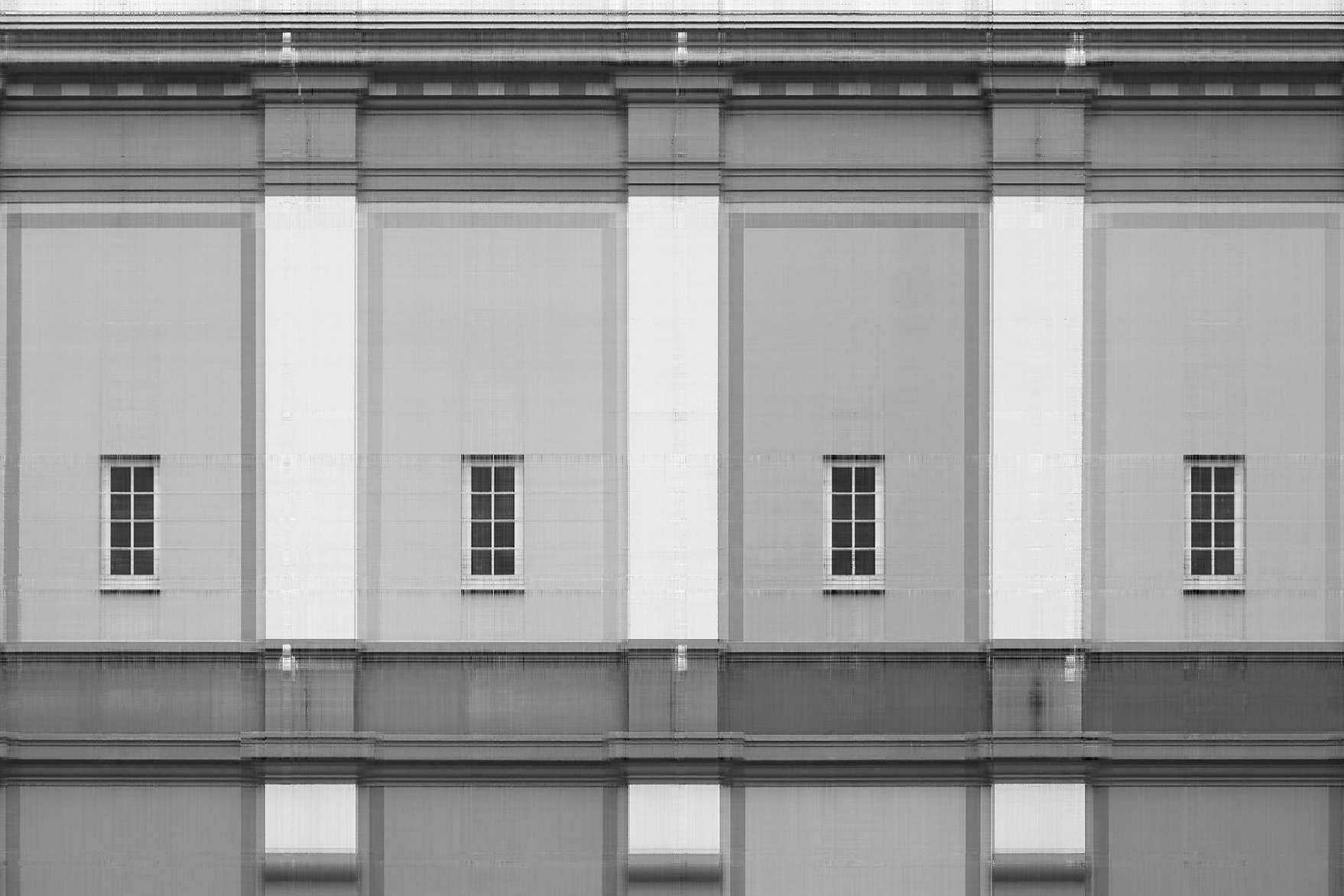}\\
\vspace{.7mm}
  \includegraphics[width=0.24\textwidth]{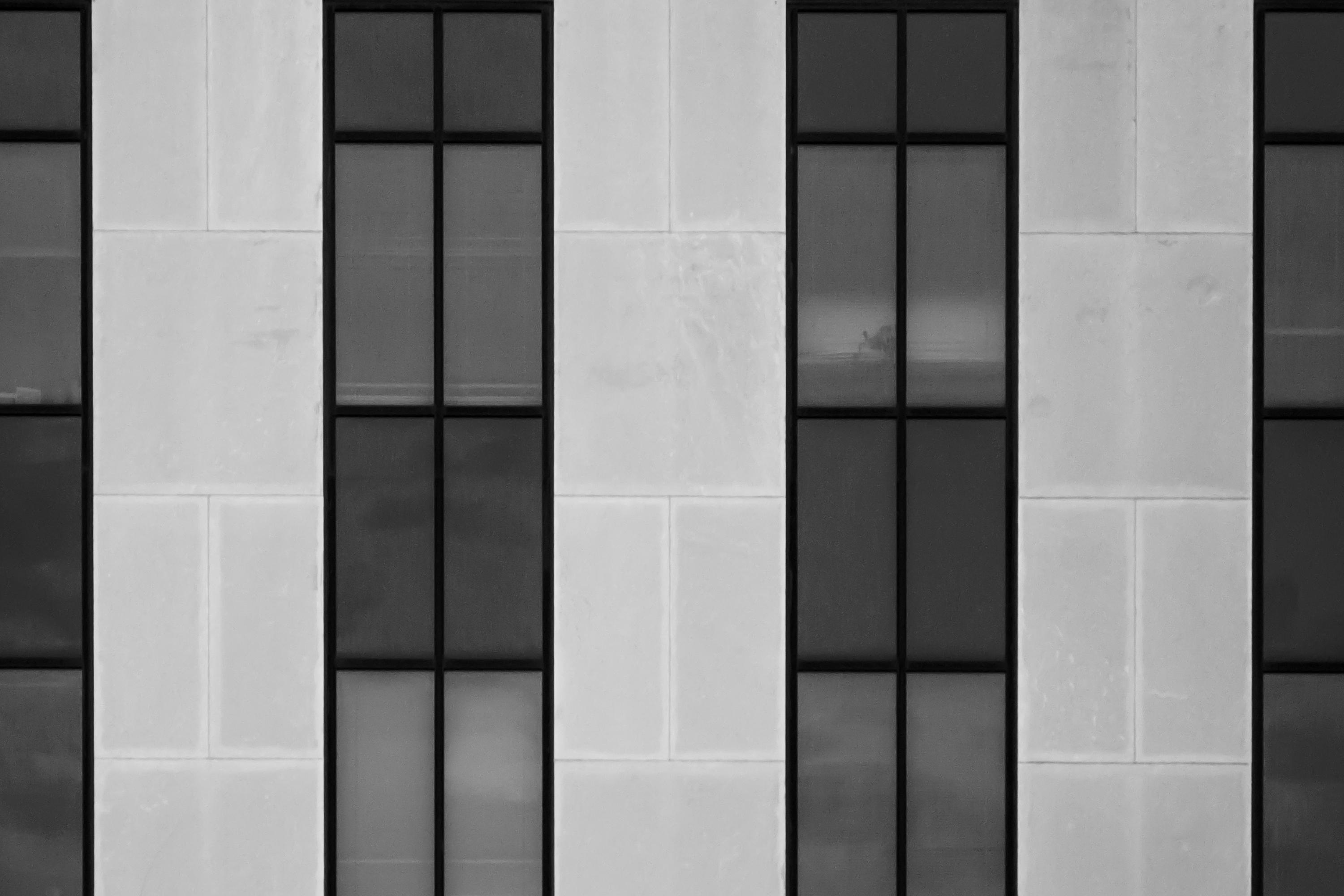}
 \includegraphics[width=0.24\textwidth]{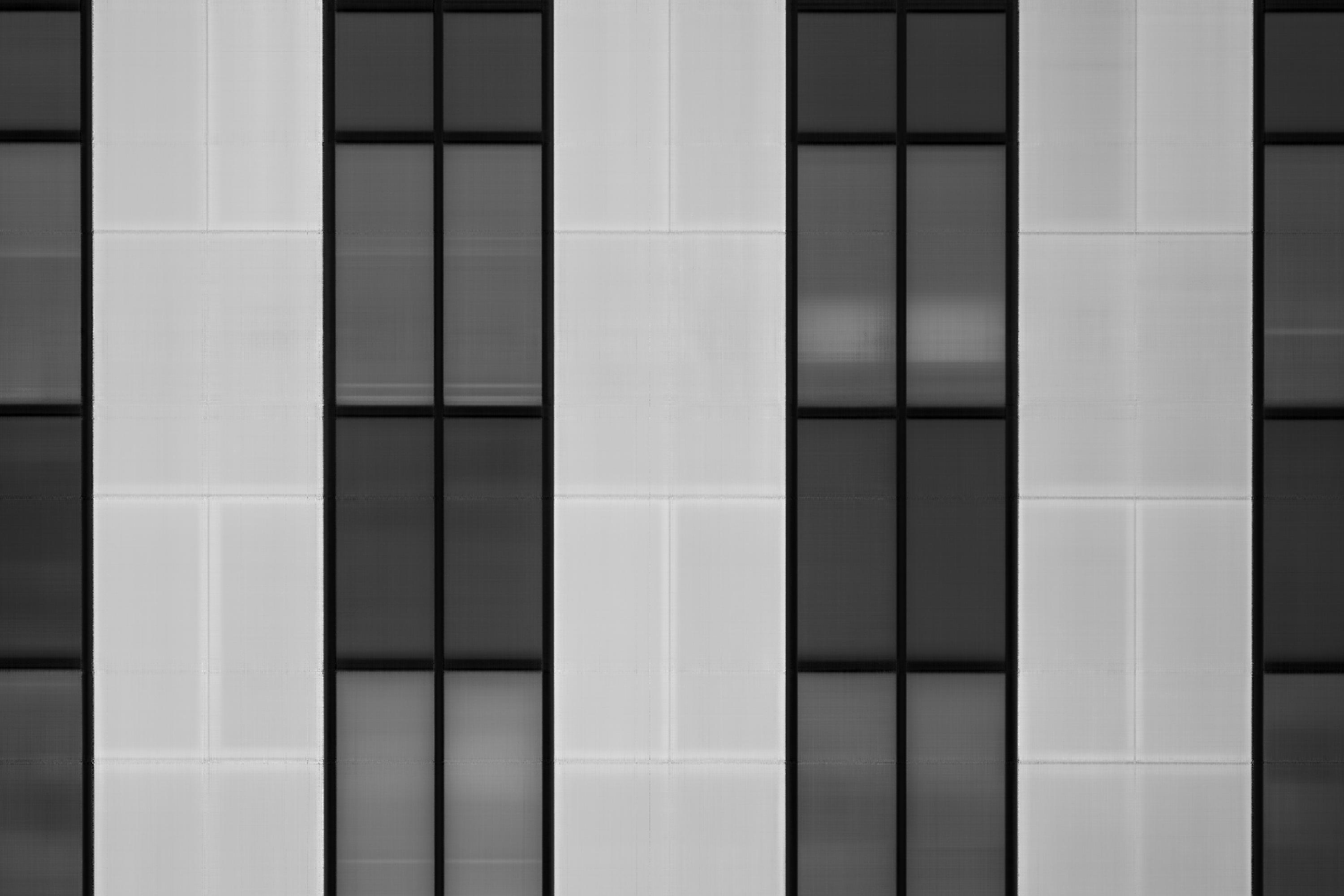}
 \includegraphics[width=0.24\textwidth]{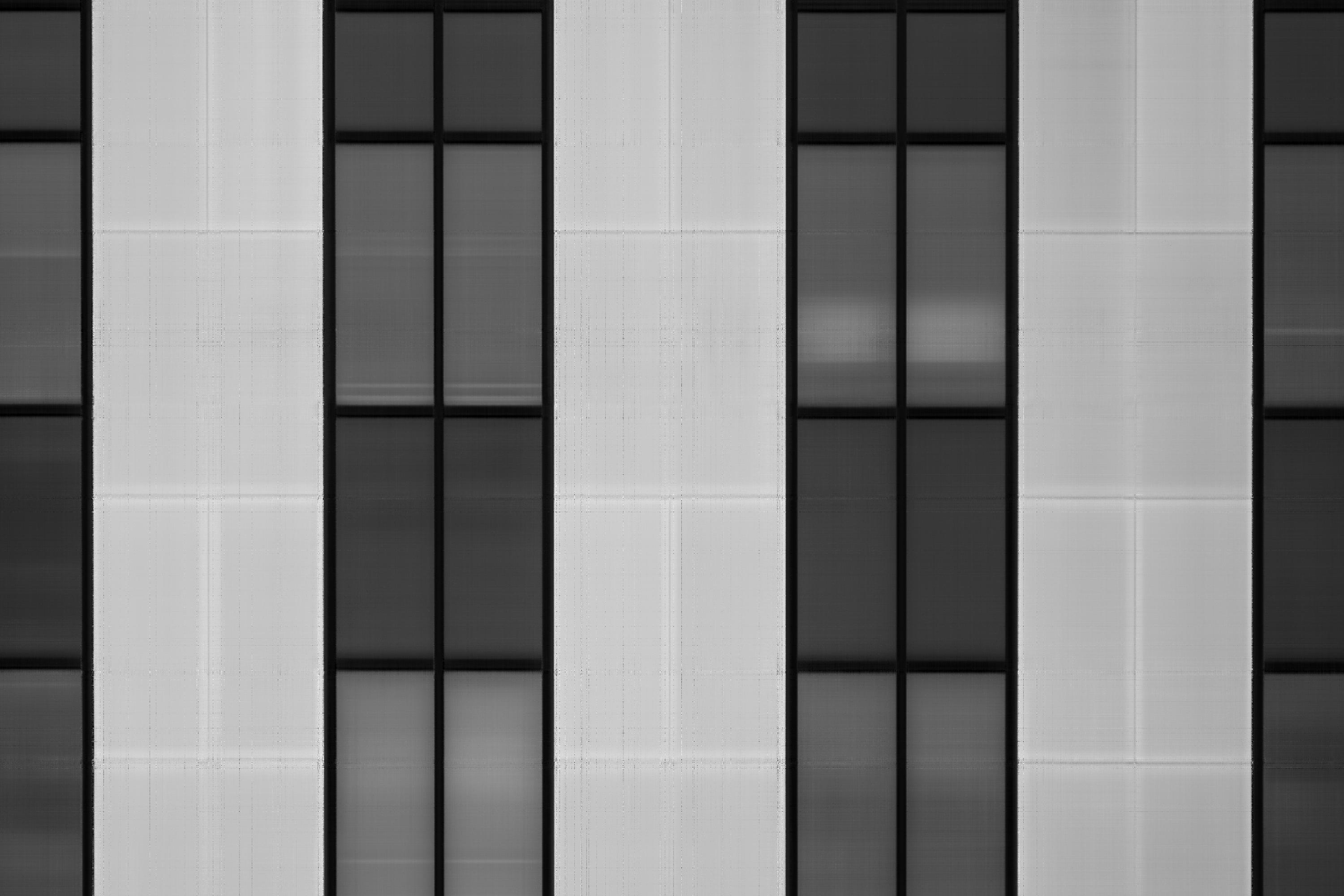}
  \includegraphics[width=0.24\textwidth]{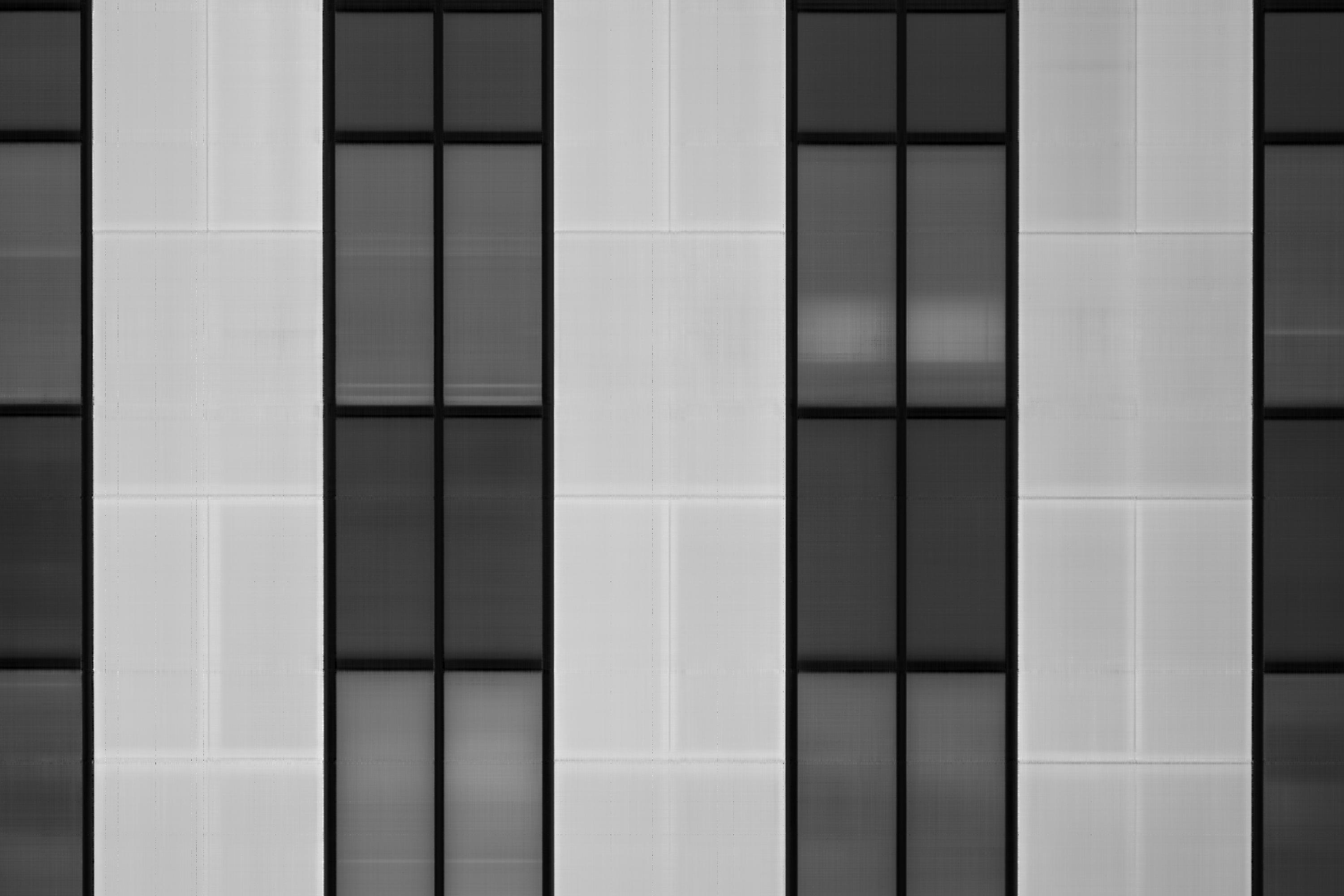}
  \caption{ 
  Visual results for image inpainting  by setting rank $r = 20$ and  the percentage of selected rows and columns $\delta = 10 \%$. ScaledPGD and SVP are based on the uniform sampling model with the same observed number of entries as the one based on CCS. All algorithms achieve visually reliable results.
  }
    \label{fig:example1}
\end{figure*}

In this section, we compare the matrix completion  performances solved by ICURC using CCS  and by ScaledPGD \cite{tong2021accelerating} and SVP \cite{jain2010guaranteed} using uniform sampling for image recovery. The simulations are tested on  two  grey-scaled images, namely ``Building''\footnote{\url{https://pxhere.com/en/photo/57707}.} %, ``Apartment''\footnote{\url{https://pxhere.com/en/photo/1393160}.}, 
and ``Window''\footnote{\url{https://pxhere.com/en/photo/1421981}.} of size  $2000 \times 3000$ by 
recording reconstruction quality and runtime. The reconstruction  quality   is measured by the signal-to-noise  ratio (SNR), which is defined as
\[\mathrm{SNR}_{\mathrm{dB}}(\widetilde{\BX})=20\log_{10}\left(\frac{\|\BX\|_\fro}{\|\widetilde{\BX}-\BX\|_\fro}\right),
\]
where $\BX$ is the original image and $\widetilde{\BX}$ represents the reconstructed image.

In this simulation, we aim  to find a rank-$20$ approximation $\widetilde{\BX}$ for the given image $\BX$. First we generate the observations according to  the CCS model. We randomly select the concentrated row and column submatrices $\BR$ and $\BC$ with  row indices $\cI$ of size $\delta m$  and column indices $\cJ$ of  size $\delta n$ columns, i.e., $\BR = [\BX]_{\cI,:}$ and $\BC = [\BX]_{:,\cJ}$. Then, we randomly select $\frac{\alpha m n}{2}$ entries on each submatrix and denote the corresponding indices of the observed entries by $\Omega_{\BR}$ and $\Omega_{\BC}$,  which result in two partially observed submatrices $\BR_{\mathrm{obs}} = [\cP_{\Omega_{\BR}}(\BX)]_{\cI,:}$ and $\BC_{\mathrm{obs}} = [\cP_{\Omega_{\BC}}(\BX)]_{:,\cJ}$. 
As a result, we obtain a partially observed image whose observed entries are concentrated on $\BR$ and $\BC$. {Then we fill in the missing pixels by applying our ICUR algorithm.} 
{In comparison, we also generate $\alpha m n$ observations based on the uniform sampling model over the original matrix X. } After that, we fill in the missing data via ScaledPGD or SVP.  The above processes are repeated for $10$ times. 
 
The averaged test results on different $\alpha$ (i.e., overall observation rates) are summarized in Table~\ref{tab:three_image_results}. Meanwhile, we provide some visual results in Figure~\ref{fig:example1}. It shows that all the algorithms achieve visually reliable results. Additionally, in comparison with the visual results in Figure~\ref{fig:GMC_on_CCS},  our ICURC algorithm has a much better performance than SVP and ScaledPGD algorithms in solving the image inpainting problems when the observed pixels are selected based on our CCS model.  From Table~\ref{tab:three_image_results}, one can observe that regardless of different combinations of the sizes of the concentrated row and column submatrices and the sampling rates,  the qualities of the results from ICURC are similar as long as the overall sampling rates are the same. This observation further illustrates the flexibility of the CCS model. Additionally, one can also find that   ICURC on CCS achieves comparable (even better) quality with the algorithms on the uniform sampling model. From the runtime perspective, our ICURC on the CCS model is substantially faster than ScaledPGD and SVP on uniform sampling. 

\subsection{Recommendation System} \label{subsec:RS}
Recommendation systems aim to predict the users' preferences from partial information of personalized item recommendations. Each dataset in a recommendation system can be represented as a matrix by arranging each item's ratings as a row and each user's ratings as a column. 
If we view the unobserved ratings as missing entries of data matrices, predicting the missing ratings can be considered a matrix completion problem, as the underlying matrix is expected to be  low-rank since only a few factors contribute to an individual's preferences \cite{plan2011compressed}. 
In this section, we evaluate the performance of our CCS model solving by ICURC algorithm on three datasets namely the Movie-1M, the Movie-10M datasets from the Movielens research project\footnote{Movie-100K and the Movie-10M datasets can be downloaded from \url{ https://grouplens.org/datasets/movielens}\label{ml_data}.} \cite{10.1145/2827872} and the FilmTrust dataset\footnote{FilmTrust dataset can be found at \url{https://guoguibing.github.io/librec/datasets.html}.} \cite{guo2013novel}. For a given dataset, we first generate an item-user matrix of size $m\times n$. Due to the large size and low observation rate of the Movie-10M dataset, we follow the instructions in \cite{kalofolias2014matrix} to extract a $2000\times 3000$ submatrix  based on the observation rates on rows and columns. 
The characteristics of all the data used in our simulations are summarized in Table~\ref{tab:dataset_info}.

To evaluate the performance, we employ the Cross-Validation method. For each run, the observed data is randomly split into training and testing sets denoted by $\Omega_{\mathrm{train}}$ and $\Omega_{\mathrm{test}}$ \cite{10.1145/1454008.1454031}. More specifically,  we randomly select $\delta m$  rows and $\delta n$ columns and then randomly choose $\alpha  m n$ entries from the observed entries on the selected rows and columns to form   $\Omega_\mathrm{train}^{(1)}$  for  CCS model. For comparison, we also randomly choose $\alpha mn$ entries from the observed data over the whole matrix to form a new training dataset $\Omega_\mathrm{train}^{(2)}$ based on the uniform sampling model. The information on the training and testing sets for each data matrix is summarized in Table~\ref{tab:dataset_info}. After the training and testing datasets are generated, we run the ICURC algorithm on $\Omega_\mathrm{train}^{(1)}$, and ScaledPGD and SVP algorithms on $\Omega_\mathrm{train}^{(2)}$. 

Following  \cite{10.1145/1454008.1454031}, we adopt two different methods to measure the recommendation quality on the testing datasets. The first one is the hit-rate (HR) which is defined as  the ratio of the number of hits to the size of the testing dataset: 
\begin{equation} \label{eq:HR}
    \mathrm{HR} = \frac{\mathrm{\#hits}}{|\Omega_\mathrm{test}|},
\end{equation}
where a predicted rating $P_i$ is considered as a hit if its rounded value is equal to the actual rating $A_i$ in the test set. To penalize each missed prediction and to emphasize the errors, we also computed the Normalized Mean Absolute Error (NMAE) \cite{ma2011fixed, goldberg2001eigentaste} defined as follows: 
\begin{equation} \label{eq:NMAE}
    \mathrm{NMAE} = \frac{1}{|\Omega_{\mathrm{test}}| (S_{\max} - S_{\min})}\sum_{i \in \Omega_{\mathrm{test}}} \abs{P_i - A_i},
\end{equation}
where $S_{\max}$ and $S_{\min}$ denote the maximum and minimum rating, respectively.    

Figure~\ref{fig:barplot_summary_RC} summaries the averaged numerical results over $10$ independent trials. One can see that the results for the ICURC algorithm based on the CCS model have better performance compared with the ones for ScaledPGD and SVP algorithms on the uniform sampling model.  
Specifically, ICURC can reach higher HR and lower NMAE in a much shorter runtime compared with other methods. This further illustrates that ICURC is computationally efficient. Fixing the overall sampling rate for the CCS model, one can observe that the performances for different combinations of the concentrated row and column submatrices are comparable. This observation further illustrates the flexibility of our CCS model.
 
\begin{table}[h]
\caption{Datasets information for collaborative filtering. Here, $\alpha$ is the overall observation rate.}
\label{tab:dataset_info}
\centering
\begin{tabular}{c|c c c c}
\toprule
\textsc{Dataset}      & \textsc{\texttt{\#}users} & \textsc{\texttt{\#}items} & $\alpha$\,(\%) & \textsc{Rating\,Range} \\ \midrule
ML-100K     & 943    & 1682    & 4.190 & 1--5   \\ 
ML-10M     & 2000    & 3000    & 32.92 & 1--10  \\ 
FilmTrust & 1508    & 2071    & 2.660    & 1--8   \\ 
\bottomrule
\end{tabular}
\end{table}

\begin{figure*}[th]
    \raggedright

    \includegraphics[width=0.29\textwidth]{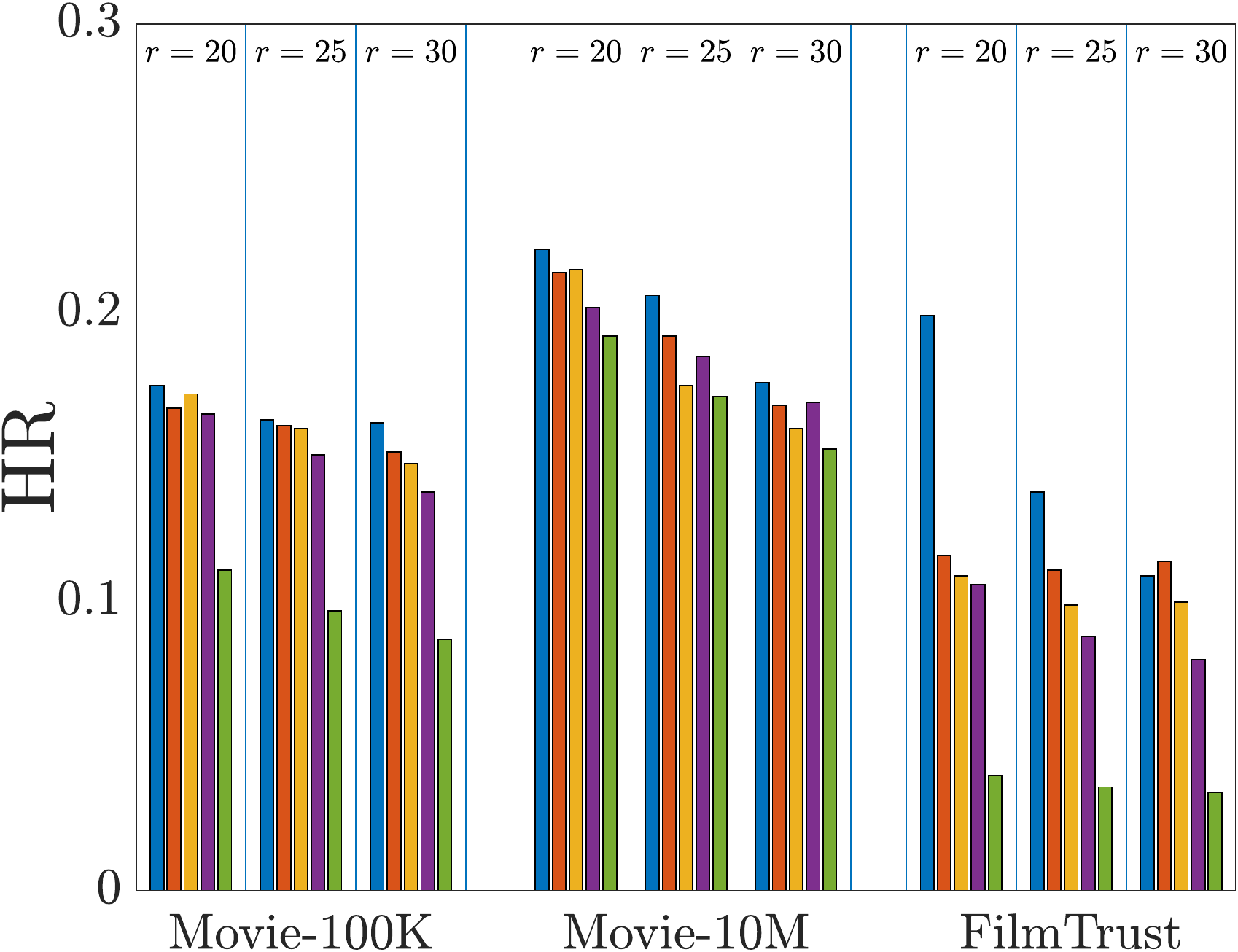}\hspace{1.5mm}
    \includegraphics[width=0.29\textwidth]{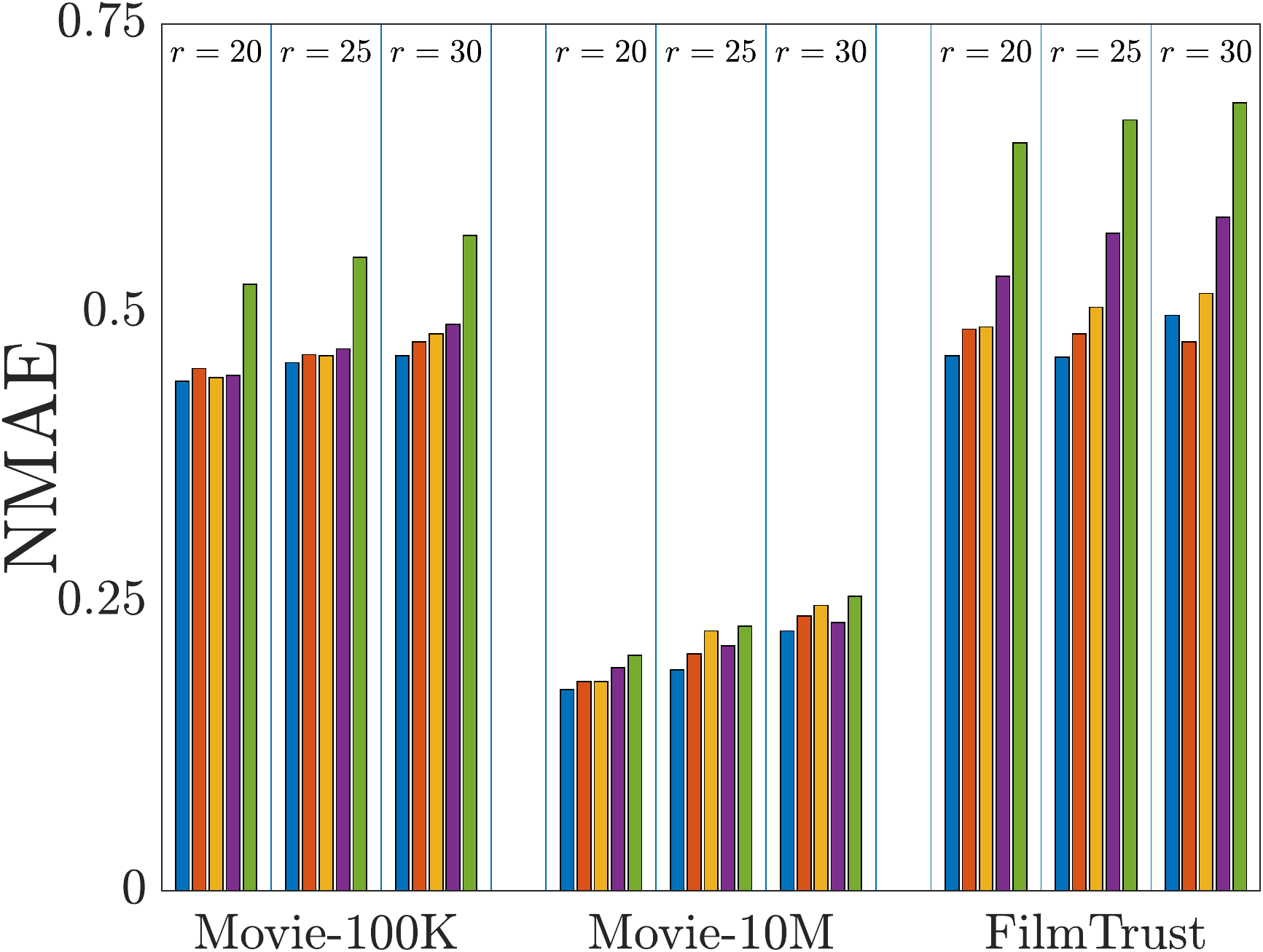}\hspace{1.5mm}
     \includegraphics[width=0.29\textwidth]{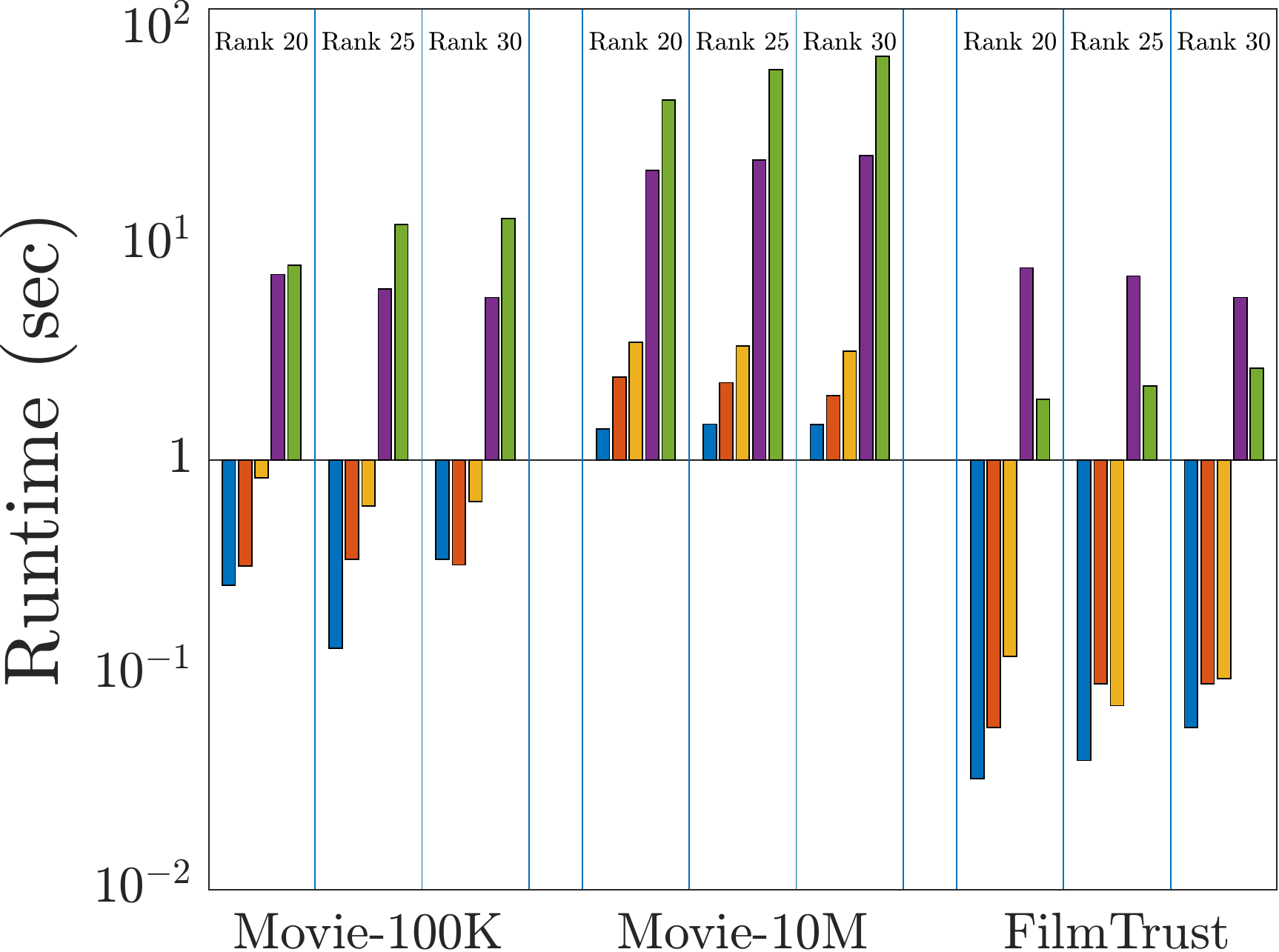}
    \raisebox{0.13\height}{ \includegraphics[width=0.09\textwidth]{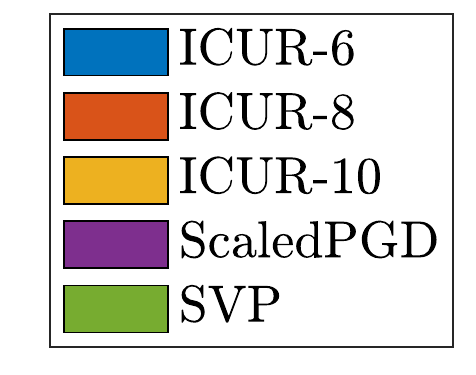}}
    \\
    
    \includegraphics[width=0.29\textwidth]{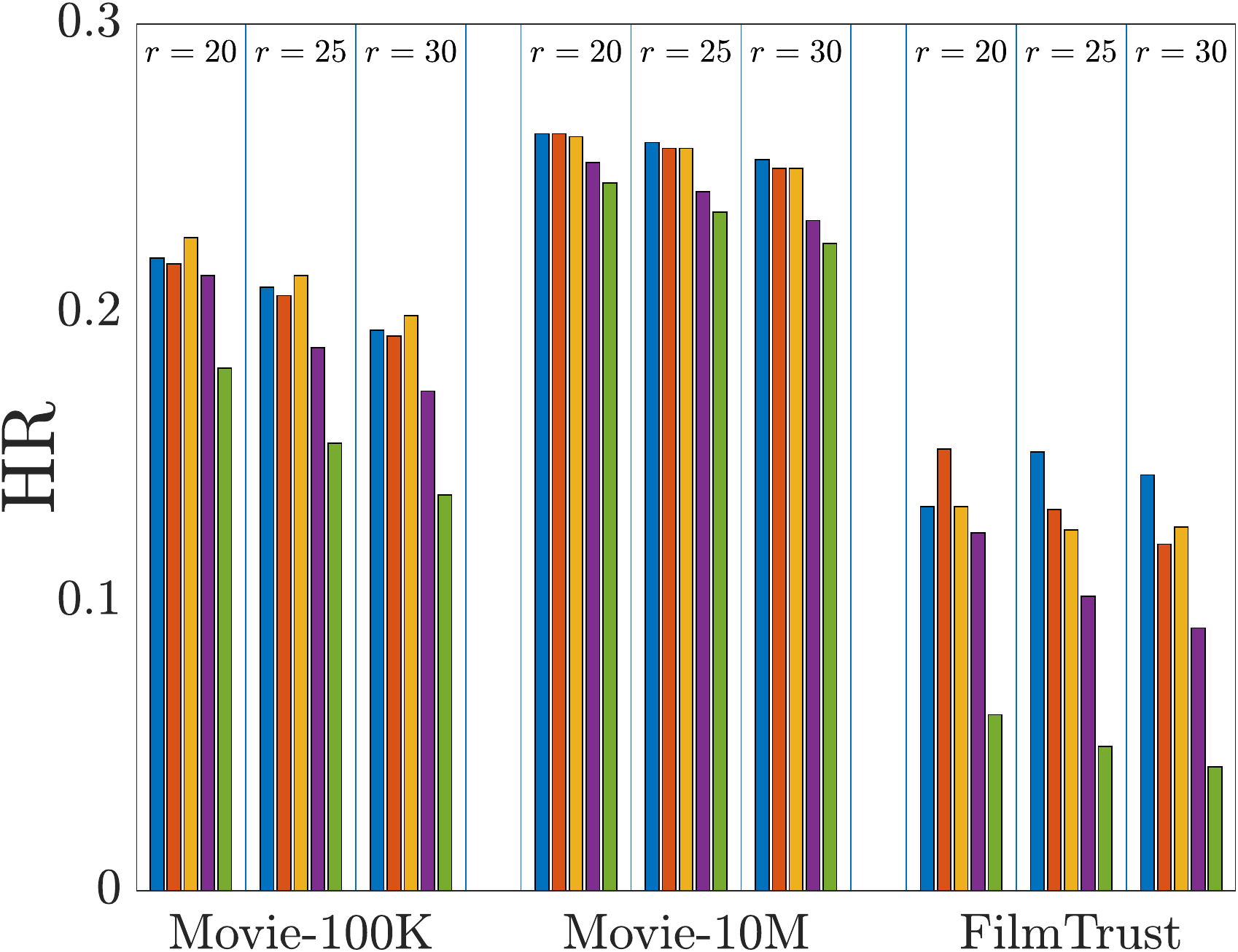}\hspace{1.5mm}
    \includegraphics[width=0.29\textwidth]{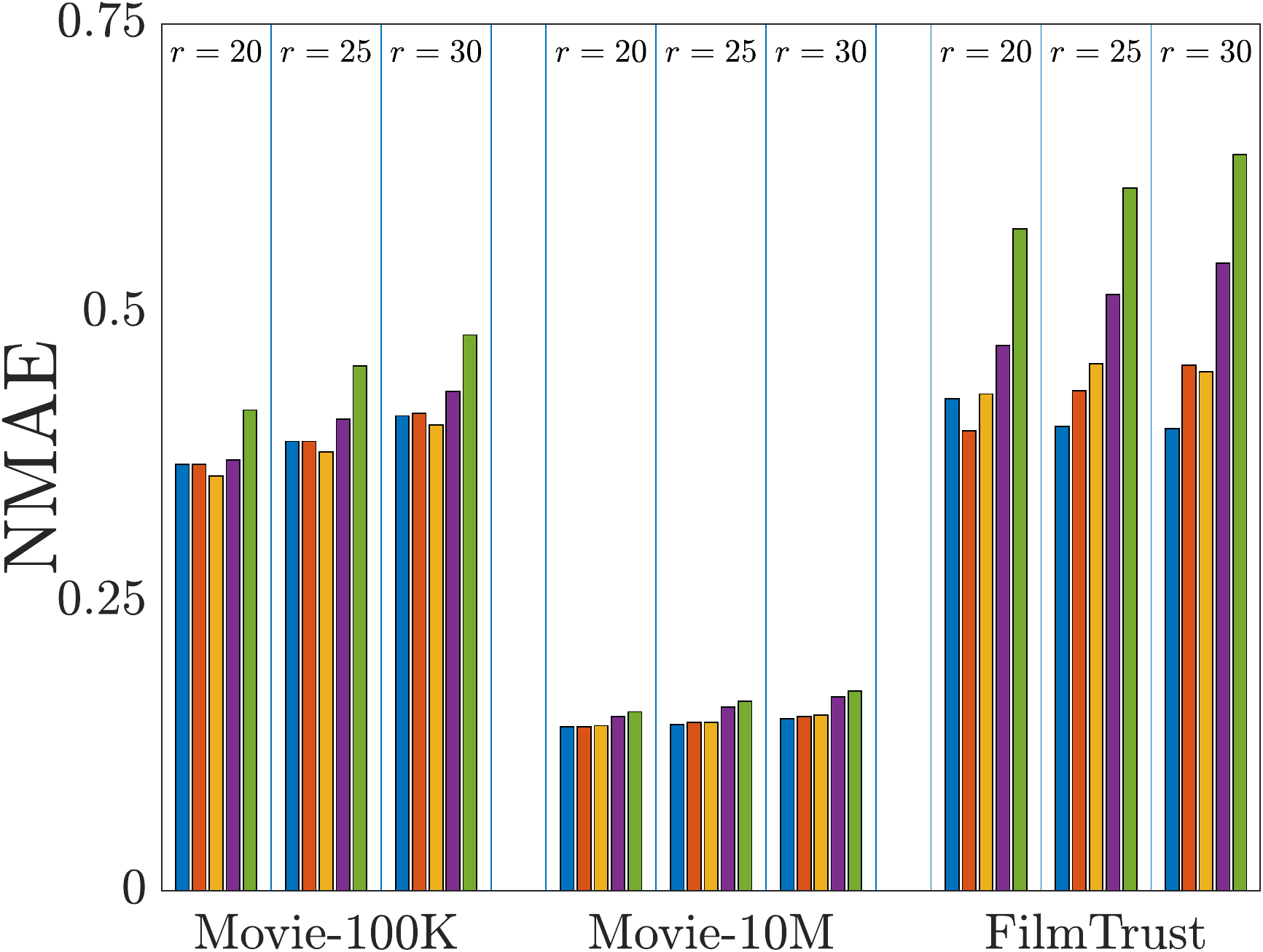}\hspace{1.5mm}
     \includegraphics[width=0.29\textwidth]{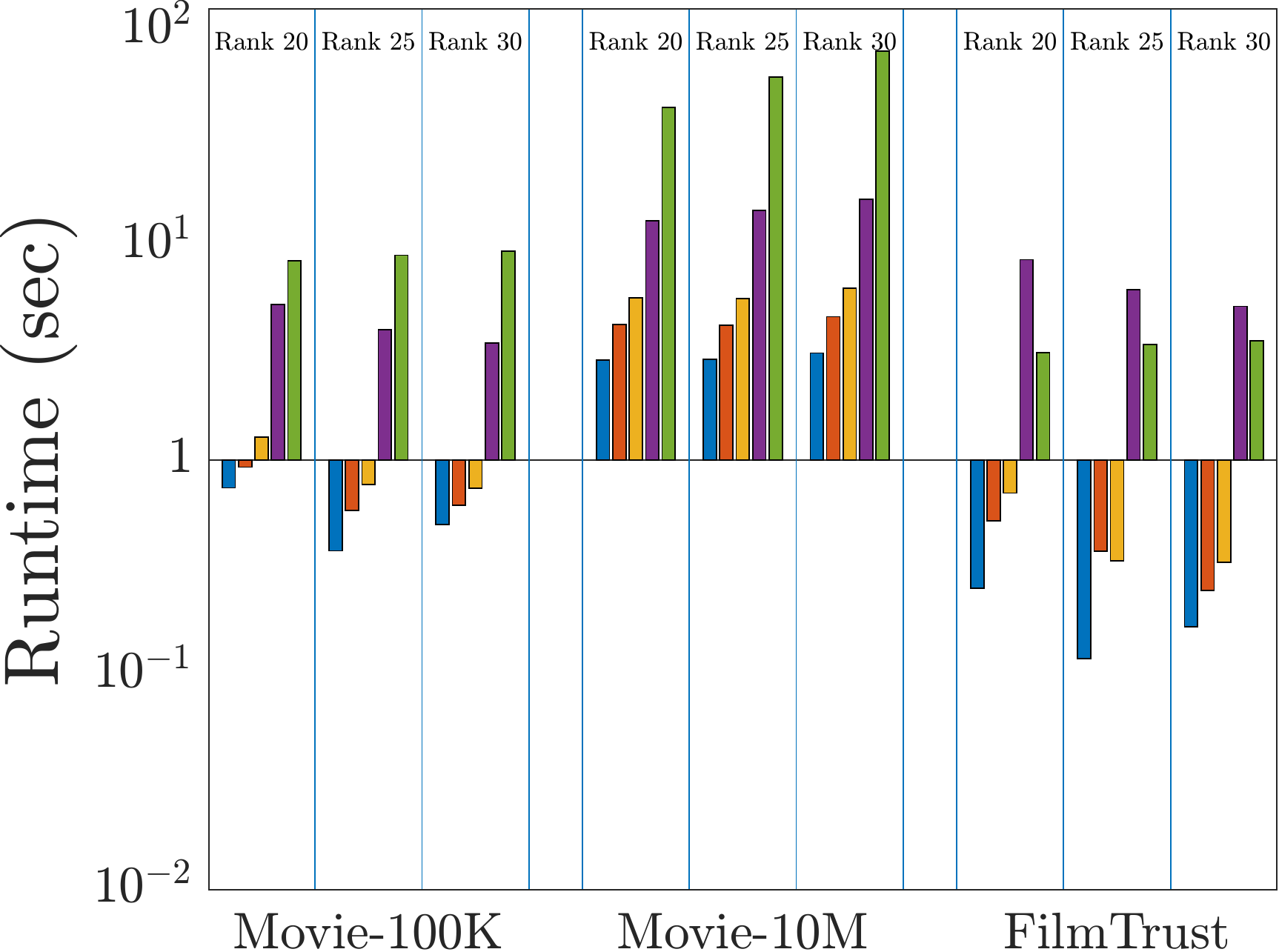}
     \raisebox{0.13\height}{\includegraphics[width=0.09\textwidth]{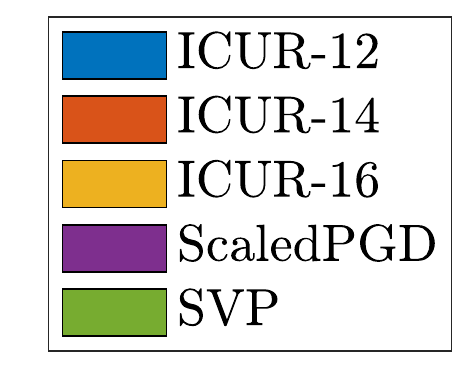}  } 
    \\

    \includegraphics[width=0.29\textwidth]{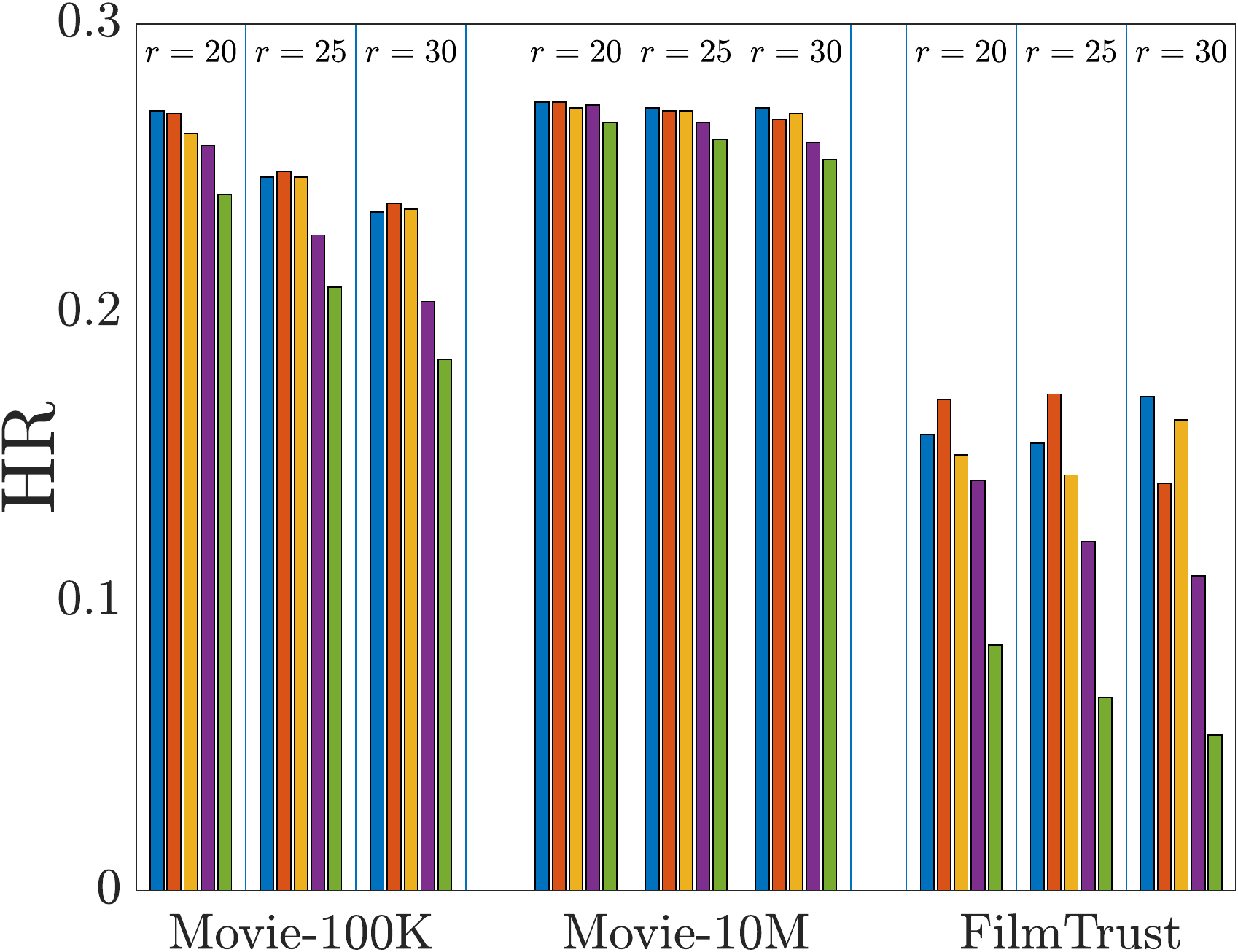}\hspace{1.5mm}
    \includegraphics[width=0.29\textwidth]{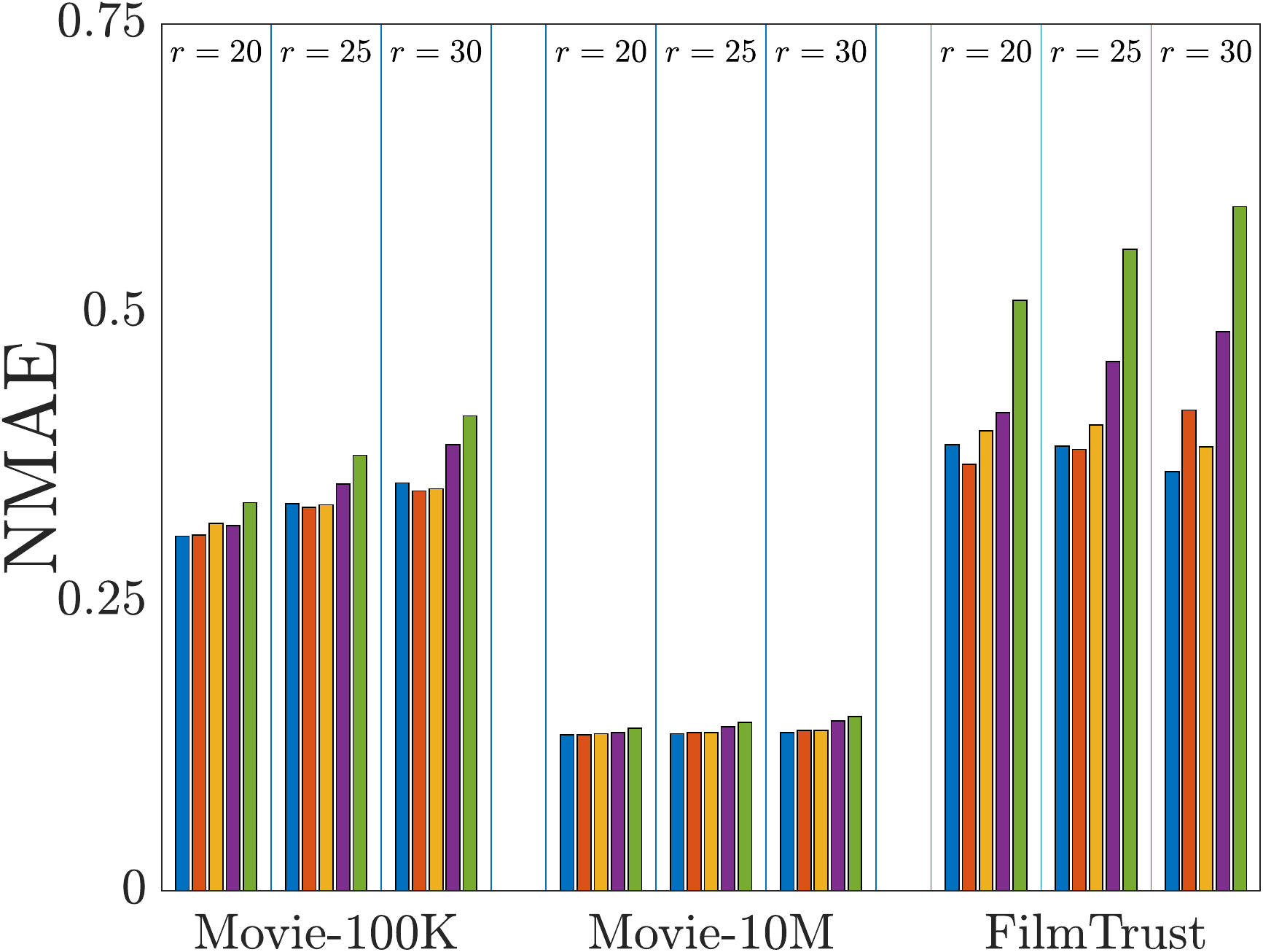}\hspace{1.5mm}
    \includegraphics[width=0.29\textwidth]{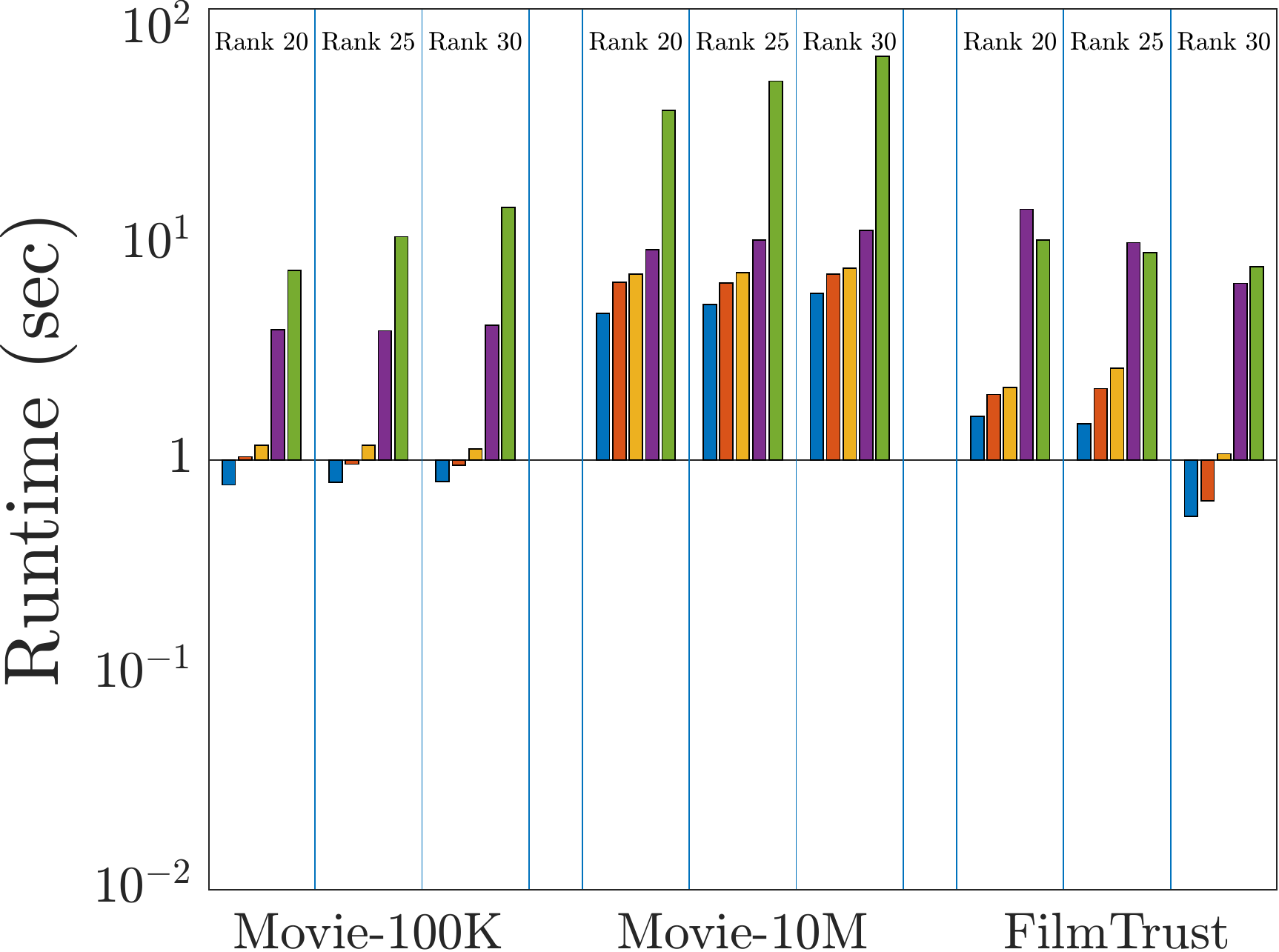}
     \raisebox{0.13\height}{\includegraphics[width=0.09\textwidth]{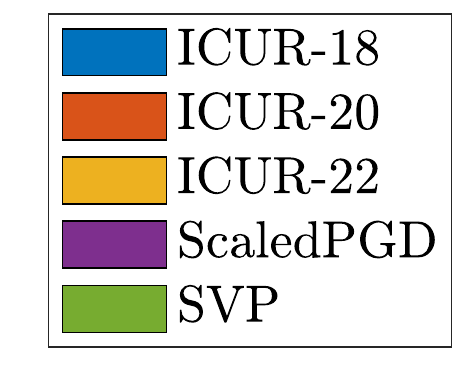}}
           
    \caption{{Bar-plot results on the recommendation system data. The performances are measured in HR, NMAE, and runtime. \textbf{Top}: overall observation rate $\alpha = 10 \%$. \textbf{Middle}: $\alpha = 20 \%$. \textbf{Bottom}: $\alpha = 30 \%$. When overall observation rate $\alpha$ is fixed, the ICURC performs better under various CCS conditions compared with other methods in uniform sampling.}} \label{fig:barplot_summary_RC}
\end{figure*}

\subsection{Link Prediction}
In link prediction problems, we are given a graph $G = (V, E)$, that has vertices $V$ and edges $E$, represented in an adjacency matrix $A$. If there exists an observed link between vertices $i$ and $j$,  then $A_{i,j}=1$; otherwise  $A_{i,j}=0$. 
Link prediction problem aims to learn the distribution of existing links and, thus, to predict the potential links in the graph \cite{pech2017link}. 
In this section, we evaluate the performance of our CCS model solved by the ICURC algorithm on three  
link prediction datasets namely Blogs\footnote{Blogs dataset can be found in \url{http://konect.cc/networks/moreno_blogs}.}  containing the hyperlinks between blogs in the context of $2004$ US election, Opsahl\footnote{ Opsahl dataset can be found in \url{http://konect.cc/networks/opsahl-openflights}.}  indicating flights between airports around the world, and Figeys\footnote{Figeys dataset can be found in \url{http://konect.cc/networks/maayan-figeys}.}   describing interactions between proteins. The three datasets come from the Koblenz Network Collection (KONECT \cite{konect}). The characteristics of 
these datasets are summarized in Table~\ref{tab:dataset_info_LP}. 

\begin{table}[h]
\centering
\caption{Information for link prediction datasets. Here, $\alpha$ is the overall observation rate.}
\label{tab:dataset_info_LP}
\begin{tabular}{c|c c}
\toprule
\textsc{Dataset} & \textsc{Size\,(\texttt{\#}nodes)} & $\alpha$\,(\%) \\ 
\midrule
Blogs   & 1224               & 1.269                      \\ 
Opsahl  & 2939               & 0.353                      \\ 
Figeys  & 2239               & 0.357                      \\ 
\bottomrule
\end{tabular}
\end{table}
To evaluate the performance of our methods, we follow the works of \cite{pech2017link, gao2015link} by randomly   dividing the existing links into training and testing samples. From the perspective of matrix completion, this is the same as randomly splitting the observed entries of the adjacency matrix $A$ into corresponding $\Omega_\mathrm{train}$ and $\Omega_\mathrm{test}$. Similar to the setup of the recommendation system problem in Section~\ref{subsec:RS}, we generate different training datasets based on the CCS model and non-CCS model and denote them as  $\Omega_\mathrm{train}^{(1)}$ and $\Omega_\mathrm{train}^{(2)}$ respectively. We apply ICURC algorithm on $\Omega_\mathrm{train}^{(1)}$ and ScaledPGD and SVP algorithms on $\Omega_\mathrm{train}^{(2)}$. We also adopt two popular metrics, $\mathrm{Precision}$   \cite{ZHAO2022116033}, which focuses on the top predicted links, and Area Under the receiver operating characteristic Curve (AUC)
h\cite{clauset2008hierarchical}, which evaluates the entire set of predicted links. Precision is defined as the ratio of the actual number of connected edges to the predicted number of connected edges. Links predicted by algorithms can be interpreted as the likelihood of unobserved (new) links; the higher likelihood indicates a greater possibility of an unobserved link \cite{pech2017link}. We sort the entries of predicted links in descending order and select the top $L$ links. $L$ is chosen to be the cardinality of the testing dataset \cite{pech2017link}. Let $L_\mathrm{m}$ be the number of links in the top $L$ predicted links that appear in the testing dataset, $ \mathrm{Precision}$ can be calculated by
\begin{equation}
    \mathrm{Precision} = \frac{L_\mathrm{m}}{L},
    \label{equ}
\end{equation}
where the higher $\mathrm{Precision}$ is, the more accurate of the prediction is \cite{ZHAO2022116033}.  

AUC measures the area under the receiver operating characteristic curve, which can be interpreted as the probability that a randomly chosen missing link from the set of  predicted $\Omega_{\mathrm{test}}$ is given a higher likelihood than a randomly chosen potentially non-existing link (which is the set of all unobserved links from $G$) \cite{clauset2008hierarchical, hanley1982meaning}. The AUC is calculated as:
\begin{equation}
    \mathrm{AUC} = \frac{y_{\mathrm{m}} + 0.5y_{\mathrm{n}}}{y},
\end{equation}
where $y$ is the number of independent comparisons between each randomly picked pair of a missing link and a non-existing link. {$y_{\mathrm{m}}$ is the times that the missing links have a higher predicted likelihood than non-existing links while $y_{\mathrm{n}}$ counts the number of times if their likelihoods are equal.} We use $y = 5000$. The degree to which the AUC exceeds $0.5$ illustrates how much better the predictions are compared with a random guess \cite{gao2015link}.

Figure~\ref{fig:barplot_summary_LP} summarises the averaged numerical results over $10$ independent trials for each fixed  $r$, $\alpha$, and $\delta$. One can see that based on the CCS model, ICURC performs better compared with the ones based on the uniform sampling model solved by other methods. Particularly, when the overall sampling rate is fixed, the ICURC based on different concentrated rows and columns can reach higher Precision and AUC under shorter intervals. This, again, confirms the computational efficiency of ICURC and the flexibility of our CCS model.

\begin{figure*}[th]
    \raggedright
    \includegraphics[width=0.29\textwidth]{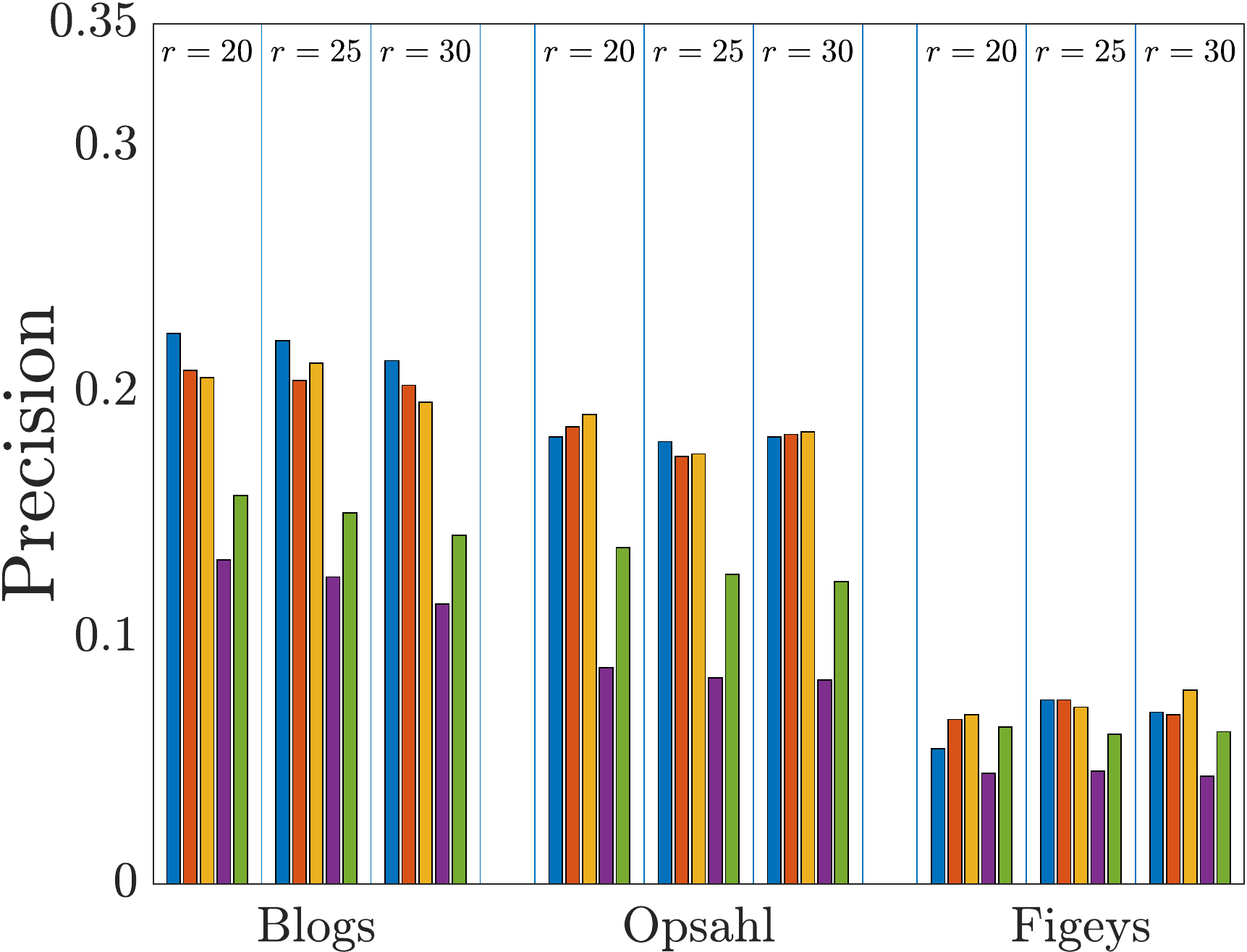} \hspace{1.5mm}
    \includegraphics[width=0.29\textwidth]{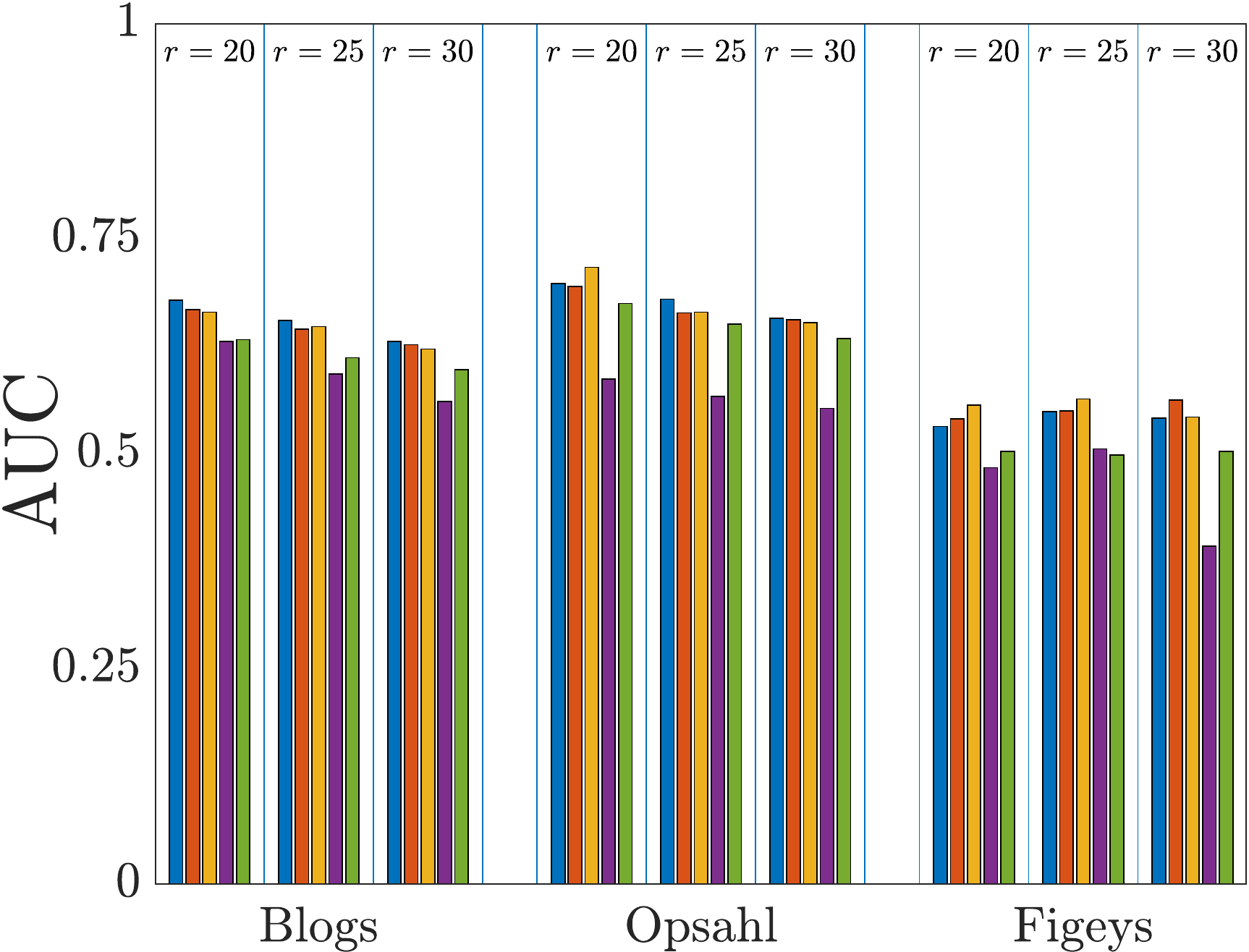} \hspace{1.5mm}
     \includegraphics[width=0.29\textwidth]{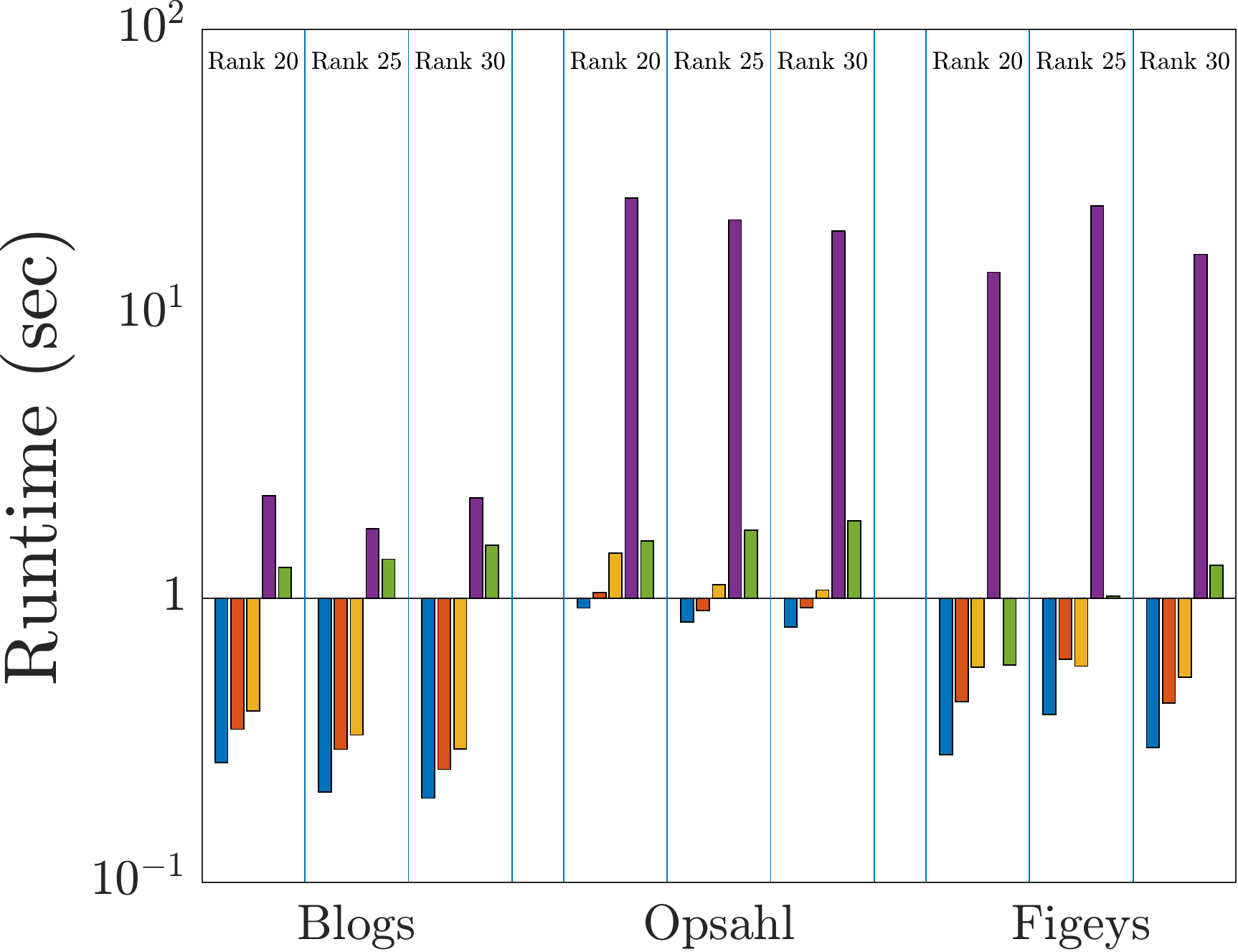}
    \\
    
    \includegraphics[width=0.29\textwidth]{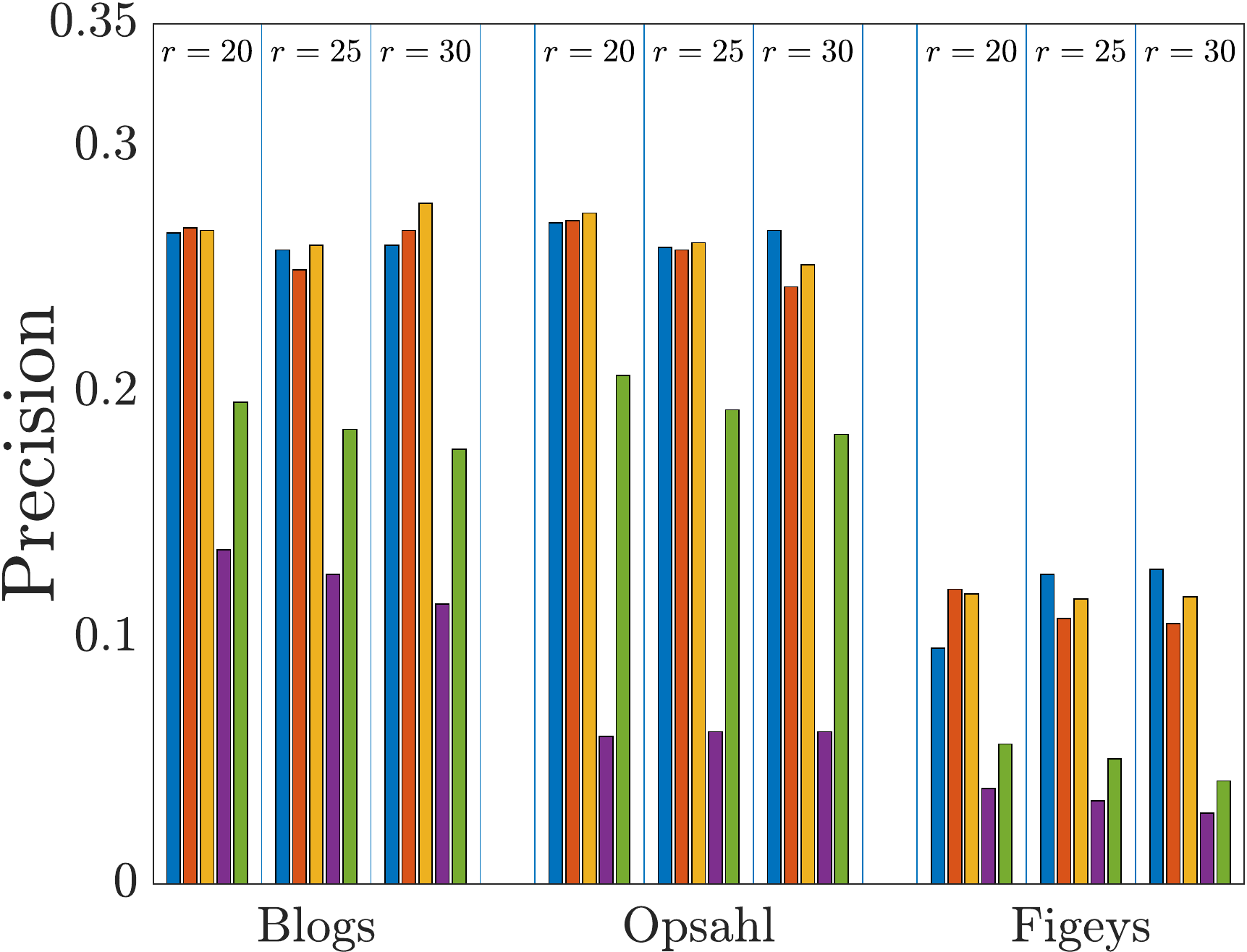} \hspace{1.5mm}
    \includegraphics[width=0.29\textwidth]{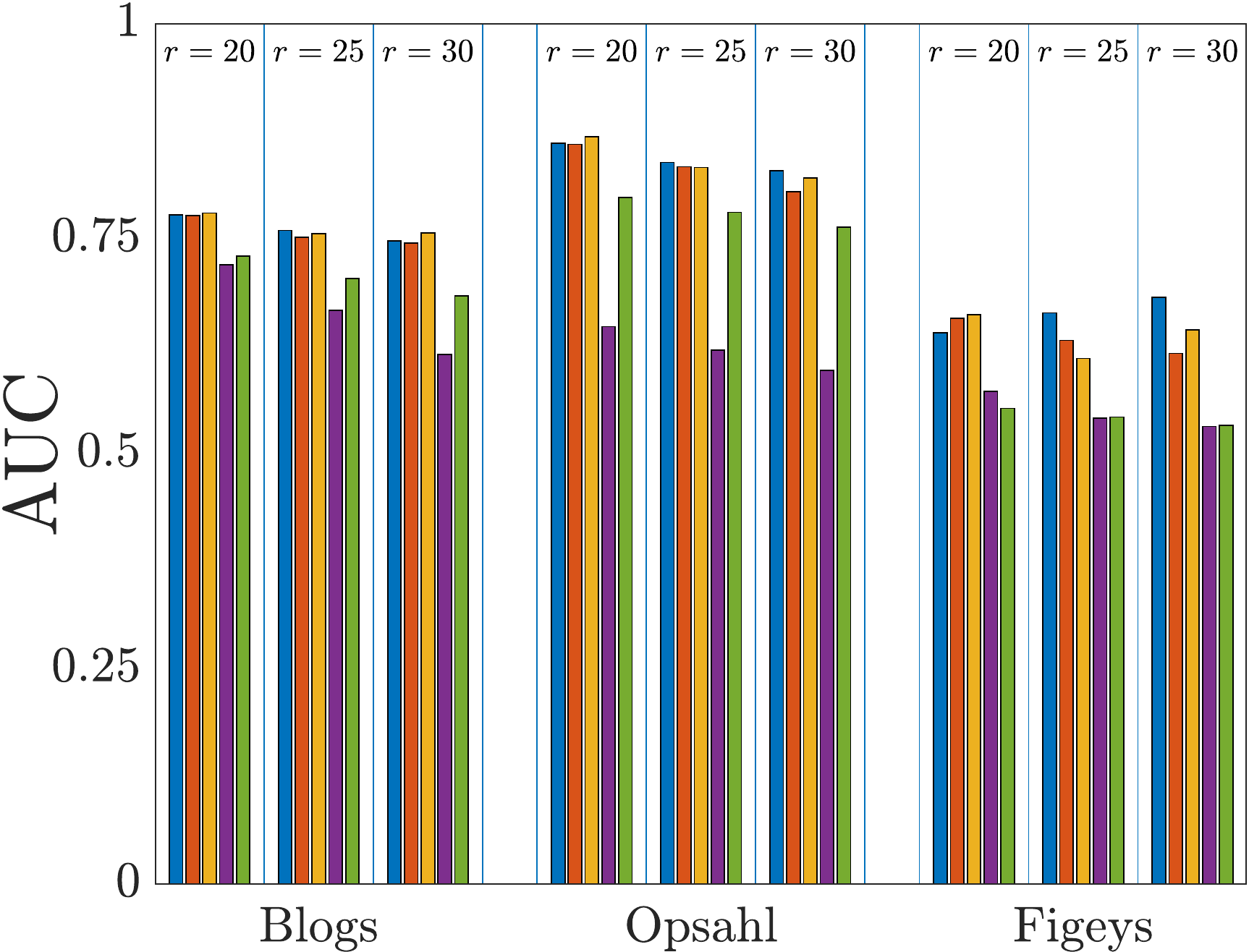} \hspace{1.5mm}
     \includegraphics[width=0.29\textwidth]{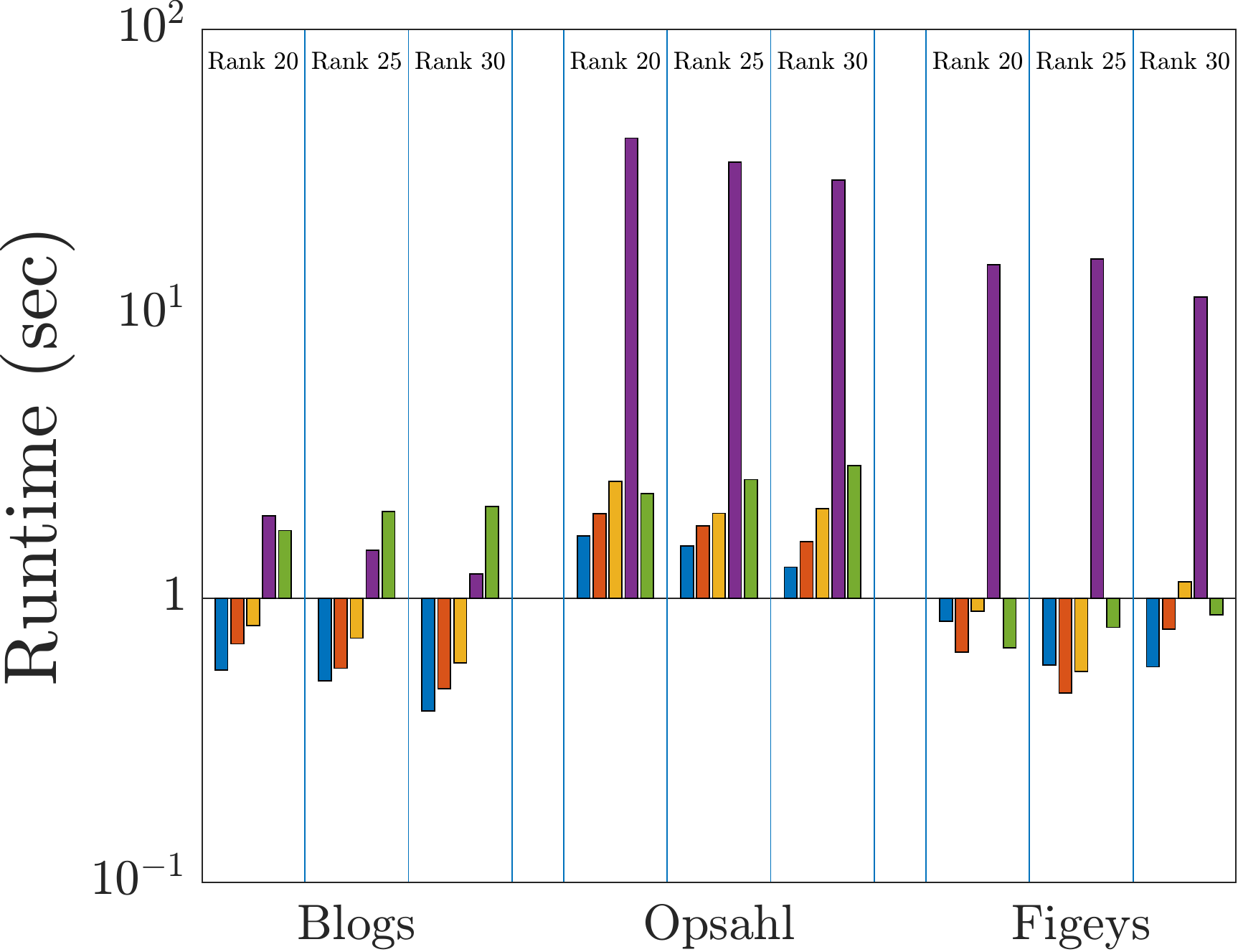}
    \\
    
    \includegraphics[width=0.29\textwidth]{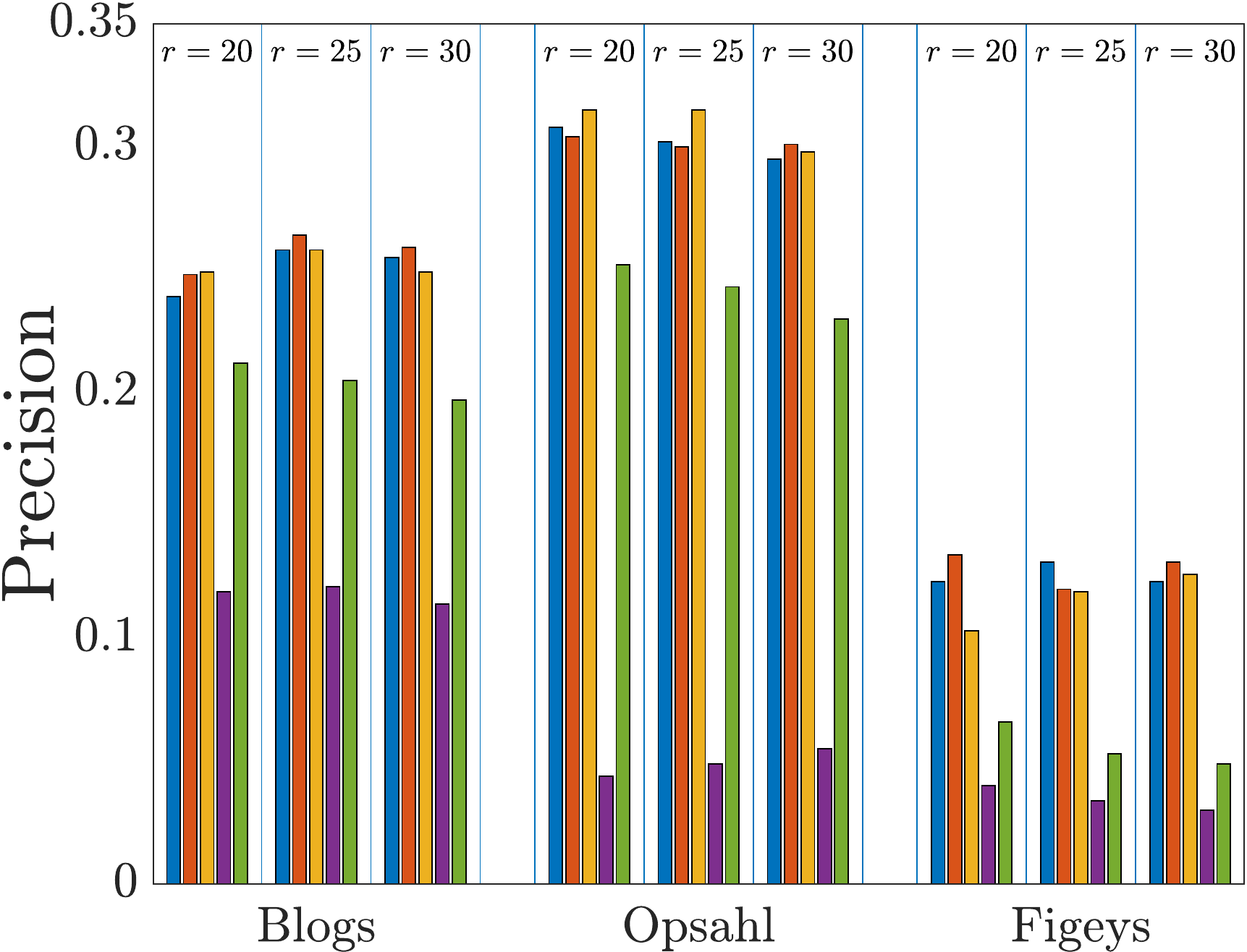} \hspace{1.5mm}
    \includegraphics[width=0.29\textwidth]{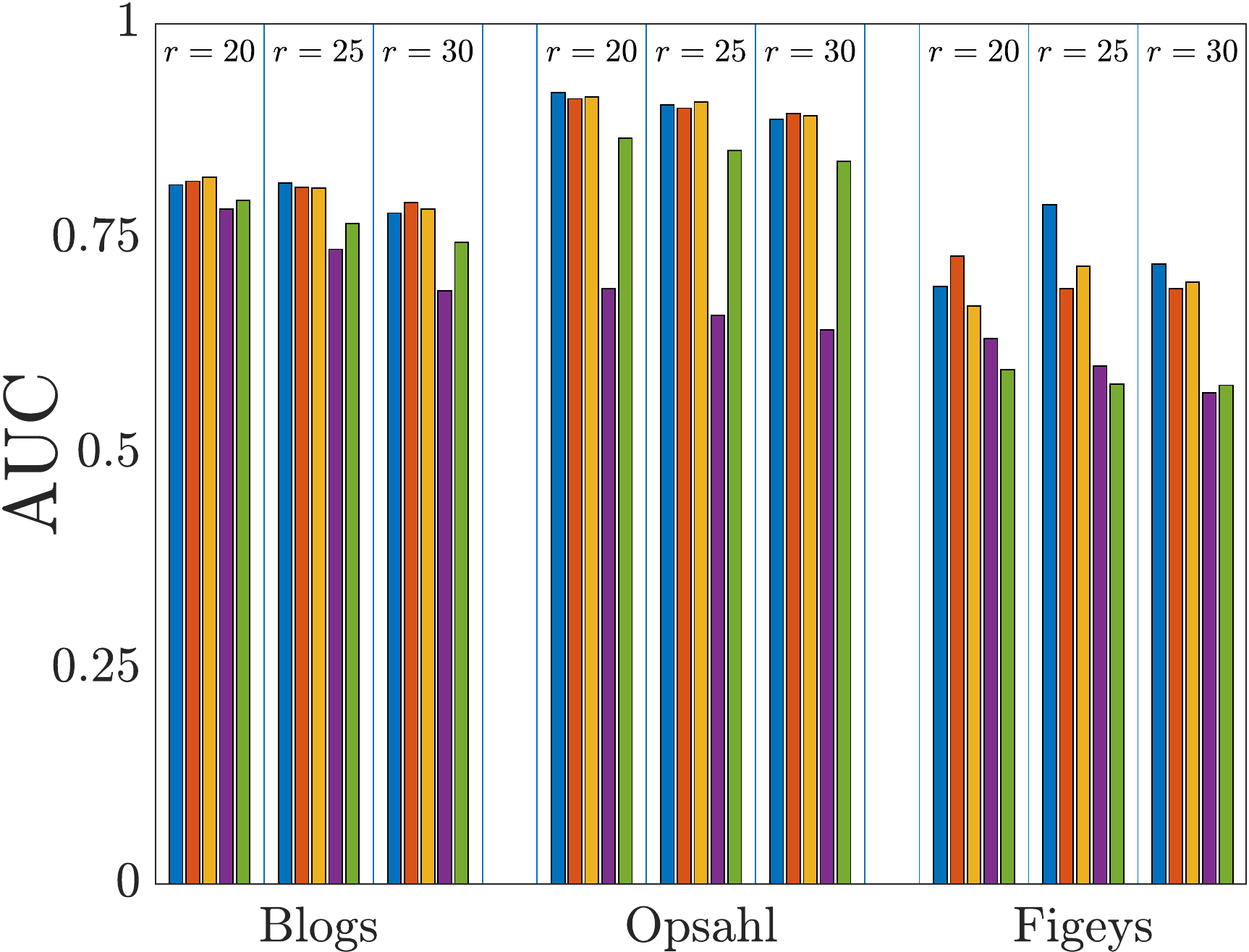} \hspace{1.5mm}
     \includegraphics[width=0.29\textwidth]{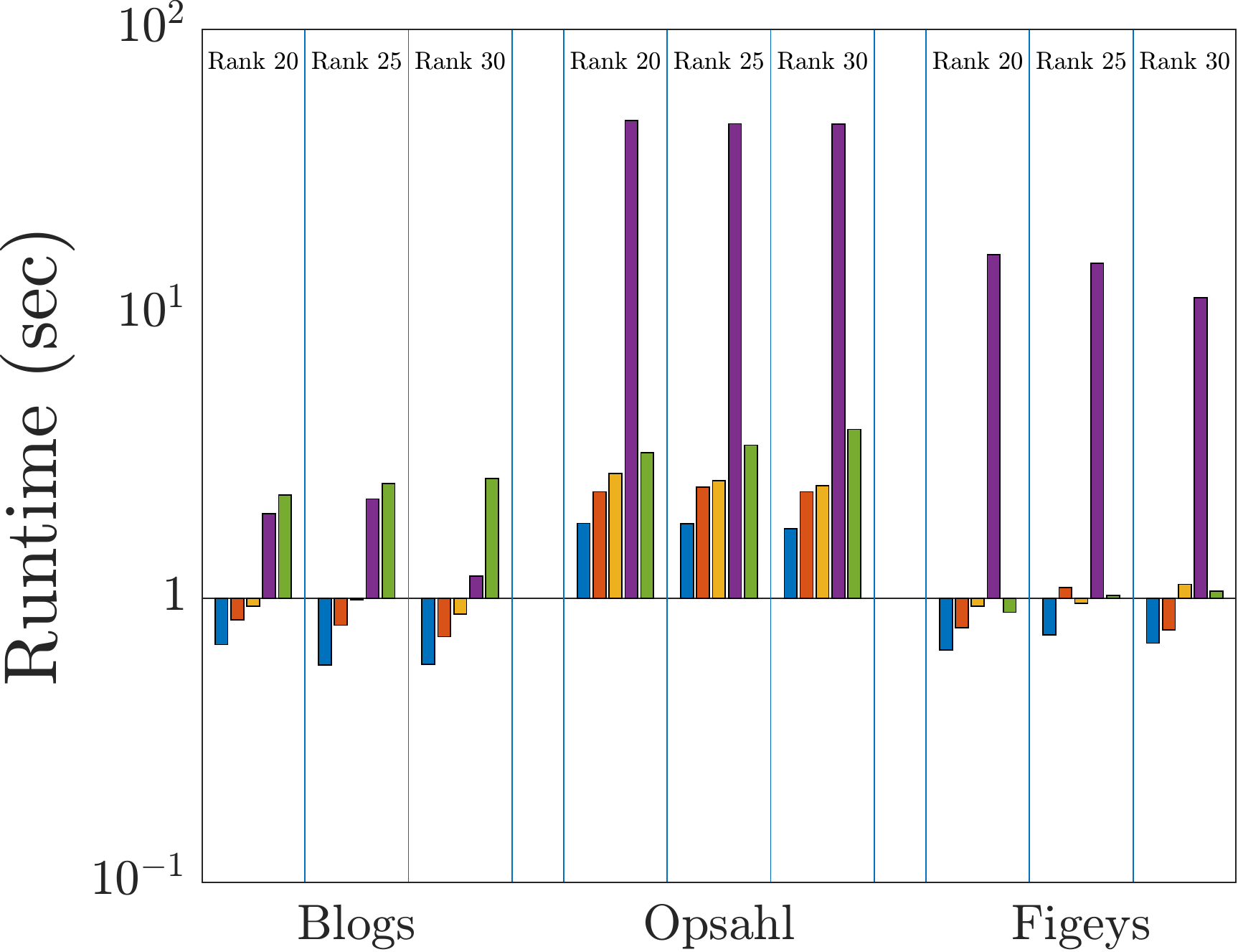}
      \raisebox{0.14\height}{\includegraphics[width=0.09\textwidth]{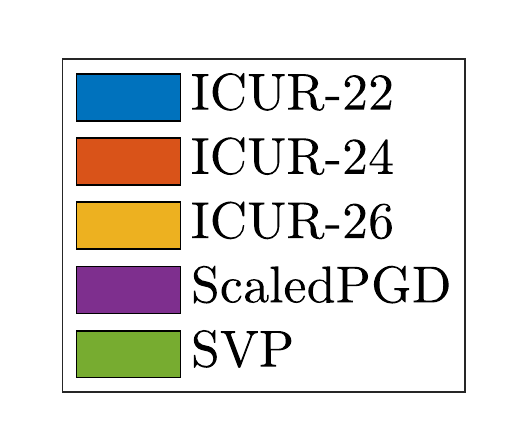}}
    
      \caption{  
        {Bar-plot results on   link prediction datasets.        The performances are measured in Precision, AUC, and runtime.  \textbf{Top}: overall observation rate $\alpha = 10 \%$. \textbf{Middle}: $\alpha = 20 \%$. \textbf{Bottom}: $\alpha = 30 \%$. When overall observation rate $\alpha$ is fixed, the ICURC performs better under various CCS conditions compared with other methods in uniform sampling.}
      }
    \label{fig:barplot_summary_LP}
\end{figure*}

\section{Proofs}\label{sec:proofs}
In this section, we provide the  proofs for the theoretical results presented in Section~\ref{sec:theoretical_results}, i.e., 
Lemma~\ref{COR:UniformIncoherence} and Theorem~\ref{thm:sufficient_condition}.

\subsection{Proof of Lemma~\ref{COR:UniformIncoherence}}

\begin{proof}[Proof of Lemma~\ref{COR:UniformIncoherence}]
We invoke \cite[Theorem 3.5]{cai2021robust}.
By setting $\delta=0.75$ and $\gamma=r\log^2(n)$ in that theorem, the following inequalities  hold  
\begin{equation*}
\begin{aligned}
\sqrt{\frac{|\cI|}{n}}\left\|[\BW]_{\cI,:}^\dagger\right\|_2&\leq 2,\cr
\mu_{1\BR}&\leq 4\kappa^2\mu_1,\cr
\mu_{2\BR}&\leq \mu_2,\cr
\kappa_{\BC}&\leq 2\sqrt{\mu_1 r}\kappa,
\end{aligned}
\end{equation*}
with probability at least
\[1-\frac{r}{\exp((0.75+0.25\log(0.25))r\log^2(n))}\geq 1-\frac{r}{n^{0.4r\log(n)}}.
\]
This completes the proof.
\end{proof}

\subsection{Proof of Theorem~\ref{thm:sufficient_condition}}\label{subsec:pf_main_thm}
The proof of our main theorem (i.e., Theorem~\ref{thm:sufficient_condition}) is based on our Two-Step Completion (TSC) algorithm. For the ease of readers, we state the  TSC algorithm first.
\begin{algorithm}[h]
 \caption{Two-Step Completion (TSC) for CCS}\label{ALG:TS_MC}
\begin{algorithmic}[1]
\State\textbf{Input}: 
$[\BX]_{\Omega_{\BR}\cup\Omega_{\BC}}$: observed data; 
$\Omega_\BR, \Omega_\BC$: observation locations; 
$\cI,\cJ$: row and column indices that define $\BR$ and $\BC$ respectively;  
$r$: target rank;
  $\mathrm{MC}$: the chosen matrix completion solver. 
\State $\widetilde{\BR} = \mathrm{MC}([\BX]_{\Omega_{\BR}},r) $ 
\State $\widetilde{\BC} = \mathrm{MC}([\BX]_{\Omega_{\BC}},r) $ 
\State $\widetilde{\BU} = \widetilde{\BC}(\cI,:) $
\State $\widetilde{\BX} = \widetilde{\BC}{\widetilde{\BU}}^\dagger \widetilde{\BR}$
\State\textbf{Output: }$\widetilde{\BX}$: approximation of $\BX$.
\end{algorithmic}
\end{algorithm}

\begin{proof}[{Proof of Theorem~\ref{thm:sufficient_condition}}]
Since $\cI$ and $\cJ$ are chosen uniformly from $[n]$, according to Lemma~\ref{COR:UniformIncoherence} we have that 
\begin{align*}
\mu_{2\BC}\leq 4\kappa^2\mu_2,~&~
\mu_{1\BR}\leq 4\kappa^2\mu_1,\cr
\kappa_{\BC}\leq 2\sqrt{\mu_2 r}\kappa,~&~
\kappa_{\BR}\leq 2\sqrt{\mu_2 r}\kappa,
\end{align*}
with probability at least $1-\frac{2r}{n^{0.4r\log(n)}}$. Thus, 
\[
\|\BW_{\BC}\BV_{\BC}^\top\|_{\infty}\leq 2\kappa\sqrt{r\mu_1\mu_2}\sqrt{\frac{r}{n|J|}}
\] 
and 
\[
\|\BW_{\BR}\BV_{\BR}^\top\|_{\infty}\leq 2\kappa\sqrt{r\mu_1\mu_2}\sqrt{\frac{r}{n|I|}}
\]
hold with  probability  at least $1-\frac{2r}{n^{0.4r\log(n)}}$.

By \cite[Theorem~2]{recht2011simpler}, the following two statements hold: 
\begin{enumerate}
    \item $|\Omega_{\BC}|\geq 128\kappa^2r^2\mu_1\mu_2(n+|\cJ|)\beta\log^2(n)$ for some $\beta>1$ ensures that $\BC$ is the minimizer   to the problem
\begin{equation*}
\begin{split}
\minimize_{\widetilde{\BC}}&\quad \|\widetilde{\BC}\|_{*} \quad\cr
\subject&\quad \cP_{\Omega_\BC}(\widetilde{\BC})=\cP_{\Omega_\BC}(\BC)
\end{split}
\end{equation*} 
  with probability at least $$1-\frac{6\log(n)}{(n+\mu_2 r^2\log^2(n))^{2\beta-2}}-\frac{1}{n^{2\beta^{0.5}-2}}.$$ 
\item $|\Omega_{\BR}|\geq 128\kappa^2r^2\mu_1\mu_2(n+|\cI|)\beta\log^2(n) $ for some $\beta>1$ ensures that $\BR$ is the minimizer to the problem 
\begin{equation*}
\begin{split}
\minimize_{\widetilde{\BR}}&\quad \|\widetilde{\BR}\|_{*} \quad \cr
\subject&\quad \cP_{\Omega_\BR}(\widetilde{\BR})=\cP_{\Omega_\BR}(\BR)
\end{split}
\end{equation*} 
 with probability at least \[1-\frac{6\log(n)}{(n+\mu_2 r^2\log^2(n))^{2\beta-2}}-\frac{1}{n^{2\beta^{0.5}-2}}.\] 
\end{enumerate}
  Since $\kappa_{\BC}\leq 2\sqrt{\mu_2 r}\kappa$ and $\kappa_{\BR}\leq 2\sqrt{\mu_2 r}\kappa$, we thus have 
  $$\rank(\BC)=\rank(\BR)=\rank(\BX)=r.$$ 
  According to \cite[Theorem 5.5]{HammHuang}, we thus have $\BX=\BC\BU^\dagger \BR$, in other words, the CUR decomposition ${\BC}\BU^\dagger \BR$ can reproduce the original $\BX$.
  
  Combining all the statements above,  $\BX$ can be exactly recovered from $\Omega_{\BC}\cup\Omega_{\BR}$ with probability  at least
\[1-\frac{2r}{n^{0.4r\log(n)}}-\frac{2}{n^{2\beta^{0.5}-2}}-\sum_{i=1}^{2}\frac{6\log(n)}{(n+\mu_ir^2\log^2(n))^{2\beta-2}}.
\]
This completes the proof.
\end{proof}

\section{Conclusion and Future Directions}
This paper proposes a novel, easy-to-implement, and practically flexible  
sampling model, coined Cross-Concentrated Sampling (CCS), for matrix completion problems that bridges the classical uniform sampling model and the CUR sampling model. For this model,  we provide a sufficient sampling bound to ensure the uniqueness of the solution for this matrix completion problem. Furthermore, we develop an efficient non-convex algorithm to solve the CCS-based MC. The efficiency of the algorithm is illustrated on synthetic and real datasets. The simulations also show that CCS provides flexibility to acquire sufficient data to ensure the successful completion that can potentially save costs in some real applications. 

There are four lines for future work. First, the sufficient bound in this paper is not tight. One can see that if the CUR sampling model is applied to a low-rank matrix of size $n\times n$, $\cO(nr\log(n))$ samples can ensure the successful completion with a high probability. There is room to improve the sampling bound for our CCS model. Second, it would be helpful to analyze the convergence guarantee of ICURC in future work. Third, our empirical simulations have shown the promising performance of ICURC, but there is likely still room for improvement. More efficient approaches will be developed with the theoretical foundation as a guide. Lastly, this novel sampling model is based on matrix CUR decomposition. Recently, tensor CUR decompositions have been proposed (cf.~\cite{cai2021mode,cai2021rtcur,che2022perturbations}). In low-rank tensor approximation,  the original tensor can be well-approximated by making good use of merely one smaller-sized subtensor and a small collection of fibers from each mode. Therefore, there is no need to access the full tensor, which makes tensor CUR decompositions memory and computationally efficient. We plan to generalize the proposed sampling model into this tensor setting \cite{THCL2023}.  

\bibliographystyle{IEEEtran}
\bibliography{IEEEabrv,ref}

\appendix

\section{Empirical Evidence for Remark~\ref{rmk:linear cong}}

 {In this appendix, we provide empirical evidence to support the claims of Remark~\ref{rmk:linear cong}. That is, our main algorithm ICURC converges linearly to the ground truth, given cross-concentrated samples. The following experiments are implemented on Matlab R2020a and executed on a Linux workstation equipped with Intel i9-9940X CPU (3.3GHz @ 14 cores) and 128GB DDR4 RAM, which is the same environment we have used for the other numerical tests in the main paper.} 

In this experiment, we form  the underlying low rank matrix $\bm{X} = \bm{A} \bm{B}^{\top} \in \mathbb{R}^{n \times n}$ of rank $r$ by generating two  random Gaussian matrices $\bm{A}$, $\bm{B} \in \mathbb{R}^{n \times r}$. To study the convergence rate of ICURC,  we have considered different settings. Specifically, we  fix the dimension  of the low-rank matrix with  $n=4000$ and produce partially observed matrices according to CCS model by varying the rank $r$, the sizes of  concentrated submatrices  $\bm{R} \in \mathbb{R}^{cr\log^2(n) \times n}$, $\bm{C} \in \mathbb{R}^{n \times cr\log^2(n)}$, and the overall observation rate   $\alpha$. 
We generate $10$ different completion problems for each given triple set  $(r,\alpha,c)$. ICURC is applied to solve the generated problems  with stop criteria $\varepsilon_{k} \le 10^{-5}$. 
The relative error $\varepsilon_{k}$  at $k$-th iteration is recorded.  The  results and detailed settings are reported in  Figures~\ref{fig:CR_4}--\ref{fig:CR_55} with vertical bars marking the error standard deviations over trials. One can see that ICURC converges almost linearly.

\vspace{2cm}

\begin{figure}[h!]
\centering
    \subfloat[$r = 5$, $\alpha = 0.05$]{\includegraphics[width=0.45\linewidth]{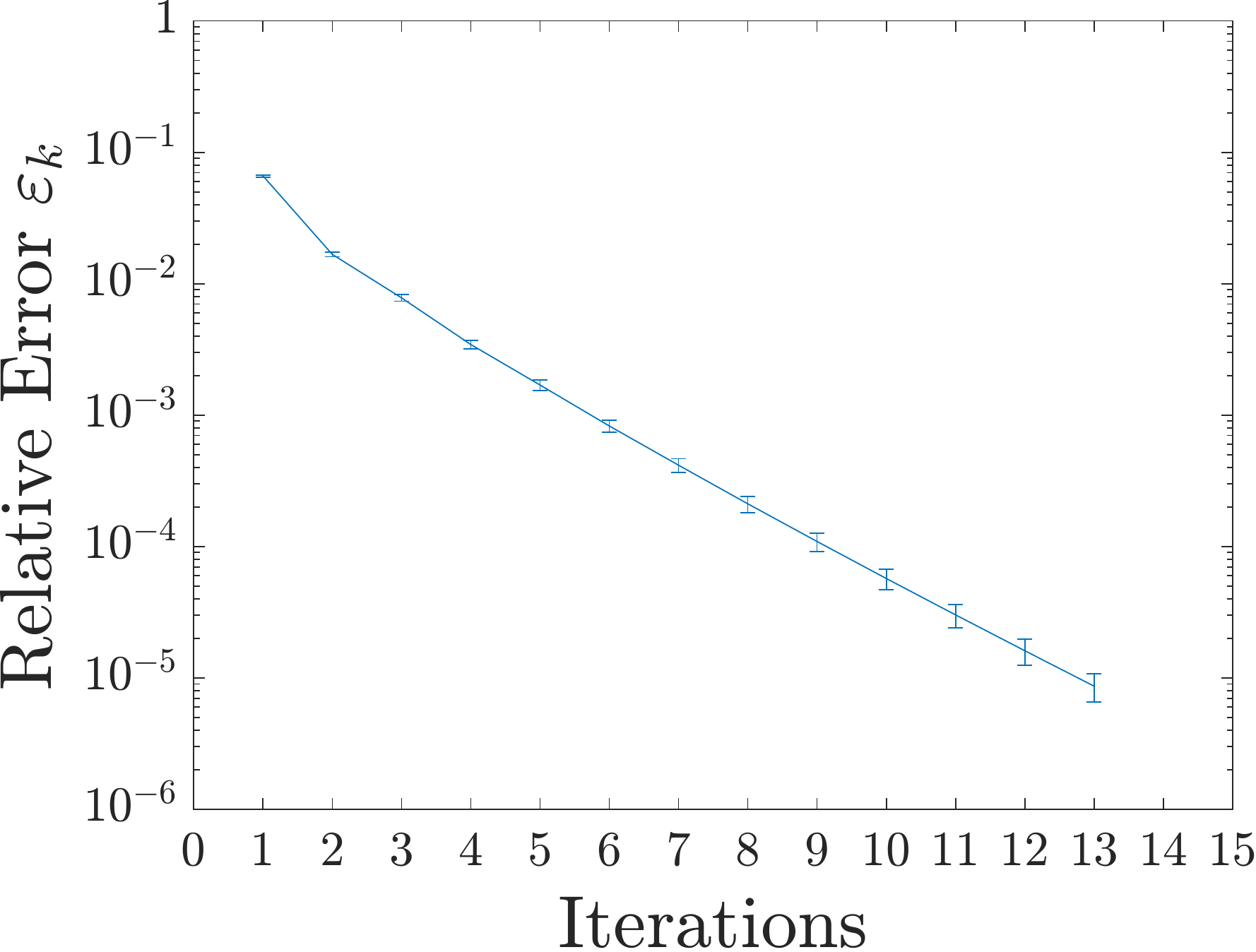}} \quad
   \subfloat[$r = 10$, $\alpha = 0.1$]{\includegraphics[width=0.45\linewidth]{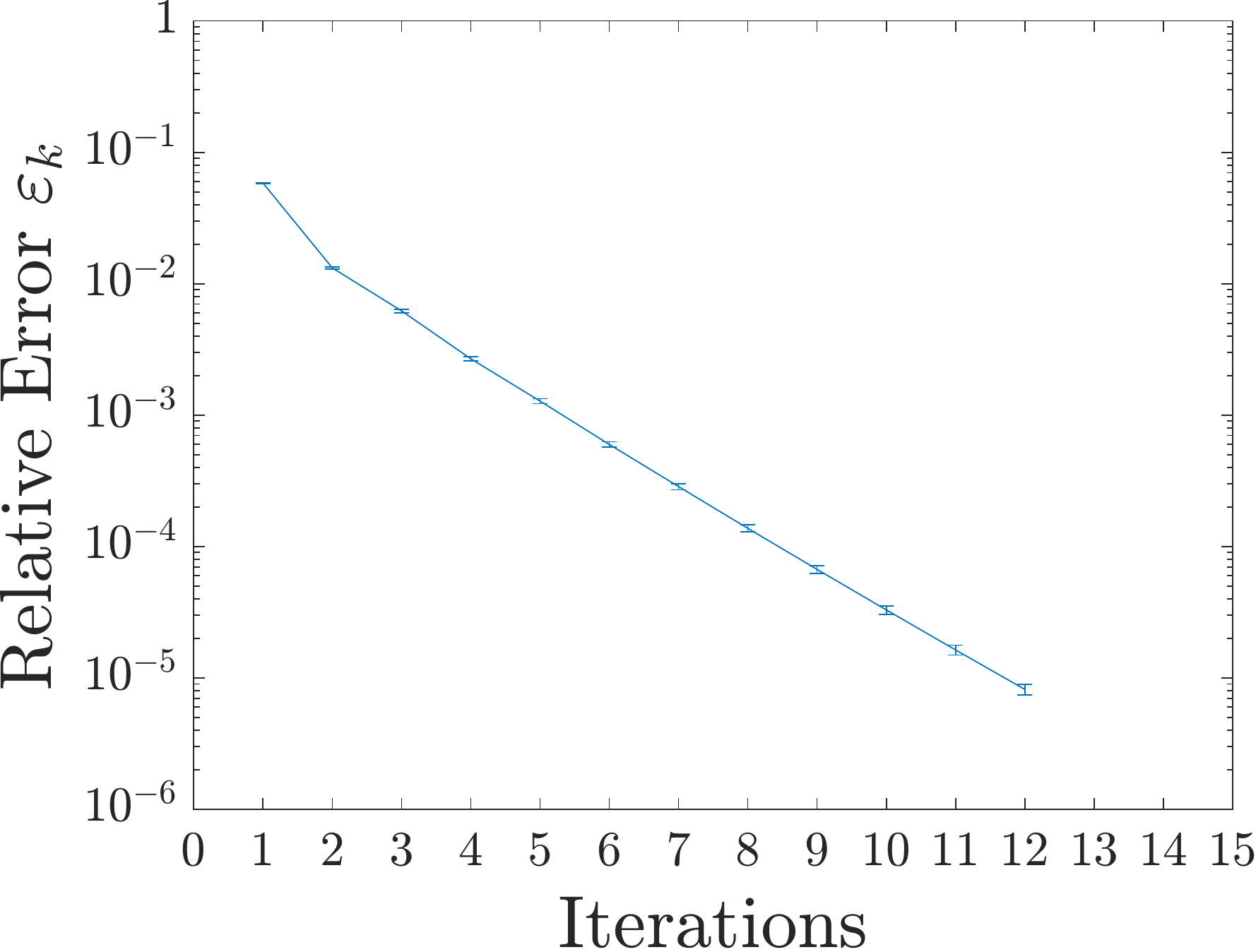}}
    \caption{The averaged relative error of ICURC with respect to iterations over $10$ independent trials with $c = 0.4$} 
    \label{fig:CR_4}
\end{figure}

\begin{figure}[h!]
\centering
    \subfloat[$r = 5$, $\alpha = 0.05$]{\includegraphics[width=0.45\linewidth]{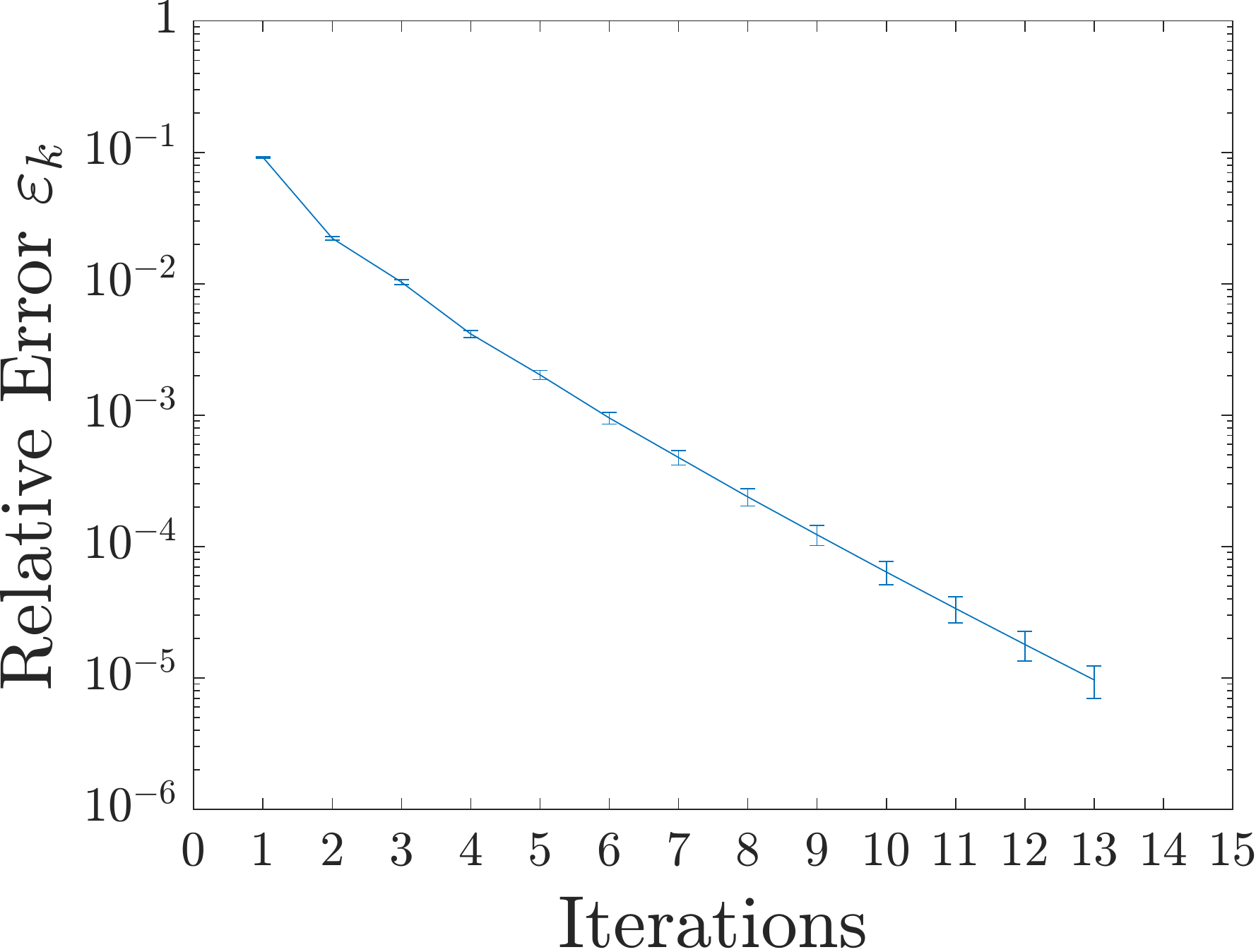}} \quad
   \subfloat[$r = 10$, $\alpha = 0.1$]{\includegraphics[width=0.45\linewidth]{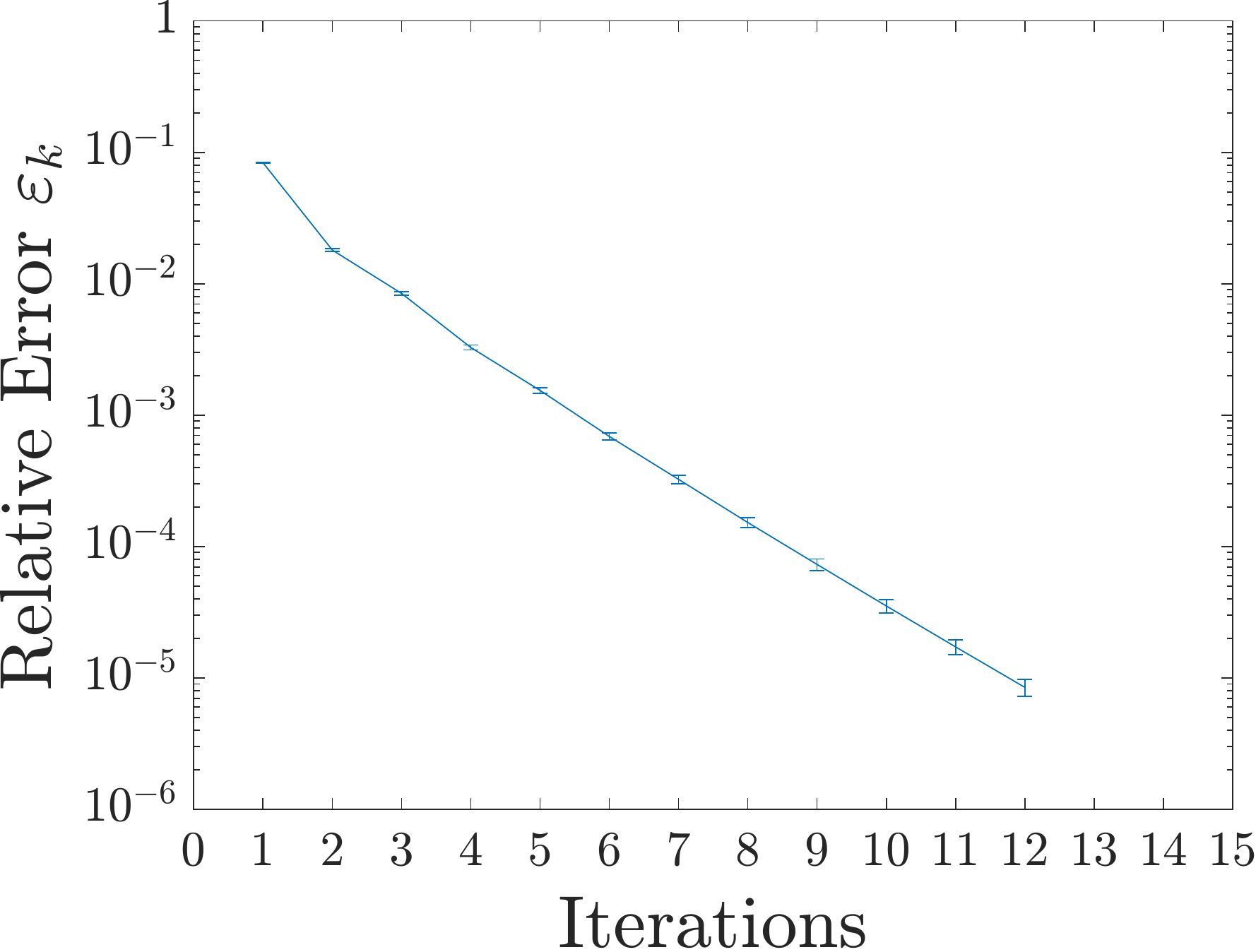}}
    \caption{The averaged relative error of ICURC with respect to iterations over $10$ independent trials with  $c = 0.45$} 
    \label{fig:CR_45}
\end{figure}

\begin{figure}[h!]
\centering
    \subfloat[$r = 5$, $\alpha = 0.05$]{\includegraphics[width=0.45\linewidth]{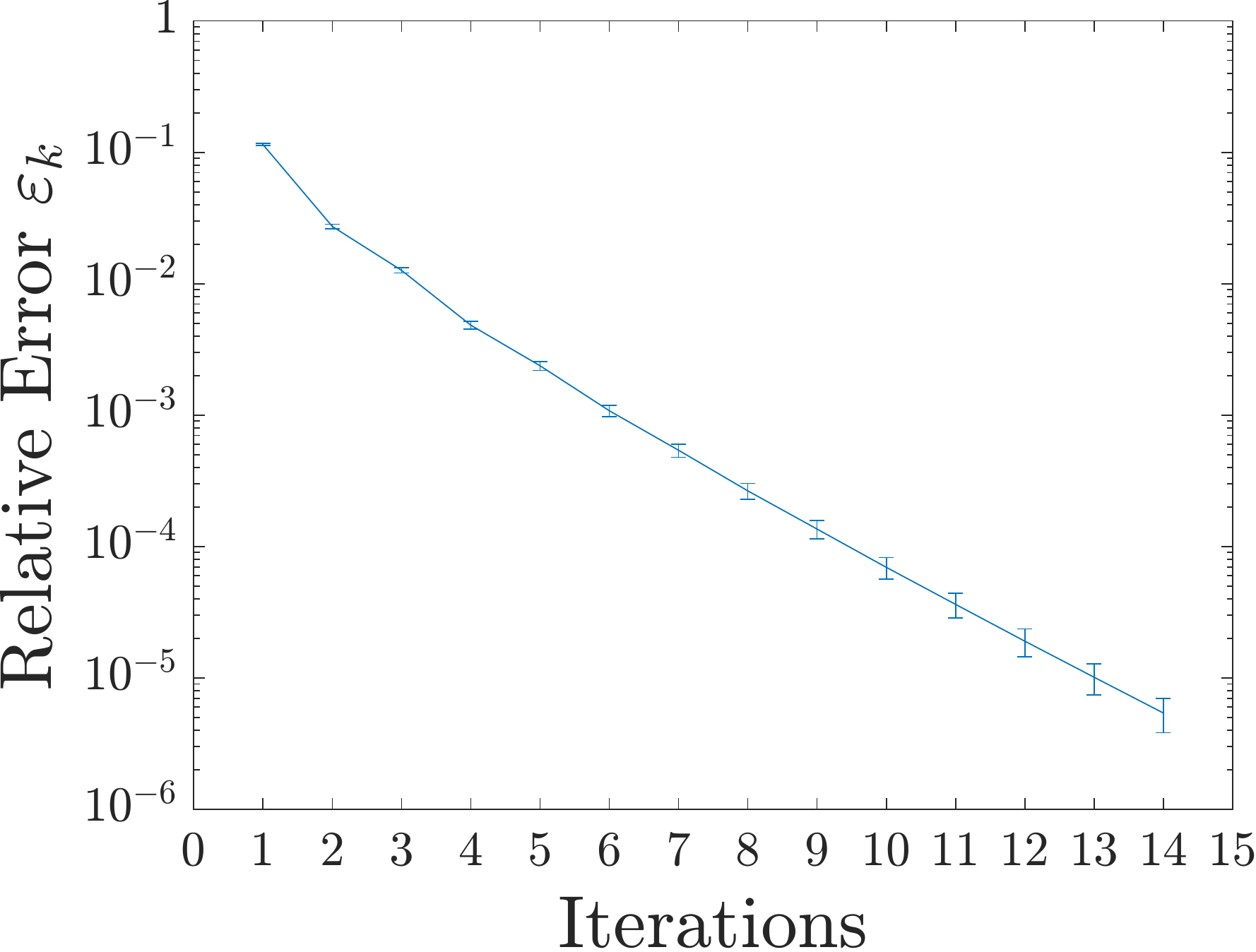}} \quad
   \subfloat[$r = 10$, $\alpha = 0.1$]{\includegraphics[width=0.45\linewidth]{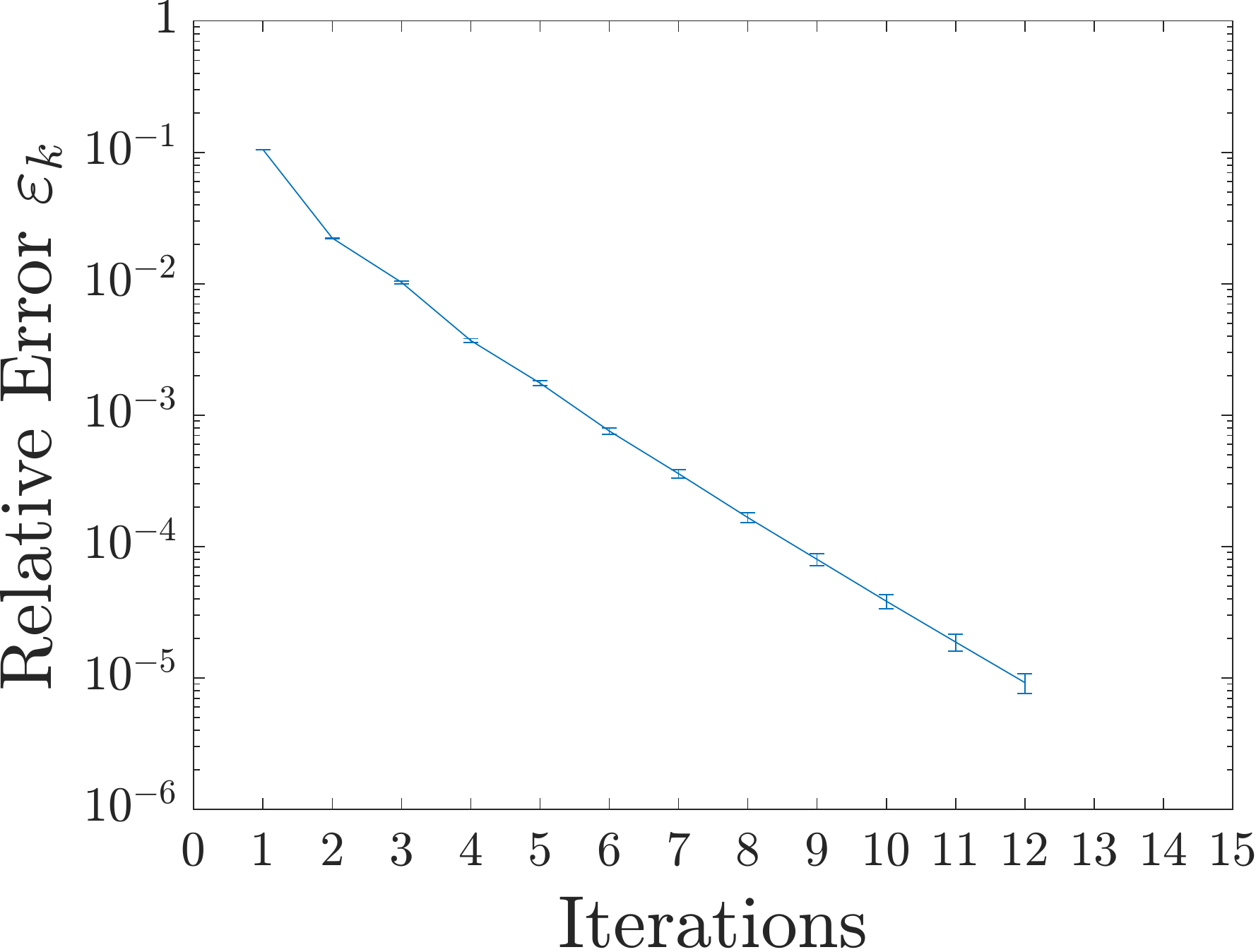}}
    \caption{The averaged relative error of ICURC with respect to iterations over $10$ independent trials with $c = 0.5$} 
    \label{fig:CR_5}
\end{figure}

\begin{figure}[h!]
\centering
    \subfloat[$r = 5$, $\alpha = 0.05$]{\includegraphics[width=0.45\linewidth]{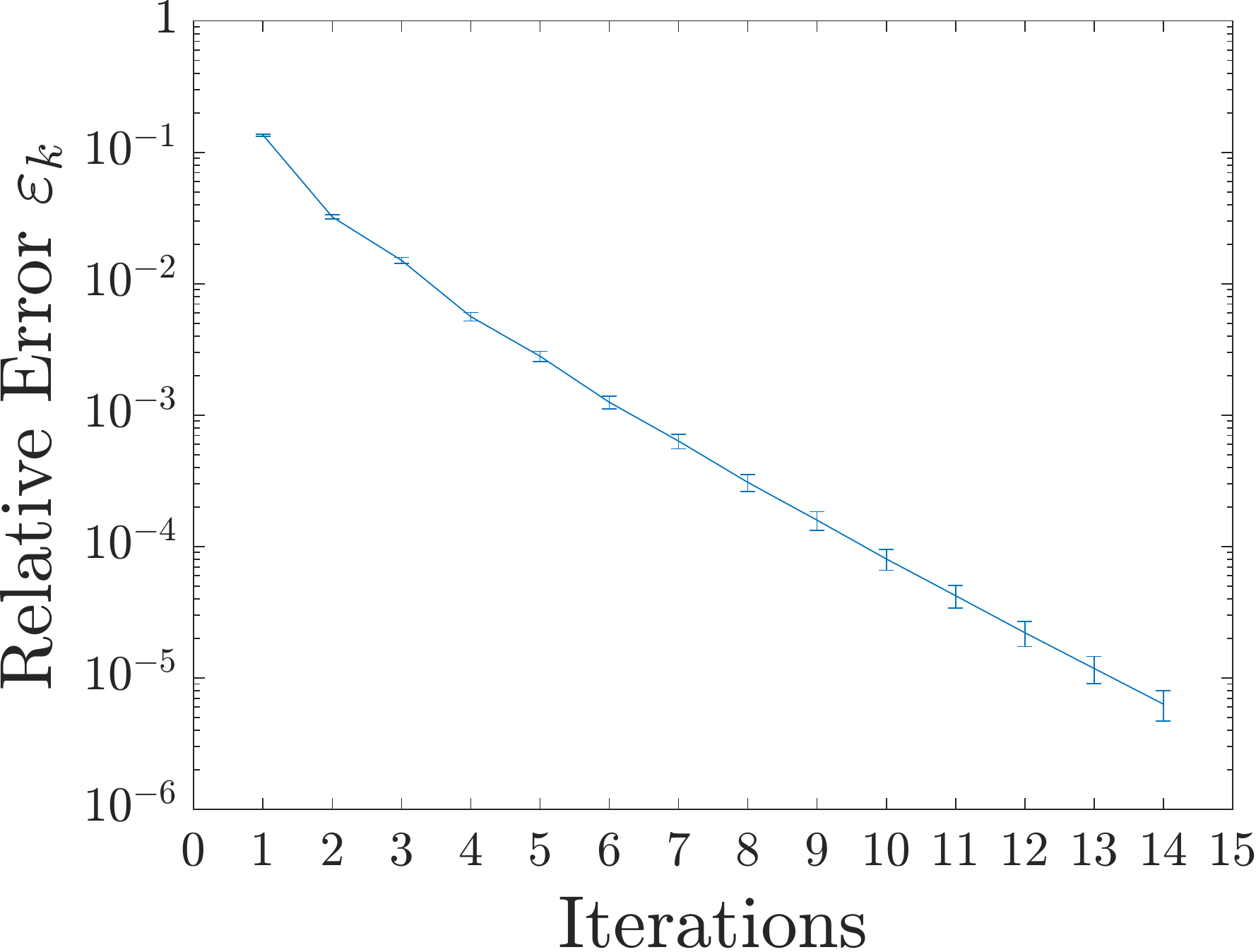}} \quad
   \subfloat[$r = 10$, $\alpha = 0.1$]{\includegraphics[width=0.45\linewidth]{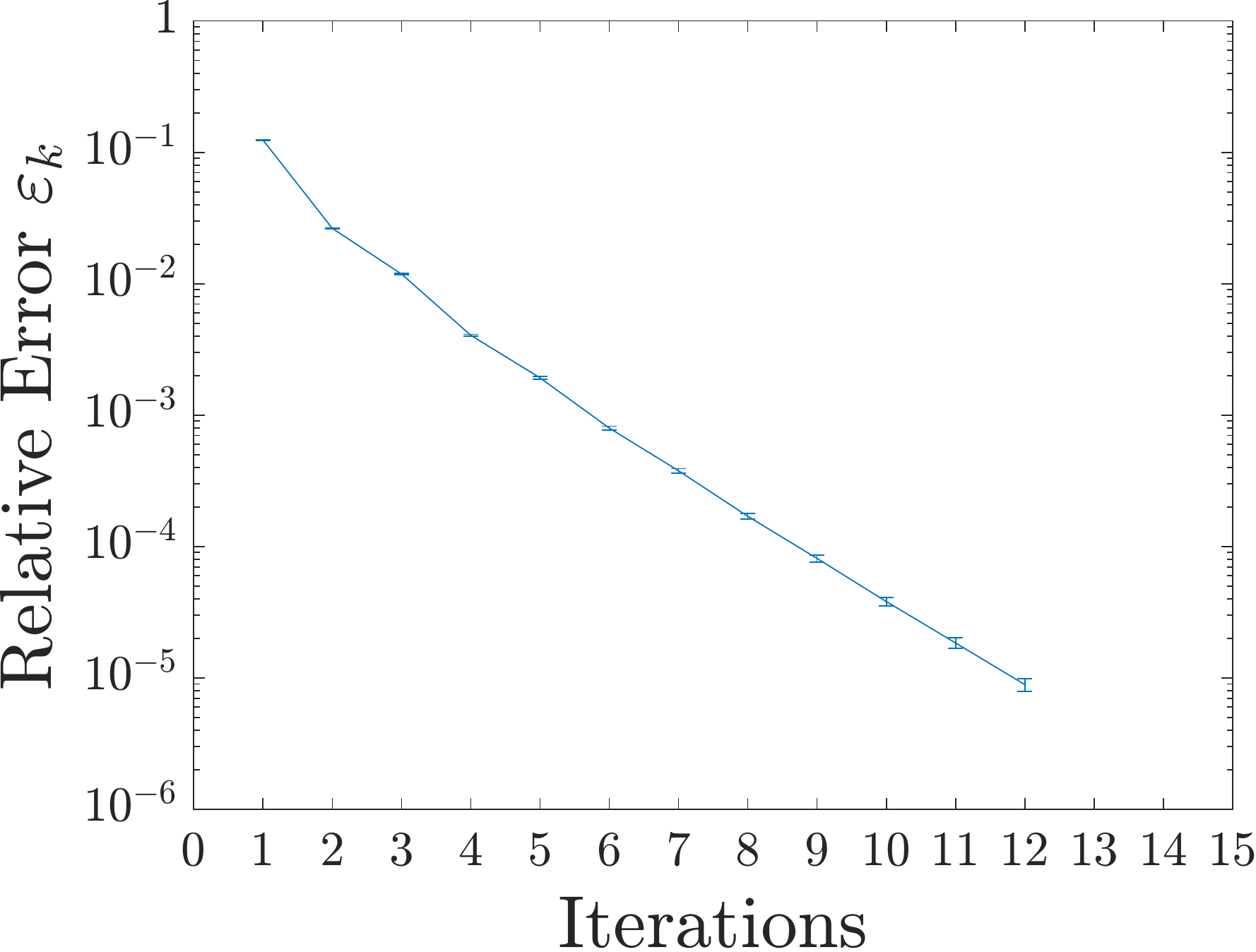}}
    \caption{The averaged relative error of ICURC with respect to iterations over $10$ independent trials with $c = 0.55$} 
    \label{fig:CR_55}
\end{figure}

\end{document}